%% file: main_arxiv.tex
\documentclass{article}

\usepackage{microtype}
\usepackage{graphicx}
\usepackage{subfigure}
\usepackage{booktabs} % for professional tables

\usepackage[accepted]{icml2024_arxiv}

\usepackage{xcolor}
\usepackage{enumitem}

\usepackage{ifthen}
\newboolean{arxiv}
\setboolean{arxiv}{true}

\usepackage{url}
\usepackage[colorlinks,
            linkcolor=blue,
           anchorcolor=blue,
            citecolor=blue,
            linktocpage
            ]{hyperref}       % hyperlinks
\usepackage{url}            % simple URL typesetting
\usepackage{booktabs}       % professional-quality tables
\usepackage{amsfonts}       % blackboard math symbols
\usepackage{nicefrac}       % compact symbols for 1/2, etc.
\usepackage{microtype}      % microtypography
\usepackage{xcolor}         % colors

\usepackage{amsmath, amssymb, amsthm, bm}
\usepackage[capitalize,noabbrev]{cleveref}
\usepackage{multirow}
\usepackage{graphicx}
\usepackage{makecell}
\usepackage{booktabs}
\usepackage{array}
\usepackage{algorithm}
\usepackage{algorithmic}
\usepackage{dsfont}
\usepackage{csquotes}
\usepackage{cancel}
\usepackage{thmtools} 
\usepackage{xfrac} 
\usepackage{thm-restate}
\usepackage{caption}
\usepackage{nicefrac}
\usepackage[disable]{todonotes}
\usepackage{indentfirst}
\usepackage{float}
\usepackage{cleveref}
\usepackage{comment}
\usepackage{multirow}
\usepackage{colortbl}
\usepackage{pdfpages}
\input{math_commands}

\input{symbols}

\newcommand{\apnote}[1]{{\color{blue}}[]}  %[AP: #1]}}
\newcommand{\snote}[1]{{\color{red}}[]} %[SM: #1]}}
\newcommand{\mz}[1]{{\color{orange}}[]}  %[MX: #1]}}

\theoremstyle{plain}
\newtheorem{theorem}{Theorem}[section]

\newtheorem{lemma}[theorem]{Lemma}

\theoremstyle{definition}
\newtheorem{definition}[theorem]{Definition}

\theoremstyle{remark}
\newtheorem{remark}[theorem]{Remark}

\icmltitlerunning{Trainable Transformer in Transformer}

\begin{document}

\twocolumn[
\icmltitle{Trainable Transformer in Transformer}



\icmlsetsymbol{equal}{*}

\begin{icmlauthorlist}
\icmlauthor{Abhishek Panigrahi}{}
\icmlauthor{Sadhika Malladi}{}
\icmlauthor{Mengzhou Xia}{}
\icmlauthor{Sanjeev Arora}{}
\end{icmlauthorlist}

\begin{center} 
{\mbox{\{ap34,smalladi,mengzhou,arora\}@cs.princeton.edu}}
\end{center} 
%\newline
\begin{center} 
{\centering\mbox{Department of Computer Science, Princeton University}}
\end{center} 

%\icmlaffiliation{yyy}{Department of Compute Science, Princeton University}

%\icmlcorrespondingauthor{Abhishek Panigrahi}{ap34@princeton.edu}

\icmlkeywords{Machine Learning, ICML}

\vskip 0.3in
]




\input{abstract}

\input{intro}

\input{design}

\input{modifications_mainpaper}

\input{experiments}

\input{related_works}
\input{Discussion}

%\newpage
\input{BroaderImpact}

\bibliography{bibliography}
\bibliographystyle{icml2024}

\newpage 

\appendix
\onecolumn
\tableofcontents

%\tableofcontents

\input{Appendix}

\end{document}

% This document was modified from the file originally made available by
% Pat Langley and Andrea Danyluk for ICML-2K. This version was created
% by Iain Murray in 2018, and modified by Alexandre Bouchard in
% 2019 and 2021 and by Csaba Szepesvari, Gang Niu and Sivan Sabato in 2022.
% Modified again in 2023 and 2024 by Sivan Sabato and Jonathan Scarlett.
% Previous contributors include Dan Roy, Lise Getoor and Tobias
% Scheffer, which was slightly modified from the 2010 version by
% Thorsten Joachims & Johannes Fuernkranz, slightly modified from the
% 2009 version by Kiri Wagstaff and Sam Roweis's 2008 version, which is
% slightly modified from Prasad Tadepalli's 2007 version which is a
% lightly changed version of the previous year's version by Andrew
% Moore, which was in turn edited from those of Kristian Kersting and
% Codrina Lauth. Alex Smola contributed to the algorithmic style files.

%% file: math_commands.tex
\usepackage{bm, dsfont}

\newcommand{\vb}{\mathbf{v}}

\newcommand{\cL}{\mathcal{L}}

\newcommand{\cV}{\mathcal{V}}

\newcommand{\RR}{\mathbb{R}}

\newcommand{\argmax}{\mathop{\mathrm{argmax}}}

\newcommand{\norm}[1]{\left\|#1\right\|}
\def\abs#1{\left| #1 \right|}

\def\va{{\bm{a}}}
\def\vb{{\bm{b}}}

\makeatletter
\newcommand{\ve}{\@ifnextchar\bgroup{\velong}{{\bm{e}}}}
\newcommand{\velong}[1]{{\bm{#1}}}
\makeatother

\def\vk{{\bm{k}}}

\def\vm{{\bm{m}}}

\def\vp{{\bm{p}}}
\def\vq{{\bm{q}}}
\def\vr{{\bm{r}}}
\def\vs{{\bm{s}}}

\def\vu{{\bm{u}}}
\def\vv{{\bm{v}}}
\def\vw{{\bm{w}}}
\def\vx{{\bm{x}}}
\def\vy{{\bm{y}}}
\def\vz{{\bm{z}}}

\def\vtheta{{\bm{\theta}}}

\def\mA{{\bm{A}}}
\def\mB{{\bm{B}}}

\def\mE{{\bm{E}}}

\def\mI{{\bm{I}}}

\def\mK{{\bm{K}}}

\def\mQ{{\bm{Q}}}

\def\mS{{\bm{S}}}

\def\mV{{\bm{V}}}
\def\mW{{\bm{W}}}

%% file: symbols.tex
\newcommand{\aux}[1]{#1_{\text{aux}}}
\newcommand{\simu}[1]{#1_{\text{sim}}}

\newcommand{\Loss}{\mathcal{L}}

\newcommand{\attsplit}{\textsc{Split}}
\newcommand{\attmerge}{\textsc{Vectorize}}
\newcommand{\attnfn}{f_{\mathrm{attn}}}
\newcommand{\vgamma}{ \mathbf{\gamma} }
\newcommand{\losspartial}[1]{\partial_{#1}}
\newcommand{\act}{\sigma_{\textrm{act}}}
\newcommand{\normalize}{f}

\newcommand{\ap}[1]{\textcolor{blue}{} }
\newcommand{\sm}[1]{\textcolor{red}{} }
\newcommand{\xm}[1]{\textcolor{green}{} }

\newcommand{\simulator}{\textsc{TinT}}
\newcommand{\simuw}[1]{ {#1}^{\textsc{TinT}}}

\newcommand{\gelu}{\textsc{GeLU}}

\newcommand{\layernorm}{f_{\text{ln}}}
\newcommand{\lora}{\textsc{LoRA}}
\newcommand{\ia}{\textsc{IA3}}

%% file: abstract.tex
\begin{abstract}
 Recent works attribute the capability of in-context learning (ICL) in large pre-trained language models to implicitly simulating and fine-tuning an internal model (e.g., linear or 2-layer MLP) during inference. 
 However, such constructions require large memory overhead, which makes simulation of more sophisticated internal models intractable.
In this work, we propose a new efficient construction, {\em Transformer in Transformer} (in short, \textsc{TinT}), that allows a transformer to simulate and fine-tune more complex models during inference (e.g., pre-trained language models). 
 In particular, we introduce  innovative approximation techniques that allow a \textsc{TinT} model with less than 2 billion parameters to simulate and fine-tune a 125 million parameter transformer model within a single forward pass.
 \textsc{TinT} accommodates many common transformer variants and its design ideas also improve the efficiency of past instantiations of simple models inside transformers.  
We conduct end-to-end experiments to validate the internal fine-tuning procedure of \textsc{TinT} on various language modeling and downstream tasks. For example, even with a limited one-step budget, we observe \textsc{TinT} for a \textsc{OPT-125M} model improves performance by $4-16\%$ absolute on average compared to \textsc{OPT-125M}. These findings suggest that large pre-trained language models are capable of performing intricate subroutines. To facilitate further work, a modular and extensible \ifthenelse{\boolean{arxiv}}{ \href{https://github.com/abhishekpanigrahi1996/transformer_in_transformer}{codebase \footnote{\url{https://github.com/abhishekpanigrahi1996/transformer_in_transformer}}}}{codebase}  for \textsc{TinT} is included. 
\end{abstract}

%% file: intro.tex
%\vspace{10em}
\section{Introduction}
%\snote{Rewrite}
\label{sec:construction}

\input{thm_box}
Large transformers~\citep{vaswani2017attention} have brought about a revolution in language modeling, with scaling yielding significant advancements in capabilities~\citep{brown2020language, chowdhery2022palm}. 
These capabilities include performing in-context learning or following natural language instructions at inference time.
%(e.g., using in-context instructions) or exemplars to learn to solve a previously unseen tasks\cite{???}.

Researchers have tried to understand how these models can learn new tasks without parameter updates~\citep{garg2022can,von2023uncovering,xie2022an,nanda2023progress}. 
%\mz{feels that we need citations here, i can find some later if you don't have anything in mind}
A popular hypothesis is that in-context learning corresponds to the transformer (referred to as the {\em simulator} from now on)  simulating gradient-based learning of a smaller model (called \emph{auxiliary} model) that is embedded within it. 
%(i.e., in-context learning).
%In the case of the former, models can follow instructions to perform previously unseen tasks.
%This capability is incredibly helpful (\snote{examples}) but can also lead the model to follow and execute harmful commands (\snote{examples}).
%In the case of in-context learning, the model can efficiently learn new input-output mappings from the input sequence and make predictions on a test query, even though the input sequence uses a format that is unlikely to have been seen during pre-training.
%how well LLMs, pre-trained and fine-tuned, can perform 
%\emph{inference-time adaptation} to learn new semantics (e.g., new functions) and syntax (e.g., instruction format) without any parameter updates.

From perspective of AI safety and alignment~\citep{amodei2016concrete, leike2018scalable, askell2021general}, the ability of a larger model to use input data (which could be arbitrary in a deployed setting) to implicitly train an auxiliary model  feels worrisome. 
This concern felt minor
%\mz{mitigate seems to be a very active word, I think we should say this worry was objectively minor due to efficiency limitations for the current simulations}
due to efficiency considerations: previous analyses and experiments required the auxiliary model to be quite tiny compared to the simulator. 
For instance, simulating and training an auxiliary model that is a linear layer  requires tens of millions of parameters in the simulator~\citep{akyurek2022learning}. 
This scaling is even more dramatic if the auxiliary model is a multi-layer fully-connected net~\cite{giannou2023looped}. 

%Building a reasonable model of how transformers perform this adaptation is crucial to promoting the helpful instances of adaptation while mitigating the harmful ones.

Our primary contribution is an explicit and nontrivial construction of a simulator called \simulator{} that explicitly adapts to the context without parameter updates.
In particular, we show that a forward pass through a modestly sized \simulator{} can involve gradient-based training of an auxiliary model that is itself a  large transformer.
%of the larger (simulator) model can involve gradient-based training on an auxiliary model that is itself another  large transformer. \mz{This sentence is very very hard to read.}
For example, we show that \simulator{} with 2B parameters can faithfully simulate fine-tuning a 125M parameter auxiliary transformer in a single forward pass.
%a 125M parameter auxiliary transformer can fit in  a transformer-style simulator  with 2B parameters. \mz{I would say: For example, a transformer-style simulator could faithfully simulate fine-tuning a 125M parameter auxiliary transformer with one single forward pass.}
(Prior constructions would have required trillions of parameters in the simulator for a far simpler auxiliary model.) %that we . 
%trillions of parameters that would have been required in the more straightforward constructions similar to ones in earlier papers would have required trillions of parameters 
%This efficiency should be contrasted with prior (weaker) results,  where  allowing an auxiliary model that is a linear classifier (and hence whose number of parameters is at most the context-length) requires tens of millions of parameters~\citep{akyurek2022learning} in the simulator. There were no efficient constructions of complicated auxiliary models. 
%existing constructions cannot express simulating a more complex auxiliary model.
%This drastic scaling inhibits us from applying this hypothesis to meaningfully study LLM test-time adaption, because (1) the simulator is restricted to very simple adaptations, and (2) the massive scale and complexity of these constructions suggests that this mechanism is unlikely to occur in the models available to us today.

\iffalse 
A \simulator{} model with fewer than two billion parameters can accommodate gradient descent on an auxiliary transformer with 125 million parameters.
This construction drastically improves the parameter efficiency of prior works and expresses a significantly more complex auxiliary model than previously considered.\fi 

%\mz{It's weird to make such retraction statement here first when people are excited to hear about how you constructed the TINT. I think we should start talking about the main result, and say  we highlight the key design choices that reduce the size of \simulator{} and finally say we defer details xxx.}

Our main result is described in \Cref{thm:main}, which details how the size of \simulator{} depends on the auxiliary model. Our construction is generally applicable to diverse variants of pre-trained language models.
The rest of the paper is structured to highlight the key design choices and considerations in \simulator{}. 
\begin{enumerate}[leftmargin=*]
    \item \Cref{sec:design} discusses the overall design decisions required to make \simulator{}, including how the simulator can read from and write to the auxiliary model and how the data must be formatted.
    \item \Cref{sec:exposition_linear_forward} uses the linear layer as an example to describe how highly parallelized computation and careful rearrangement of activations enable \simulator{} to efficiently simulate the forward pass of the auxiliary model.
    \item \Cref{sec:modification} describes how \simulator{} uses first-order approximations and stop gradients to compute the \emph{simulated gradient} of the auxiliary model. 
    \item \Cref{sec:experiments} performs experiments comparing \simulator{} to suitable baselines in language modeling and in-context learning settings. Our findings validate that the simulated gradient can effectively update large pre-trained auxiliary models. Notably, we instantiate \simulator{} in a highly extensible codebase, making \simulator{} the first such construction to undergo end-to-end evaluation.
\end{enumerate}

Due to the complexity of the construction, we defer the formal details of \simulator{} to the appendix.

%% file: thm_box.tex
\begin{figure*}[t] % t stands for top of the page, adjust as needed
    \centering
    \fbox{ % Use \fbox to draw a box around content
        \begin{minipage}{\textwidth} % Use minipage to limit the width of the box
            \centering
            \vspace{2pt}
            \textbf{\simulator{} can efficiently perform simulated gradient descent of an auxiliary model.} \\ % Adjust the title size and weight
            \vspace{4pt} % Add some vertical space between title and text
            \begin{theorem}
    Consider an auxiliary transformer with  $L$ layers, $\aux{D}$ embedding dimension, $\aux{H}$ attention heads, and a maximum sequence length of $\aux{T}$. 
    Given a hyperparameter $S$ (see~\Cref{sec:stack}), \simulator{} can perform an efficient forward pass~(\Cref{sec:exposition_linear_forward}), compute the simulated gradient~(\Cref{sec:modification}), and evaluate the updated auxiliary model with a total of $$\left(\frac{(c_1 S^2 + c_3) \aux{D}^2}{\min(\aux{H}, S^2)} \cdot \aux{D}^2 + c_2 S \aux{D} \min(S^2, \aux{H}) + c_3 \frac{\aux{T} \aux{D} S}{ \min(\aux{H}, S^2) } \right) L$$ parameters, with constants $c_1, c_2, c_3 < 150$.
    The \simulator{} model has $\simu{D} = S\aux{D}$ embedding dimension and $\simu{H}=\min(S^2, \aux{H})$  attention heads. 
    See~\Cref{tab:construction} for a detailed breakdown of the parameters.
    \label{thm:main}
\end{theorem} 
\vspace{2pt}
        \end{minipage}
    }
\end{figure*}

%% file: design.tex
\input{figures/General_structure_ICL.tex}

\begin{figure*}[t]
    \centering
    \includegraphics[width=\textwidth]{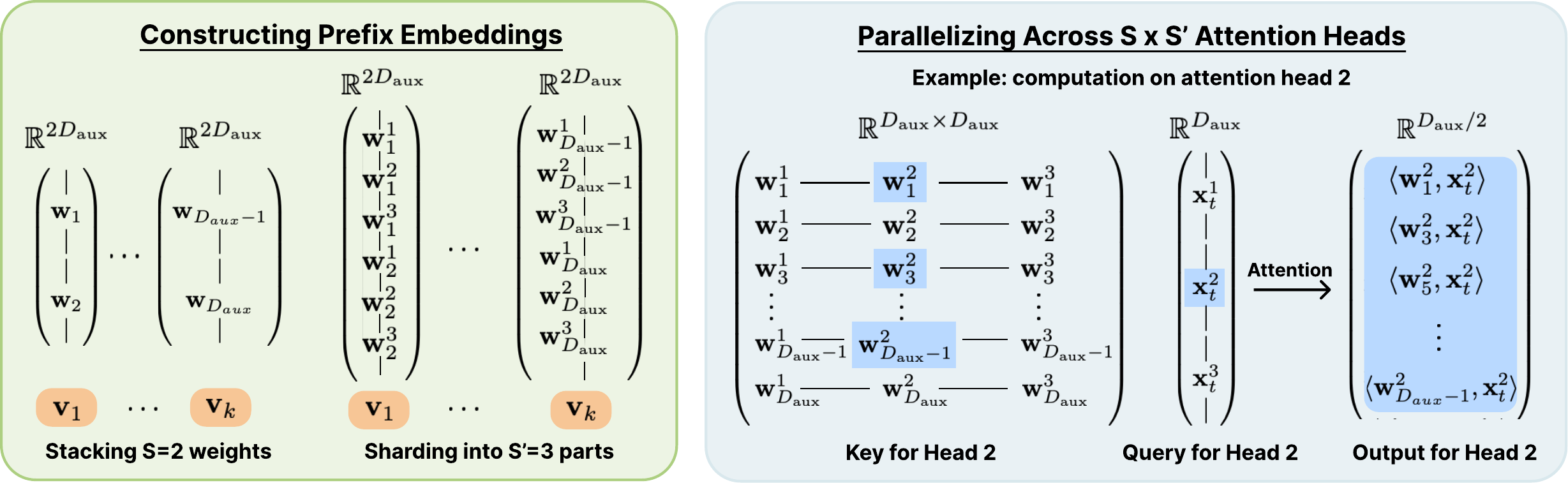}
    \caption{\simulator{} simulates the forward pass of a linear layer with a $\simu{H}$-head attention layer ($\simu{H}=6$ here). We stack $S$ weights per prefix embedding to reduce the number of prefix embeddings required ($S=2$ here). We furthermore shard each weight and token embedding $\vx_t$ into $S'$ shards and compute inner products of each shared in parallel using $S \times S'$  attention heads ($S'=3$ here). Please see~\Cref{sec:stack}.}
    \label{fig:linear_forward}
\end{figure*}

\section{Design Considerations}\label{sec:design}
Our goal is to construct a simulator that can train an auxiliary model over the course of an inference pass. 
This procedure requires four steps:
\begin{enumerate}[leftmargin=*]
	\item \textbf{Forward Pass}: A forward pass to compute the auxiliary model output $f(\xi;\aux{\vtheta})$ on training input $\xi$ and a loss $\Loss$.
	\item \textbf{Backward Pass}: Backpropagation to compute the gradient of the auxiliary model $\nabla_{\aux{\vtheta}} \Loss(f(\xi;\aux{\vtheta}))$.
	\item \textbf{Parameter Update}: Update the auxiliary model using gradient descent, setting $\aux{\vtheta}' = \aux{\vtheta} -\eta\nabla_{\aux{\vtheta}} \Loss(f(\xi;\aux{\vtheta}))$.
	\item \textbf{Output}: Output next-token predictions $f(\xi';\aux{\vtheta}')$ on a test input $\xi'$ using the updated auxiliary model. %Output these predictions as the simulator output.
\end{enumerate}
Note that steps 1-3 can be looped to train the auxiliary model for a few steps\footnote{Looping steps 1-3 scales the depth of the simulator model.}, either on the same training data or on different training data for each step, before evaluating it on the test input~\citep{giannou2023looped}. 
The above method highlight two crucial features of the simulator: (1) it has access to some amount of training data, and (2) it can use (i.e., read) and update (i.e., write) the auxiliary model.
Below, we discuss how to design a modest-sized simulator around these two considerations.

\subsection{Input structure}\label{sec:input_structure}
For simplicity, we describe only one update step on a single batch of training data $\xi$ but note that our formal construction and our experiments handle multiple training steps (see~\Cref{def:finetune}).
Steps 1 and 4 show that the simulator must access some training data $\xi$ to train the auxiliary model and some testing data $\xi'$ on which it evaluates the updated auxiliary model. 
For the sake of illustration we consider the following simple setting: given a sequence of input tokens $\ve_1,...,\ve_T$, we split it into training data $\xi=\ve_1,...,\ve_r$ and testing data $\xi'=\ve_{r+1},...,\ve_T$.

Suppose $\xi$ contains an in-context input-output exemplar and $\xi'$ contains a test input.
Then, the simulator performs a very natural operation of training the auxiliary model on a task-specific example and outputs results for the test example. 

On the other hand, if the input is not specially formatted, $\xi$ and $\xi'$ may simply contain some natural language tokens.
In this case, the simulator is using the first part of the context tokens to do a quick fine-tune of the auxiliary for some task before outputting the subsequent tokens with the auxiliary model. In a worst-case scenario, users might provide harmful contents, leading the model to implicitly fine-tune on them and potentially output even more harmful content.

Our experiments consider many options for splitting a sequence into $\xi$ and $\xi'$, and we defer a more detailed discussion of possible setups to~\Cref{sec:experiments}.

\looseness-1\paragraph{Accessing Training Labels.} 
The simulator must be able to see the labels of the training tokens 
in order to compute the loss $\cL$ (usually, the autoregressive cross-entropy loss) in step 1. 
For example, in \cref{fig:general_structure_simulator}, when we compute the loss for the token $\ve_2$ in the second position, we need to use its label $\ve_3$ in the third position.
However, this is not possible if the simulator uses strictly autoregressive attention (\cref{sec:lm_head} contains a more general discussion).
We thus use a bidirectional attention mask on the training tokens and autoregressive attention on the evaluation portion.
We note that encoding relevant (e.g., retrieved) context with bidirectional attention is a popular way to improve autoregressive capabilities in language modeling and natural language tasks~\citep{raffel2020exploring,borgeaud2022improving,izacard2020leveraging,izacard2023atlas,wang2023shall,tay2022ul2}.
This empirical approach is similar in motivation to how \simulator{} uses a few context tokens to adapt the auxiliary model to a given input.
Having established the training and testing data, we can now move to discussing how the simulator can access (i.e., read) and update (i.e., write to) the auxiliary model at inference time.  

\subsection{Read and write access to auxiliary model}
As discussed in the start of this section, the simulator must have read and write access to the parameters of the auxiliary model.
Crucially, the simulator must do at least two forward passes through the auxiliary model, one with the current parameters $\aux{\vtheta}$ and one with the updated parameters $\aux{\vtheta}'$.

The straightforward  way to simulate the forward pass of the auxiliary model would be to store its weights in the simulator's weights and run a forward pass as usual.
One can analogously simulate the backward pass according to the loss $\Loss$ to compute the gradients.
However, \textbf{the simulator cannot update its own weights at inference time}, so this strategy would not permit the model to write the updated parameters $\aux{\vtheta}'$ and later read them when simulating the second forward pass.
Therefore, the auxiliary model $\aux{\vtheta}$ must be available in the activations of the simulator. 

To this end, \citet{wei2022statistically,perez2021attention} model the simulator after a Turing machine, where the activation $\ve_t^{(\ell)} \in\RR^{\simu{D}}$ in each layer acts as a workspace for operations, and computation results are copied to and from memory using attention operations.
In this paradigm, if $\aux{D} = 768$, computing a dot product $\langle \vw, \vx_t^{(\ell)}\rangle$ with weight $\vw\in\RR^{768}$ requires at least $6.4$ million parameters in the simulator\footnote{
Using a feedforward module to mimic the dot product (as in \citet{akyurek2022learning}, see thm. \ref{thm:gelu_multiplication}), where the simulator embedding comprises $[\vw, \vx_t] \in \RR^{1536}$, necessitates a minimum of $4.7$ million parameters. Using an attention module to copy the weight from memory adds another $1.7$ million parameters.}.
Given the pervasiveness of dot products in neural network modules, this strategy would yield a simulator with trillions of parameters.

Alternatively, one can store parameters in the first few context tokens and allow the attention modules to attend to those tokens~\citep{giannou2023looped}. 
This removes the need for copying and token-wise operations. 
Then, the same dot product requires only a self-attention module with $1.7$ million parameters.
We thus adopt this strategy to provide relevant auxiliary model weights as \emph{prefix embeddings}.
\begin{definition}[Prefix Embeddings]
    $\{ \vv_j^{(\ell)} \}_{j=1}^K$ denotes the $K$ prefix embeddings at the $\ell$th layer  in \simulator . These contain \emph{relevant} auxiliary model weights or simulated activations.
    \label{def:prefix_embs}
\end{definition}
We now consider how to efficiently simulate the building block of neural networks: matrix-vector multiplication.
In the next section, we demonstrate that a careful construction of the prefix embeddings enables efficient parallelizaton of matrix-vector products across attention heads.

\input{construction_v2_linear}

%% file: figures/General_structure_ICL.tex
%\ifthenelse{\boolean{arxiv}}{
\begin{figure*}[t]
    \centering
    \includegraphics[width=\textwidth]{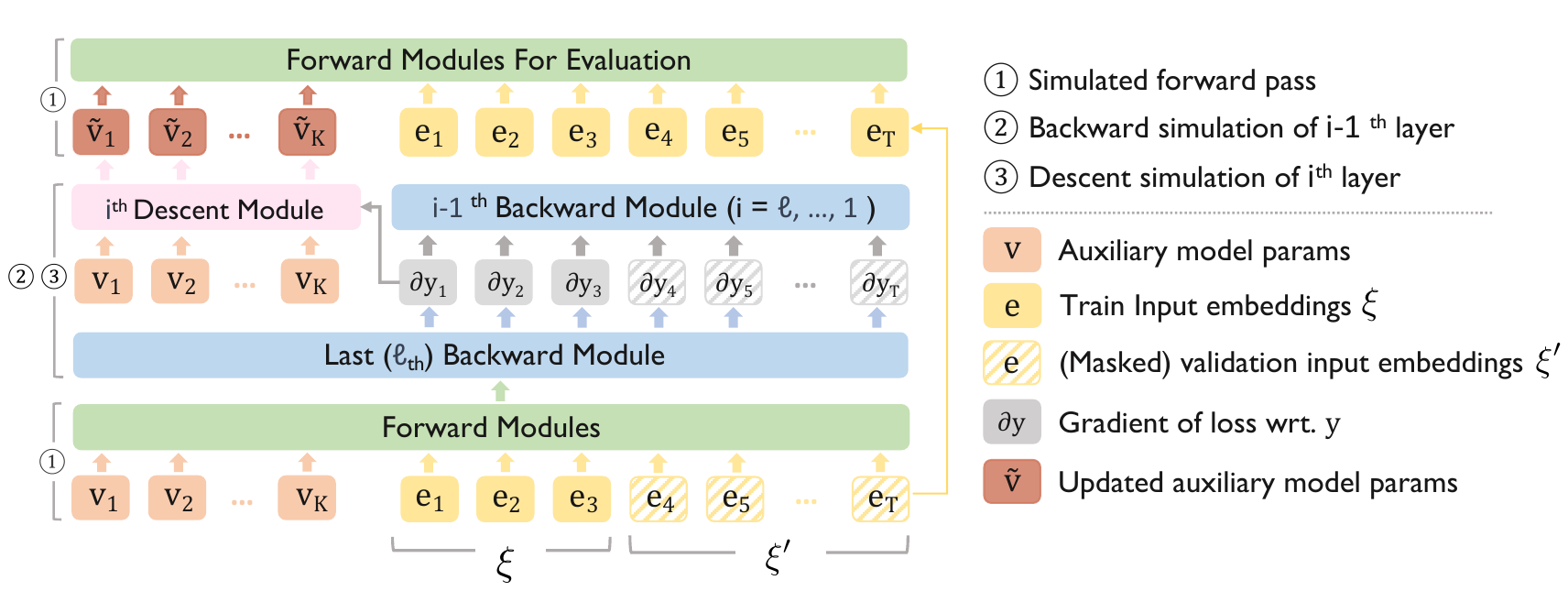}
    \caption{The overall structure of \textsc{TinT} (see~\Cref{sec:design} for an overview). Each forward, backward, and descent module is represented using combinations of linear, self-attention, layernorm, and activation layers.
    The input consists of prefix embeddings (\Cref{def:prefix_embs}) that represent relevant auxiliary model parameters in each layer followed by natural language input. A prefix mask separates the train and test segments of the input (\textsection\ref{sec:input_structure}). 
    %Auxiliary model parameters are updated in the descent module using the training part, and the updated prefix tokens are transferred to the forward modules via residual connections for evaluating the rest. 
    %\snote{Update figure to contain $\xi$ and $\xi'$ notation.}
    }
    \label{fig:general_structure_simulator}
\end{figure*}
%}{

\iffalse
\begin{figure}
    \centering
    \includegraphics[width=\textwidth]{figures/figure1.pdf}
    \caption{The overall structure of \textsc{TinT}. The input consists of prefix embeddings, that represent relevant auxiliary model parameters in each layer, input token embeddings, and a binary prefix mask to separate the train and evaluation segments of the input. The auxiliary model parameters are updated in the descent module using the training part of the segment, and the updated prefix tokens are transferred to the forward modules via residual connections for evaluating the rest of the segment. 
    }
\end{figure}
}
\fi

%% file: construction_v2_linear.tex
\section{Efficient Forward Propagation}\label{sec:exposition_linear_forward}
We now discuss how \simulator{} performs a highly efficient forward pass through the auxiliary model.
Here, we focus on the linear layer  because it is repeated many times in various transformer modules (e.g., in self-attention), so improving the efficiency dramatically reduces \simulator{}'s size.
\begin{definition}[Linear layer]\label{def:linear_mainpaper}
    For a weight $\mW\in\RR^{\aux{D} \times \aux{D}}$, a linear layer takes $\vx\in\RR^{\aux{D}}$ as input and outputs $\vy = \mW\vx$.\footnote{Linear layers are applied token-wise, so we can consider a single position $t$ without loss of generality.}
\end{definition}
We compute $\vy$ coordinate-wise, i.e., $\langle \vw_i, \vx_t\rangle$ for all $i \in [\aux{D}]$, where $\vw_i$ is the $i$th row of $\mW$. 
The simulator represents $\langle \vw_i, \vx_t\rangle$ as an attention score between the row $\vw_i$ and the input $\vx_t$. 
So, the input embeddings $\ve_t$ contain $\vx_t$ in the first $\aux{D}$ coordinates, and the rows $\{\vw_i\}$ of the weight matrix $\mW$ are in prefix embeddings $\{\vv_j\}$ (def. \ref{def:prefix_embs}).

We strategically distribute the weights (\textsection\ref{sec:stack}) and aggregate the parallelized computation results (\textsection\ref{sec:aggregate}).
As we briefly mentioned in the previous section, a straightforward construction of the linear layer would use the context and attention heads inefficiently.
Our construction instead parallelizes the computation across attention heads in such a way that aggregating the output of the linear operation can also be conducted efficiently.

\subsection{Stacking and Sharding}\label{sec:stack}
We partition the inner product computation across attention heads by carefully rearranging the weights and activations via stacking and sharding (\Cref{fig:linear_forward}).

Instead of representing each weight $\vw_i$ as its own prefix token $\vv_i$, we \emph{stack} $S$ weights on top of each other to form each prefix embedding $\vv_i$. $S$ drives a trade-off between the embedding dimension of the \simulator, $\simu{D}= \aux{D}S$, and the context length to the \simulator, $\simu{T} = K+\aux{T}$. 
We set $S=4$.

A simple strategy now would be to use different attention heads to operate on different rows; however, this would still use only $S$ attention heads whereas we could parallelize across many more heads.
We instead parallelize across more attention heads, where each head is responsible for computing the inner product on a subset of the coordinates. % while ensure the resulting output can easily be compiled:
We \emph{shard} each individual weight and the activation into $S'$ parts and compute the inner product on each of the $S'$ parts in parallel 
We set $S$ and $S'$ such that $\simu{H} = S\times S'$, thereby using all of \simulator{} heads to efficiently compute the dot products. 

\subsection{Efficient Aggregation}\label{sec:aggregate}
The attention module outputs a sparse matrix with shape $(\simu{D} / \simu{H})\times \simu{H}$ containing the inner products on various subsets of the coordinates in its entries.
To complete the linear forward pass, we need to sum the appropriate terms to form a $\simu{D}$-length vector with $\mW\vx$ in the first $\aux{D}$ coordinates.
Straightforwardly summing along an axis aggregates incorrect terms, since the model was sharded.
On the other hand, rearranging the matrix would require an additional $\simu{D}\times\simu{D}$ linear layer.
Instead, \simulator{} saves a factor of $\simu{H}\times$ parameters %\mz{anything missing here?} 
by leveraging the local structure of the attention output.
We illustrate this visually in \cref{sec:H-split_operation_app}.
This procedure requires $\simu{D}^2/\simu{H} + \simu{D}\simu{H}$ parameters. 
This efficient aggregation also compresses the constructions for the \simulator's backpropagation modules for layer normalization and activations (\cref{sec:LNappendix,sec:act_appendix}).

%% file: modifications_mainpaper.tex
\section{Simulated Gradient} \label{sec:modification}
\simulator{} adapts backpropagation to compute gradients 
~(\Cref{fig:general_structure_simulator}).
We aim to train a capable (i.e., pre-trained) auxiliary model for just a few steps, so high precision gradients may be unnecessary. Instead, \simulator{} performs an approximate backpropagation.  
\simulator{} then uses this \emph{simulated gradient} to update the auxiliary model.
Prior works computed similar approximate gradients in hopes of more faithfully modeling neurobiology~\citep{scellier2017equilibrium,hinton2022forwardforward} or improving the efficiency of training models~\citep{hu2021lora,malladi2023finetuning}.
\iffalse In this section, we describe how \simulator{} performs an approximate backpropagation to compute a \emph{simulated gradient} of the auxiliary model. Due to the complexity of transformer modules, computing {\em exact} gradients of a transformer auxiliary model can require a huge simulator model. Our experiments show that the simulated gradient is a reasonable approximation. \fi 
\iffalse that makes \simulator{} much more expressive than prior constructions despite being orders of magnitude smaller and defer the complete formal construction to the appendix.\fi 
We note that the approximations in the simulated gradients can be made stronger at the cost of enlarging \simulator{}.
Indeed, one could construct a simulator to \emph{exactly} perform the procedure outlined in \textsection{\ref{sec:design}}, though it would be orders of magnitude larger than \simulator{}.
For brevity's sake, we focus on the key approximations and design choices and defer formal details to the appendix.

\subsection{First-order approximations} \label{sec:first_order_approx}
We use first-order approximations of gradients to backpropagate through 
%training dynamics to update the parameters of 
the layer normalization layer.\footnote{We discuss a layer normalization layer $\layernorm$ without scale and bias parameters, but \cref{sec:LNappendix} contains a general construction.} 
It normalizes the input using its mean and standard deviation across the input dimensions. Since the operation is token-wise, we can consider a single position $t$ without loss of generality.

\begin{definition}[Layer normalization]
     A layer normalization layer $\layernorm$  takes input $\vx\in\RR^{\aux{D}}$ and outputs $\vy = (\vx - \mu) / \sigma$, where $\mu$ and $\sigma$ denote its mean and standard deviation.
\end{definition}
\textbf{High precision gradients:} 
Formally, for input-output pair $(\vx, \vy)$, we can compute the  gradients $\partial_{\vy}$, $\partial_{\vx}$ with  chain rule:
\begin{gather}
    \partial_{\vx} = \left(\frac{\partial \layernorm (\vx)}{\partial \vx} \right)^{\top} \losspartial{\vy}  \nonumber \\= \frac{1}{\sigma} \left(  \langle \losspartial{\vy}, \vy \rangle \vy + \losspartial{\vy} - \frac{1}{\aux{D}} \sum_{i=1}^{\aux{D}} \losspartial{y_i} \right). \label{eq:ln_informal}
\end{gather}
\textbf{Inefficiency of exact computation:} 
A \simulator{} layer simulating backpropagation through an auxiliary's layer normalization layer receives  $\losspartial{{\vy}_t}$ and $\vx_t$ in its input embeddings. We go through the exact gradient and why it is inefficient.

For exact computation one could first compute $\vy_t$ using a normalization layer and store in the embeddings. However, inefficiency arises from computing the term  $\langle \losspartial{{{\vy}_t}}, \vy_t \rangle \vy_t$.
To calculate $\langle \losspartial{{{\vy}_t}}, \vy_t \rangle \vy_t$ at each token position $t$, we could either: (1) use a two-layer MLP that focuses on each token separately, or (2) a single self-attention module to treat the operation as a sequence-to-sequence task. 

For (1) we could initially compute $\langle \losspartial{{\vy}_t}, \vy_t \rangle$ via an MLP, followed by computation of $\langle \losspartial{{\vy}_t}, \vy_t \rangle \vy_t$ using another MLP. The element-wise multiplication in embeddings would be facilitated with a nonlinear activation function like \gelu{}~\citep{akyurek2022learning} (refer to thm. \ref{thm:gelu_multiplication} for details). However, this approach would need  substantial number of simulator parameters to represent the MLPs.

\looseness-1Alternatively, we could  use a single self-attention module. Constructing such a module would require careful engineering to make sure the input tokens only attend to themselves while keeping an attention score of $0$ to others. If we used a linear attention, we would need to space out the gradient $\losspartial{{\vy}_t}$ and $\vx_t$ in each position $t$, such that the attention score is $0$ between different tokens. This would require an embedding dimension proportional to the context length. 
On the other hand, if we used a softmax attention module, we would need an additional superfluous token in the sequence. Then, a token at position $t$ would attend to itself with attention $\langle \partial{\vy}_t, \vy_t\rangle$ and to the extra token with an attention score of $1-\langle \partial{\vy}_t, \vy_t\rangle$. The extra token would return a value vector $0$. \sm{Rephrase the prior sentence...I can't understand how this would work} To avoid such inefficiency, we opt for a first-order approximation instead.

\textbf{Efficient approximation:} 
Instead of explicitly computing each term in the chain rule of $\left(\frac{\partial \layernorm (\vx)}{\partial \vx} \right)^{\top} \losspartial{\vy}$ in Eq. \ref{eq:ln_informal}, we instead use a first order Taylor expansion of $\layernorm$.
\begin{align*}
    \layernorm(\vx + \epsilon \losspartial{\vy}) = \layernorm(\vx) + \epsilon \left(\frac{\partial \layernorm (\vx)}{\partial \vx} \right) \losspartial{\vy} + \mathcal{O}(\epsilon^2).
\end{align*}
Rearranging allows us to write
\begin{align*}
    \left(\frac{\partial \layernorm (\vx)}{\partial \vx} \right) \losspartial{\vy} =  \frac{1}{\epsilon} \left( \layernorm (\vx + \epsilon \partial_{\vy}) - \layernorm (\vx) \right) + \mathcal{O}(\epsilon).
\end{align*}
Similar to the computation of Eq. \ref{eq:ln_informal}, we can show  
\begin{align*}
    \frac{\partial \layernorm (\vx)}{\partial \vx} = \frac{1}{\sigma} \left( (1-{\aux{D}}^{-1})\mI - \layernorm(\vx) \layernorm(\vx)^{\top} \right).
\end{align*}
Because  $\partial \layernorm (\vx) / \partial \vx$ is symmetric\footnote{For a linear function $f$ with matrix $\mW$, $\frac{\partial f (\vx)}{\partial \vx} = \mW$. Since $\mW$ may not be a symmetric matrix, this method can't be generally applied to approximately backpropagate linear layers or causal self-attention layers.}, we can write
\begin{align*}
    \left(\frac{\partial \layernorm (\vx)}{\partial \vx} \right)^{\top} \losspartial{\vy} & =  \left(\frac{\partial \layernorm (\vx)}{\partial \vx} \right) \losspartial{\vy} \\& =  \frac{1}{\epsilon} \left( \layernorm (\vx + \epsilon \partial_{\vy}) - \layernorm (\vx) \right) + \mathcal{O}(\epsilon).
\end{align*}

Then, ignoring the small error term, we can use just two linear layers, separated by a normalization layer, to simulate the approximation.

\subsection{Fuzzy backpropagation via stop gradients} \label{sec:stop_gradients}
Self-attention is inherently quadratic, because it uses the keys and queries to compute attention scores between every possible pair of tokens in the sequence.
These scores then linearly combine the value vectors (see def. \ref{def:self-attn_auxiliary_single}). 
Computing the gradient exactly is thus a very complex operation.
Instead, we stop the gradient computation through attention scores in the self-attention layer. 
For similar reasons, we only update the value parameter in the self-attention module.

\textbf{Gradient backpropagation:}
For an input, output sequence pair $\{\vy_t\}, \{\vy_t\}$, if $\{\vq_t, \vk_t, \vv_t\}$ denote the intermediate query, key, value vectors, on gradients $\{ \partial_{\vy_t} \}$, $\{ \partial_{\vx_t} \}$ is given via the chain rule:
\begin{align}
    \losspartial{\vx_t} &= \mQ^{\top} \losspartial{\vq_t} + \mK^{\top} \losspartial{\vk_t} + \mV^{\top} \losspartial{\vv_t}.
    \label{eq:attn_backprop_informal}
\end{align}
Here, $\mV, \mK, \mQ$ denote the query, key, and value matrices. 

\textbf{Inefficiency in exact computation:} 
Here, we demonstrate that simulating computation of the three terms in Eq. \ref{eq:attn_backprop_informal} is inefficient, because
$\losspartial{\vq_t}, \losspartial{\vk_t}$ depend on the derivatives w.r.t. the attention scores. 
As an example, we focus on $\losspartial{\vk_t}$:
\begin{align*}
    \losspartial{\vk_t} = \sum_j a_{t, j} ( (\losspartial{\vy_t})^{\top} \vv_j ) (\vk_j - \sum_{j'} a_{t, j'} \vk_{j'}).
\end{align*}

Computing this term would require us at least 2 self-attention layers and an MLP layer. The first attention layer would compute $(\losspartial{\vy_t})^{\top} \vv_j$ for different token pairs, similar to the forward simulation of a linear layer with linear attention (\textsection{\ref{sec:exposition_linear_forward}}). These would be then multiplied to the pair-wise attention scores $a_{t, j}$  with an MLP to compute $a_{t, j} ( (\losspartial{\vy_t})^{\top} \vv_j )$, with elementwise product would be facilitated by GeLU non-linearity (thm. \ref{thm:gelu_multiplication}). 
These would be finally used by an attention layer to combine the different key vectors. A similar simulation would be necessary to compute $\losspartial{\vq_t}$.

\textbf{Stop gradients through query and key vectors:}
In order to reduce the necessary resources, we ignore the query and key gradients in Eq. \ref{eq:attn_backprop_informal}. When we ignore these gradient components,  $\{ \partial_{\vx_t} \}$ can be simplified as
\begin{align}
    \losspartial{\vx_t} \approx \mV^{\top} \losspartial{\vv_t} = \mV^{\top} \sum_{j} a_{j, t} \losspartial{\vy_t}.
    \label{eq:approx_attn_gradient}
\end{align}
A single self-attention layer can compute this by using the attention scores to combine the token-wise gradients.

\textbf{Why won't it hurt performance?} Estimating $\partial_{\vx_t}$ as described is motivated by recent work \citep{malladi2023finetuning} showing that fuzzy gradient estimates don't adversely affect fine-tuning of pre-trained models. Furthermore, we theoretically show that when the attention head for each position pays a lot of attention to a single token (i.e., behaves like hard attention \citep{perez2021attention}), the approximate gradient in Eq. \ref{eq:approx_attn_gradient} is entry-wise close to the true gradients (thm. \ref{thm:attn_backprop}).

The other approximation is to update only the value parameters $\mV$ of the auxiliary model (\textsection{\ref{sec:self-attnt-backprop_appendix}}). 
This is motivated by parameter efficient fine-tuning methods like \lora{} \cite{hu2021lora} and \ia{} \cite{liu2022few}, which restrict the expressivity of the gradient updates without degrading the quality of the resulting model.
We similarly show in the next section that the simulated gradients in \simulator{} can effectively tune large pre-trained transformers.

%% file: experiments.tex
\input{Tables/LM_table}

\begin{table}[!t]
\centering
\caption{\looseness-1Language modeling results on \textsc{WikiText-103}. We use $30\%, 50\%, 70\%$ and $90\%$ of sequences for training in the language modeling setting (\textsection{\ref{setting:lm}}). $\textsc{TinT}$
improves the auxiliary model perplexities by $0.3-0.7$ absolute on average.
The small perplexity difference between the \simulator{} and explicitly updating the auxiliary model suggests that the simulated gradient (\cref{sec:modification}) can still effectively fine-tune the auxiliary model. 
%\sm{Update caption to clearly reference the one-off definition. Change table rows to auxiliray model, one-off dynamic eval, and tint. How many gradient steps is TinT encoding here? Standardize the fonts (small caps vs normal). Check if captions go above or below the table} 
}
\label{tab:language_modeling}
\resizebox{\columnwidth}{!}{%
\begin{tabular}{@{}llccccc@{}} \toprule
& & \multicolumn{4}{c}{Training proportion} \\
\cmidrule(lr){3-6} 
& Evaluating with& $30\%$   & $50\%$   & $70\%$   & $90\%$ \\
 % \multicolumn{5}{c}{Auxiliary model: \textsc{GPT2}}    \\ 
 \cmidrule(r){1-2} \cmidrule(lr){3-6}
%\cmidrule(r){1-5}
%                Model & \multicolumn{4}{c}{\textsc{GPT2}}        \\ 
%              \cmidrule(r){1-1} \cmidrule(lr){2-5} 
%\cmidrule(lr){6-9}
\multirow{3}{*}{GPT-2} & \text{Auxiliary Model}    & 25.6 & 24.9  & 24.5 & 23.3 \\
 &\text{Fine-tuning}     & 24.9 & 24.0    & 23.5 & 22.2  \\
 & \simulator{}           & 25.1 & 24.3 & 23.8 & 22.6 \\  
%\midrule
 %\cmidrule(r){1-1} \cmidrule(r){2-5}
%\cmidrule(r){1-1} 
% \cmidrule(r){2-5}
%Model & \multicolumn{4}{c}{\textsc{OPT-125m}} \\
%\cmidrule(r){1-1} \cmidrule(r){2-5}
\cmidrule(r){1-2} \cmidrule(lr){3-6} 
 % \multicolumn{5}{c}{Auxiliary model: \textsc{OPT-125m}} \\
% Auxiliary model: \textsc{OPT-125m}
\multirow{3}{*}{OPT-125M} & \text{Auxiliary Model} & 29.6 & 28.8 & 28.0 & 28.0  \\
& \text{Fine-tuning} & 29.0 & 28.2 & 27.4 & 27.4 \\
& \simulator{}  & 29.3 & 28.4  & 27.5 & 27.4  \\
 \bottomrule
\end{tabular}}
\end{table}

% Our approach provides constructions for diverse variants of pre-trained language models. 
% \Cref{tab:construction} highlights many types of modules and the required size and computation for each.
% %We showcase examples in \Cref{tab:construction} of the types of modules and their required size and computation. 
% The size of a constructed model is influenced by various factors, including the number of layers, and embedding dimension in the auxiliary. 
% We demonstrate the effectiveness of constructed models through language modeling and in-context learning tasks. We evaluate the \simulator{} construction on  \textsc{OPT-125m} model. 

%We conduct experiments to validate the efficacy of our approach. 
%Through these experiments, we demonstrate that the approximations used in constructing the model do not impact model performance.

%We conduct experiments to validate that the approximations introduced (\Cref{sec:modification}) do not significantly harm the capability of \simulator{} to simulate and train an internal model.

% \footnote{Performing many gradient steps can scale the depth of \simulator{} to make the experiments computationally infeasible.}

 \begin{figure*}[t]
    \centering
    \includegraphics[width=0.75\textwidth]{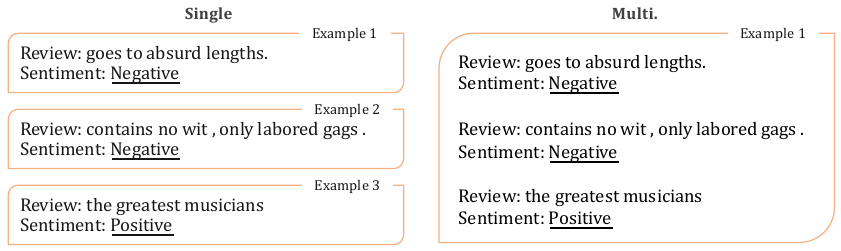}
    \caption{Different settings in few-shot learning ($k=3$) using \simulator{}. The \textbf{Single} mode (left) treats each example as a training datapoint, and the auxiliary model is updated with a batch of inputs (see def.~\ref{def:finetune}). The \textbf{Multi.} mode (right) concatenates all examples to form a single input and uses batch size $1$ in def.~\ref{def:finetune}. For \textbf{Label loss}, only underlined label words are used as training signal, while \textbf{full context loss} includes all tokens.}
    \label{fig:example}
\end{figure*}

\section{Experiments}\label{sec:experiments}
\input{Tables/Downstream_table}

We evaluate the performance of the \simulator{}s constructed using GPT2 and \textsc{OPT-125M} as auxiliary models. 
The findings from our experiments in the  language modeling and in-context learning settings confirm that fine-tuning with the simulated gradients (\Cref{sec:modification}) still allows for effective learning in the auxiliary model.
We loop the training steps (i.e., steps 1-3) outlined in \Cref{sec:design} to accommodate solving real-world natural language tasks.
We formalize the setting below.
%In particular, we allow steps 1-3 (i.e., training the auxiliary model on input data) to be looped.

\subsection{Setting: $N$-step Fine-Tuning}
We formalize the procedure in~\Cref{sec:design} to construct a suitable setting in which we can compare \simulator{} to explicitly training the auxiliary model.
\begin{definition}[$N$-step Fine-Tuning]
    \label{def:finetune}
    Given a batch of training datapoints $\xi_1, \cdots, \xi_B$ and a validation input $\xi'$, we compute and apply gradient updates on the auxiliary model $\aux{\vtheta}$ for timesteps $t=0,...,N-1$ as
    $$ \aux{\vtheta}^{t+1} = \aux{\vtheta}^t - \eta \sum_{i=1}^{B} \nabla_\vtheta \cL(f(\xi_i; \aux{\vtheta}^t)) $$
    where $\eta$ is the learning rate and $\cL$ is a self-supervised loss function on each input $\xi_i$.
    Then, we evaluate the model $\aux{\vtheta}^N$ on $\xi'$.
    $\aux{\vtheta}^0$ denotes the pre-trained auxiliary model.
\end{definition}

Below, we instantiate this setting with text inputs of different formats and different self-supervised loss functions $\cL$. 
To manage computational demands, we limit $N$ to $3$ or fewer.\footnote{Performing many gradient steps scales the depth of \simulator{} and makes experimentation computationally infeasible.}

\subsection{Case Study: Language Modeling}\label{setting:lm}

The first case we consider is language modeling, where the input data $\ve_1,...,\ve_T$ is natural language without any additional formatting. 
%We consider $\Loss$ to be the self-supervised autoregressive next-token prediction task (\cref{eq:auto_loss}). 
We use a batch size of $1$ in def.~\ref{def:finetune}, and delegate $\xi_1=\ve_1,...,\ve_t$ and $\xi'=\ve_{t+1},...,\ve_T$. 
The loss $\cL$ is the sum of the token-wise autoregressive cross-entropy loss in the sequence $\xi_1$. %\footnote{It is defined as -\sum_{} \log f(x_i) }
%We use the first $t$ tokens as the training data $\xi$, and the rest of the tokens $\ve_{t+1}, \cdots, \ve_T$ as evaluation data $\xi'$. 
For example, given an input \textcolor{red}{Machine learning is a useful tool} \textcolor{brown}{for solving problems.}, we use the \textcolor{red}{red} part as the training data $\xi_1$, and the \textcolor{brown}{brown} part as the validation data $\xi'$. 
We perform language modeling experiments on \textsc{WikiText-103}~\citep{merity2016pointer} and vary the number of tokens $t$ used as training data $\xi$. 

\textbf{Results.}
In \cref{tab:language_modeling}, we observe that \simulator{} achieves a performance comparable to explicit fine-tuning of the auxiliary model, indicating that the simulated gradient~(\Cref{sec:modification}) is largely effective for fine-tuning. Both \simulator{} and explicitly fine-tuning the auxiliary model show improvement over the base model, confirming that minimal tuning on the context indeed enhances predictions on the test portion. 
% \simulator{} achieves lower perplexity lower than the auxiliary model that it was built on. 
\subsection{Case Study: In-Context Learning}
For in-context learning, we consider input data to be a supervised classification task transformed into a next-token prediction task using surrogate labels (see~\Cref{fig:example}). Using binary sentiment classification of movie reviews as an example, given an input (e.g., the review), the model's predicted label is computed as follows. First, we design a simple task-specific prompt (e.g., ``Sentiment:'') and select label words $c_1,...,c_n$ to serve as surrogates for each class (e.g., ``positive'' and ``negative''). Then, we provide the input along with the prompt to the model, and the label assigned the highest probability is treated as the model's prediction.  We describe the zero-shot and few-shot settings below.

\textbf{Zero-shot.}
In the zero-shot setting, we are given text with the first $T-1$ tokens as the input text and final token as the surrogate text label. Hence, we adapt def.~\ref{def:finetune} to use batch size $B=1$, training data $\xi_1=x_1,...,x_{T-1}$, and testing data $\xi' = x_T$.
The loss $\cL$ is again the sum of the token-wise autoregressive cross-entropy losses. 
%\mz{if you move sum of loss to footnote, maybe you need to edit here as well.}

\textbf{Few-shot.}
In the few-shot setting, we are given input texts that are a concatenation of $k$ sequences $\xi_1, \cdots, \xi_k$.
Each sequence contains the input text followed by the surrogate label for the in-context exemplar.
These $k$ exemplars are followed by test data $\xi'$. 
In this case, we can compute the gradient updates to $\aux{\vtheta}$ in two different ways~(\cref{fig:example}). 
The first setting, denoted \textbf{Single}, treats the $k$ sequences as a batch of $B=k$ training datapoints $\xi_1,...,\xi_B$.
The second setting, denoted \textbf{Multi}, treats the concatenation of the $B$ sequences as a single training datapoint $\xi_1$.
%These two settings are referred to as \textbf{Single} and \textbf{Multi} mode setting . 
Furthermore, $\Loss$ for a training datapoint can be defined in two different ways. The first setting, denoted as \textbf{Full context loss}, defines $\Loss$ for a training datapoint $\xi_i$ as the sum of cross entropy loss over all tokens. The second setting, denoted as \textbf{Label loss}, defines $\Loss$ for a training datapoint $\xi_i$ in def. \ref{def:finetune} as the sum of cross entropy loss over the surrogate label tokens.

\textbf{Tasks.} 
We evaluate 7 classification tasks for zero-shot and few-shot settings: 
%For language modeling, we use \textsc{Wikitext-103}~\citep{merity2016pointer}. , we evaluate on $7$ downstream classification tasks: 
\textsc{SST-2} \citep{socher2013recursive}, MR \citep{pang2004sentimental}, CR \citep{hu2004mining}, MPQA \citep{wiebe2005annotating}, Amazon Polarity \citep{zhang2015character}, AGNews \citep{zhang2015character}, and Subj \citep{pang2005seeing}.

%Our methodology draws inspiration from dynamic evaluation \citep{krause2019dynamic}, where a segment of the input sequence is used to update the auxiliary model, and the remaining portion is used to assess the updated auxiliary model's performance. 

\textbf{Model.} We compare a \simulator{} model that uses an \textsc{OPT-125m} pre-trained model as its auxiliary model against two alternative approaches: (1) directly fine-tuning OPT-125m, and (2) performing standard evaluation using OPT-1.3b, which is of a similar size to \simulator{}.\footnote{Our construction is generally applicable to diverse variants of pre-trained language models (\cref{app:other_construction}).}

\iffalse
Mathematically, consider an input token sequence $\{s_{t}\}_{t \le T}$. If $\Pr[y \vert \{s_{t}\}_{t \le T}]$ is the model's predicted probability for a given label $y$ and $\Pr[y \vert \Phi]$ is its predicted probability when presented with a simple task-specific context, the model's predicted label on the sequence $\{s_{t}\}_{t \le T}$ with calibration is computed as
\begin{align*}
    \arg\max_y \frac{ \Pr[y \vert \{s_{t}\}_{t \le T}] } {\Pr[y \vert \Phi]}.
\end{align*}
We use $\Phi$ as the set of prompt tokens that we append to each input. For example, in \cref{fig:example}, we use $\Phi$ as the set of tokens involved in the phrase "Sentiment:". 
\fi

\textbf{Observations.} We observe that inferences passes through \simulator{} perform on par with directly fine-tuning the auxiliary model, affirming the validity of the construction design (see~\Cref{sec:design}). As expected, \simulator{} outperforms the base auxiliary model, since it simulates training the auxiliary model. More intriguingly, \simulator{} demonstrates performance comparable to a pre-trained model of similar size (\textsc{OPT-1.3b}). This suggests that the capabilities of existing pre-trained models may be understood via the simulation of smaller auxiliary models. For further details and results of the experiments, please refer to~\Cref{app:experiment}.

%suggesting that standard pre-trained transformer models may inherently employ a similar mechanism to internally train an auxiliary model during inference.

% For downstream tasks (\cref{tab:main_result}), explicit internal training within \simulator{} surpasses vanilla zero-shot evaluation and in-context learning, even with a limited budget of a single forward pass. Moreover, \simulator{} achieves a performance comparable to finetuning the auxiliary model, indicating that the approximations made during its construction largely preserve its effectiveness for fine-tuning. Additionally, we find that \simulator{} outperforms or is on par with a similarly sized pre-trained model (\textsc{opt-1.3b}) . This suggests that the capabilities of existing pre-trained models may be understood via the simulation of smaller auxiliary models. Refer to \cref{app:experiment} for more experiment details and results. 

% We also report the comparisons with calibration in \cref{tab:main_result_app} in the appendix.

%% file: Tables/LM_table.tex
\iffalse
\begin{table*}[t]
\centering
\caption{\looseness-1Language modeling results on \textsc{WikiText-103}. We use $30\%, 50\%, 70\%$ and $90\%$ of sequences for training in dynamic eval and \textsc{TinT} and the rest of the sequence for evaluation. $\textsc{TinT}$
improves upon the auxiliary model perplexities by $0.3-0.7$ absolute on average.
The small perplexity difference between the \simulator and dynamic evaluation suggests that the approximations introduced in the descent algorithm (\cref{sec:modification}) can still effectively fine-tune the auxiliary model. \sm{Update caption to clearly reference the one-off definition. Change table rows to auxiliray model, one-off dynamic eval, and tint. How many gradient steps is TinT encoding here? Standardize the fonts (small caps vs normal). Check if captions go above or below the table} }
\label{tab:language_modeling}
\begin{tabular}{@{}lcccccccc@{}} \toprule
                 & \multicolumn{4}{c}{\textsc{GPT2}}      & \multicolumn{4}{c}{\textsc{OPT-125m}}  \\ 
              \cmidrule(lr){2-5} \cmidrule(lr){6-9}
Training proportion & $30\%$   & $50\%$   & $70\%$   & $90\%$   & $30\%$   & $50\%$   & $70\%$   & $90\%$   \\ 
\cmidrule(r){1-1} \cmidrule(lr){2-5}\cmidrule(lr){6-9}
\textsc{Vanilla Model}    & 25.6 & 24.9  & 24.5 & 23.3 & 29.6 & 28.8 & 28.0 & 28.0 \\
\textsc{Dyna. FT}     & 24.9 & 24.0    & 23.5 & 22.2 & 29.0 & 28.2 & 27.4 & 27.4 \\
\textsc{TinT}             & 25.1 & 24.3 & 23.8 & 22.6 & 29.3 & 28.4  & 27.5 & 27.4 \\ \bottomrule
\end{tabular}
\end{table*}
\fi

%% file: Tables/Downstream_table.tex
\begin{table*}[t]
  \centering
  \caption{Zero-shot and few-shot in-context learning results across $7$ downstream tasks. All the few-shot results are averaged over three training seeds. \textsc{TinT} consistently surpasses its auxiliary model and achieves comparable performance to one-off dynamic evaluation. \textsc{TinT} outperforms auxiliary models by $3-4\%$ and $12-16\%$ absolute points on average in $0$-shot and  $32$-shot experiments respectively. \textsc{TinT} performs competitively with a similar-sized pre-trained model (\textsc{opt-1.3b}) in both $0$-shot and $32$-shot settings. We show the standard deviation for few-shot settings in parentheses.
  }
  \label{tab:main_result}
  \resizebox{\textwidth}{!}{
  \begin{tabular}{lc|cccccccc}
    \toprule
    \textbf{Model} & \textbf{Shots} & \textbf{Subj} & \textbf{AGNews} & \textbf{SST2} & \textbf{CR} & \textbf{MR} & \textbf{MPQA} & \textbf{Amazon} & \textbf{Avg.} \\
    \midrule
%\multicolumn{2}{c}{} & \multicolumn{8}{c}{\textbf{\textit{Without Calibration}}} \\ %\midrule
\textsc{OPT-125m} & $0$ & $64.0$ & $66.0$ & $70.5$ & $64.5$ & $71.0$ & $68.0$ & $76.5$ & $68.6$ \\
\textsc{OPT-1.3b} & $0$ & $59.0$ & $55.5$ & $54.0$ & $50.5$ & $52.5$ & $74.0$ & $57.0$ & $57.5$ \\
\textsc{OPT-125m} Fine-tuning  & $0$ & $71.0$ & $67.0$ & $79.5$ & $71.5$ & $70.0$ & $68.0$ & $85.5$ & $73.2$ \\
\rowcolor{gray!10}\textsc{OPT-125m TinT} & $0$ & $67.5$ & $66.0$ & $76.5$ & $69.0$ & $76.0$ & $70.5$ & $78.5$ & $72.0$ \\
 \midrule
\textsc{OPT-125m} & $32$ & $58.7_{(4.9)}$ & $33.7_{(8.4)}$ & $50.8_{(1.2)}$ & $51.3_{(1.9)}$ & $50.0_{(0.0)}$ & $54.3_{(2.5)}$ & $55.0_{(6.7)}$ & $50.5_{(1.9)}$ \\
\textsc{OPT-1.3b} & $32$ & $74.2_{(6.1)}$ & $71.3_{(5.3)}$ & $89.8_{(3.6)}$ & $71.5_{(4.5)}$ & $68.3_{(6.1)}$ & $81.7_{(3.3)}$ & $70.3_{(9.9)}$ & $75.3_{(0.4)}$  \\ 
\textsc{OPT-125m} Fine-tuning & $32$ & $78.0_{(1.4)}$ & $66.7_{(1.6)}$ & $71.5_{(1.4)}$ & $73.7_{(3.3)}$ & $72.0_{(0.0)}$ & $80.7_{(0.6)}$ & $79.8_{(0.2)}$ & $74.6_{(2.7)}$ \\
\rowcolor{gray!10}\textsc{OPT-125m TinT} & $32$ & 
$82.3_{(2.7)}$ & $69.3_{(0.9)}$ & $73.7_{(0.8)}$ & $75.7_{(1.9)}$ & $72.3_{(1.2)}$ & $83.2_{(1.0)}$ & $78.2_{(0.2)}$ & $76.4_{(0.7)}$ \\ 
\bottomrule
  \end{tabular}}
\end{table*}

%% file: related_works.tex
\section{Related Work}
\vspace{-0.2em}
\textbf{Gradient-based learning and in-context learning:} Several works relate in-context learning to gradient-based learning algorithms. \citet{bai2023transformers} explicitly constructed transformers to simulate simple gradient-based learning algorithms. \citet{mahankali2023one,ahn2023transformers} suggested one attention layer mimics gradient descent on a linear layer, and \citet{zhang2023trained} showed polynomial convergence. 
\citet{cheng2023transformers,han2023context} extended these ideas to non-linear attentions. Experiments in \citet{dai2022gpt} suggest that LLM activations during in-context learning mirror fine-tuned models. 
These works focus on using a standard transformer for the simulator and hence cannot accommodate more complex auxiliary models; on the other hand, our work uses structural modifications and approximations to construct an efficient simulator for complex auxiliary models.
Our work in contrast attempts to build even stronger transformers by introducing few structural modifications that can run gradient descent on auxiliary transformers.

\textbf{Transformer Expressivity:}
\citet{perez2021attention,pérez2018on} show that Transformers with hard attention are Turing complete, and 
\citet{wei2022statistically} construct transformers to study statistical learnability, but the proposed constructions are extremely large.
%In \cref{sec:design}, we point out that this scheme often results in gigantic constructions.
Other works have investigated encoding specific algorithms in smaller simulators, e.g. bounded-depth Dyck languages \citep{yao2021self}, modular prefix sums \citep{anil2022exploring}, adders \citep{nanda2023progress}, regular languages \citep{bhattamishra2020ability}, and sparse logical predicates \citep{edelman2022inductive}. \citet{liu2023transformers} aim to understand automata-like mechanisms within transformers.
\citet{ba2016using} connect self-attention and fast weight programmers (FWPs), which compute input-dependent weight updates during inference. Follow-up works~\citep{schlag2021linear,irie2021going} use self-attention layers to update linear and recurrent networks during inference. 
\citet{clark2022meta} add and efficiently tune Fast Weights Layers (FWL) on a frozen pre-trained model.

%% file: Discussion.tex
%\vspace{-0.5em}
\section{Discussion}
%\vspace{-0.5em}

%{\sc try to say something novel, not repeat the abstract. Example:  perhaps our architectural changes could help usual transformer in vanilla pre-training, which we hope to explore.}

We present a parameter-efficient construction \simulator{} capable of simulating gradient descent on an internal transformer model during inference. 
Using fewer than 2 billion parameters, it can simulate fine-tuning a 125 million transformer (e.g., GPT-2) internally, dramatically reducing the scale required by previous works. Language modeling and in-context learning experiments demonstrate that the efficient approximations still allow the \simulator{} to fine-tune the model. Our work emphasizes that the inference behavior of complex models may rely on the training dynamics of smaller models.
As such, the existence of \simulator{} has strong implications for interpretability and AI alignment research. 

%The approximations and architectural modifications in \simulator{} have potential value for future architectural development and applications such as pre-training and instruction tuning. Additionally, o

While our work represents a significant improvement over previous simulations in terms of auxiliary model complexity, similar to prior research in this area, our insights into existing pre-trained models are limited. Furthermore, we have not yet examined potential biases that may arise in the auxiliary models due to one-step gradient descent. We plan to investigate these aspects in future work.

\iffalse
\paragraph{A Cognitive Science View of \simulator}
 \simulator{} provides interesting insight into whether or not backpropagation is a biologically plausible learning algorithm. 
For example, prior works have suggested that backpropagation is difficult to execute in real neurons because it requires access to a backward computation graph that mirrors the forward propagation.
Hinton further argued that it is unlikely that the brain learns sequences through backpropagation at time, because ``the perceptual system needs to perform inference and learning in real time without stopping to perform backpropagation.''
The efficient construction of \simulator{} enables it to indeed perform simulation and training in a single inference pass, suggesting that the brain implementing backpropagation is perhaps more plausible than previously believed.  \apnote{I don't think this is true, since we do forward and backprop on train split, before inferring on the remaining input. I think Hinton is trying to say that the brain is unlikely to stop at a point to backprop} 
\fi 
%non-standard use of attention mechanisms to enable arithmetic parallelism

%Despite the model being in inference mode, the construction can actively learn from the provided context. The existence of such a model at a moderate scale may have implications for interpretability and alignment research.

%% file: BroaderImpact.tex
%\clearpage
\section*{Impact Statements}
We note that the construction of \simulator{} does not appear to increase the probability of harmful behavior, because the construction's primary objective is to implicitly tune an internal model (\textsection{\ref{sec:design}}).
Such tuning has been possible for a long time and is not made more expressive by \simulator{}.

Our findings suggest that existing transformer-based language models can plausibly possess the ability to learn and adapt to context by internally fine-tuning a complex model \emph{even during inference}. 
Consequently, although users are unable to directly modify deployed models, these models may still undergo dynamic updates while processing a context left-to-right, resulting in previously unseen behavior by the time the model reaches the end of the context.
This has significant implications for the field of model alignment.
It is challenging to impose restrictions on a model that can perform such dynamics updates internally, so malicious content can influence the output of deployed models.

Alternatively, we recognize the potential benefits of pre-training constructed models that integrate explicit fine-tuning mechanisms. By embedding the functionalities typically achieved through explicit fine-tuning, such as detecting malicious content and intent within the models themselves, the need for external modules can be mitigated. Pre-training the constructed model may offer a self-contained solution for ensuring safe and responsible language processing without relying on external dependencies.

\section*{Acknowledgements}
The authors acknowledge funding from NSF, ONR, Simons Foundation, and DARPA.
We thank Danqi Chen, Jason Lee, Zhiyuan Li, Kaifeng Lyu, Simran Kaur, Tianyu Gao, and Colin Wang for their suggestions and helpful discussions at different stages of our work.

%% file: Appendix.tex
\section*{Brief overview of the appendix}
In \cref{sec:additional_related_works}, we report few additional related works. In \cref{sec:deferred}, we present some of the deferred definitions from the main paper. In \cref{sec:additional_notations}, we present all the important notations used to present the design of \simulator{}. In \cref{sec:Linearappendix,sec:self-attnt-backprop_appendix,sec:LNappendix,sec:act_appendix}, we present the simulation details of all operations on linear, self-attention, layer normalization, and activation layers respectively for an auxiliary model. In \cref{sec:lm_head}, we present the details for simulating loss computation with the language model head of the auxiliary model. In \cref{sec:additional_modules}, we discuss simulation of additional modules necessary to simulate transformer variants like LLaMA \cite{touvron2023llama} and BLOOM \cite{scao2022bloom}. Finally, in \cref{app:experiment}, we discuss the deferred experimental details from the main paper.

\input{Attention_LN_backprop_mainpaper}

\input{Notations}

\input{Linear_simulation_appendix}

\input{Attention_simulation_appendix}

\input{Layernorm_simulation}

\input{Activation_simulation}

\input{language_model_head}

\input{Parameter_sharing_app}

\input{Additional_modules}

\input{experiment_app}

%% file: Attention_LN_backprop_mainpaper.tex
\section{Additional related works} \label{sec:additional_related_works}
\looseness-1\textbf{Interpretability:} Mechanistic interpretability works reverse-engineer the algorithms simulated by these models~\citep{elhage2021mathematical, olsson2022context, wang2022interpretability, nanda2023progress, chughtai2023toy, conmy2023towards}. These works study local patterns, e.g. activations and attention heads, to derive interpretable insights. 
Other works~\citep{weiss2021thinking,lindner2023tracr} use declarative programs to algorithmically describe transformer models. \citet{zhou2023algorithms} use these to explain task-specific length generalization of transformer models.

\textbf{Alternative Explanations for ICL:} Some works study ICL using a Bayesian framework. \citet{xie2022an} model pretraining data as a mixture of HMMs and cast ICL identifying one such component. \citet{hahn2023theory} later modeled language as a compositional grammar, and propose ICL as a composition of operations. \cite{zhang2023and,jiang2023latent,wang2023large,wies2023learnability} further strengthen this hypothesis by generalizing the underlying latent space. 
On the other hand, careful experiments in \citet{chan2022data} show that data distributional properties (e.g. Zipf's law) drive in-context learning in transformers.  

\input{Tables/table_construction}

\textbf{Transfer learning:} Our construction uses a pre-trained model to initialize a larger transformer, which is similar to several other more empirically oriented works~\citep{gong2019efficient,reddi2023efficient}.

\section{Deferred defintions from main paper}\label{sec:deferred}

For simplicity of exposition, we showcase the definition on a single head self-attention layer (multi-head attention is in \cref{def:self-attn}).
%The self-attention in the auxiliary model is the same as in \simulator{} (\Cref{def:self-attn_construct_single}) without a position vector.
\begin{definition}[Auxiliary model softmax self-attention]
\label{def:self-attn_auxiliary_single}
    A self-attention layer with parameters $\{\mW_Q, \mW_K, \mW_V\}$ takes a sequence $\{ \vx_t \}_{t \le \aux{T}}$ and outputs a sequence $\{ \vy_t \}_{t \le \aux{T}}$, such that
    $$
    \vy_t = \sum_{j} a_{t, j} \vv_j, \qquad \text{with } a_{t, j} =\mathrm{softmax}(\mK \vq_t)_j, \quad 
    \vq_t = \mW_Q \vx_t, \quad \vk_t = \mW_K \vx_t, \quad \vv_t = \mW_V \vx_t,
    $$
    for all $t \le \aux{T}$, and $\mK \in \RR^{ \aux{T} \times  \aux{D} }$ defined with rows $\{ \vk_t \}_{t=1}^{\aux{T}}.$

\end{definition}

%% file: Tables/table_construction.tex
\begin{table*}[t]
  \centering
  \caption{Number of parameters of \simulator{}  for the forward, backward, and gradient update operations on various modules. For simplicity, we have ignored biases in the following computation. We set $S=4$, i.e. stack $4$ weights in each prefix embedding. We set $\simu{H}=12$ for \textsc{OPT-125M} and $\simu{H}=16$ for the other models, $\simu{D} = 4 \aux{D}$ for all the models, and $\simu{T} = \aux{T} + K$, with $\aux{T}=2048$ for \textsc{opt} models, and $K=\aux{D}/4$. $Q=4 Q_{ split} + 3 \simu{T} \simu{D} / \simu{H} $, where $Q_{ split } = \frac{1}{ \simu{H} }(  \simu{D} )^2 +  \simu{H} \simu{D}$,  denotes the number of parameters in a \simulator{} Linear Forward module (\cref{sec:exposition_linear_forward}). % $\aux{N}$ denotes the total number of parameters in the auxiliary model and needs to be included in the final \simulator{} size since they are stored in the Embedding matrix.
  }
  \label{tab:construction}
   %\mz{maybe add vanilla construction vs. construction after approximation? and make it clear how the size and flops increases with $h_d$.}
  \begin{tabular}{lcccc}
    \toprule
    & \multicolumn{4}{c}{Module Size} \\ \cmidrule(lr){2-5} 
    Module Name & Forward & Backward & Descent  & Total \\
    \midrule
    Linear layer  & $Q$ & $Q$ & $Q$  & $3Q$\\
    Layer norms & $Q$ & $Q + 2 \simu{D} \simu{H}$ & $Q$ & $3Q + 2 \simu{D} \simu{H}$\\
    Self-Attention  & $2Q$ & $2Q$ & $2Q$  & $6Q$\\
    Activation & $Q_{ split }$ & $2 \simu{D} \simu{H} $ & $0$ & $Q_{ split } + 2 \simu{D} \simu{H}$\\ \midrule
    Self-Attention block & $4Q$ &  $ 4Q+ 2\simu{D} \simu{H} $  & $ 4 Q $ & $ 12Q + 2\simu{D} \simu{H} $ \\
    Feed-forward block & $3Q+Q_{ split }$ & $ 3Q + 4\simu{D} \simu{H} $ & $ 3 Q $  & $ 9Q + 4\simu{D} \simu{H} $\\
    Transformer block & $7Q+Q_{ split }$ & $7Q + 6\simu{D} \simu{H}$ & $7Q$ & $ 21Q + 6\simu{D} \simu{H}+Q_{ split } $\\
    Transformer & $7QL + LQ_{ split }$ & $(7Q + 6 \simu{D} \simu{H})L$ & $7QL$ &  $ (21Q + 6\simu{D} \simu{H}+Q_{ split })L $ \\ % + $\aux{N}$\\ 
    \midrule
    \textsc{OPT-125m} & \textsc{0.4b} & \textsc{0.4b} & \textsc{0.4b} & \textsc{1.2b}       \\
    \textsc{OPT-350m} & \textsc{1.2b} & \textsc{1.1b} &  \textsc{1.1b} & \textsc{3.4b}       \\
    \textsc{OPT-1.3b} &  \textsc{3.7b}   &  \textsc{3.6b}   &    \textsc{3.5b} & \textsc{10.8b}  \\
    \textsc{OPT-2.7b} & \textsc{7.4b}   &  \textsc{7.2b}   &    \textsc{7.2b} & \textsc{21.8b}\\
    \bottomrule
  \end{tabular}
\end{table*}

%% file: Notations.tex
\section{Notations}\label{sec:additional_notations}
%We use $D$, $N$, $T$, and $K$ to denote the embedding size of \textsc{TinT}, the auxiliary model, the length of the input sequence for the auxiliary model, and the number of prefix tokens respectively. We denote the contextual embedding of a token at position $t$ at the input of any layer $\ell$ as $\ve_t^{(\ell)},$ and $\vx_t^{(\ell)}$ for \textsc{TinT} and the auxiliary model respectively. We denote the embeddings of the prefix tokens at any layer as $\{\vv_j^{(\ell)}\}_{j=1}^K$. For typographical simplicity, when not needed explicitly, we will ignore the superscript that represents the layer index and the subscript that represents the position index. 
Let $D$ denote the embedding dimension for a token and $T$ denote the length of an input sequence. 
$H$ denotes the number of attention heads. 
With the exception of contextual embeddings, we use subscripts to indicate if the quantity is from \simulator{} or from the auxiliary model. For example, $\aux{D}$ refers to the embedding dimension and $\simu{D}$ refers to the \simulator{} embedding dimension.
For contextual embeddings, we use $\ve_t^{(\ell)}\in\RR^{\simu{D}}$ to denote activations in \simulator{} and $\vx_t^{(\ell)}\in\RR^{\aux{D}}$ to denote activations in the auxiliary model, where $\ell$ is the layer and $t$ is the sequence position. When convenient, we drop the superscript that represents the layer index and the subscript that represents the position index. For a matrix $\mA$, $\va_j$ refers to its $j$th row, and for any vector $\vb$, $b_j$ refers to its $j$th element. \simulator{} uses one-hot positional embeddings $\{ \simuw{\vp}_i \in \RR^{\simu{T}} \}_{i \le \simu{T} }$.

We differentiate the parameters of the auxiliary model and \textsc{TinT} by using an explicit superscript $\textsc{TinT}$ for \textsc{TinT} parameters, for example, the weights of a linear layer in \textsc{TinT} will be represented by $\simuw{\mW}$. We use two operations throughout: $\attsplit_h$ and $\attmerge$. Function $\attsplit_h: \RR^{d} \to \RR^{h \times  \lfloor d/h \rfloor  }$ takes an input $\vx \in \RR^d$ and outputs $H$ equal splits of $\vx$, for any arbitrary dimension $d$.  Function $\attmerge:  \RR^{h \times d} \to \RR^{dh}$ concatenates the elements of a sequence $\{\vx_i \in \RR^{d}\}_{i \le h}$ into one single vector, for any arbitrary $d$ and $h$. Recall that for a matrix $\mA$, $\va_j$ refers to its $j$th row, and for any vector $\vb$, $b_j$ refers to its $j$th element. However, at a few places in the appendix, for typographical reasons, for a matrix $\mA$, we have also used $(\mA)_j$  to refer to its $j$th row, and for any vector $\vb$, $(\vb)_j$ to refer to its $j$th element.

\paragraph{ \simulator Attention Module} We modify the usual attention module to  include the position embeddings $\{ \simuw{\vp}_i \in \RR^{ \simu{T} } \}_{i \le \simu{T}  }$. 
In usual self-attention modules, the query, key, and value vectors at each position are computed by token-wise linear transformations of the input embeddings. In \textsc{TinT}'s Attention Module, we perform additional linear transformations on the position embeddings, using parameters $\mW^p_Q, \mW^p_K, \mW^p_V$, and decision vectors $\lambda^Q, \lambda^K, \lambda^V \in \RR^{\simu{H}}$ decide whether to add these transformed position vectors to the query, key, and value vectors of different attention heads.  For the following definition, we use $\Hat{e}$ to represent input sequence and $\Tilde{e}$ to represent the output sequence: we introduce these general notations below to avoid confusion with the notations for token and prefix embeddings for \simulator illustrated in \cref{fig:general_structure_simulator}.

\begin{definition}[\textsc{TinT}'s self-attention with $\simu{H}$ heads]\label{def:self-attn_construct}
    For parameters $\{ \simuw{\mW}_Q, \simuw{\mW}_K, \simuw{\mW}_V \in \RR^{ \simu{D} \times \simu{D} } \}$, $\{ \simuw{\vb}_Q, \simuw{\vb}_K, \simuw{\vb}_V \in \RR^{\simu{D}} \}$, $\{ \mW^p_Q, \mW^p_K, \mW^p_V \in \RR^{ \simu{T} \times \simu{D}/\simu{H} } \}$ and $\{ \lambda^{Q},  \lambda^{K}, \lambda^{V} \in \RR^{ \simu{H} } \}$, \simulator{} self-attention with $\simu{H}$ attention heads and a function $\attnfn: \RR^{ \simu{T} } \to \RR^{ \simu{T} }$ takes a sequence $\{ \Hat{\ve}_t\in\RR^ { \simu{D} }  \}_{t \le \simu{T}}$ as input and outputs $\{ \Tilde{\ve}_t \in \RR^{ \simu{D} } \}_{t \le \simu{T} }$, with 
    \begin{align*}
        &\Tilde{\ve}_t = \attmerge ( \{ \sum_{j \le \simu{T} } a^h_{t, j}  \Tilde{\vv}^{h}_j )_h \}_{h \le \simu{H} } ), \text{ with   } a^h_{t, j} =
        \attnfn ( \Tilde{\mK}^{h} \Tilde{\vq}_t^h  )_j  \\
        &\Tilde{\vq}^{h}_t = \attsplit_H (\vq_t)_h + \lambda^Q_h \mW^p_Q \simuw{\vp}_t; \quad \Tilde{\vk}^{h}_t = \attsplit_H (\vk_t)_h + \lambda^K_h \mW^p_K \simuw{\vp}_t; \\& \Tilde{\vv}^{h}_t = \attsplit_H (\vv_t)_h + \lambda^V_h \mW^p_v \simuw{\vp}_t.
    \end{align*}    
     Here, $\vq_t$, $\vk_t$, $\vv_t$ denote the query, key, and value vectors at each position $t$, computed as $\simuw{\mW}_Q \Hat{\ve}_t + \simuw{\vb}_Q$, $\simuw{\mW}_K \Hat{\ve}_t + \simuw{\vb}_K$, and $\simuw{\mW}_V \Hat{\ve}_t + \simuw{\vb}_V$ respectively.
     $\Tilde{\mK}^h \in \RR^{ \simu{T} \times \simu{D}/\simu{H} }$ is defined with its rows as $\{\Tilde{\vk}^h_t\}_{t \le \simu{T}}$ for all $h \le \simu{H}$.
\end{definition}

$\attnfn$ can be either linear or softmax function. 
%This roughly justifies why we 
%use Linear \textsc{TinT} Self-Attention layers, whenever necessary.

\input{Theorems/softmax_linear}

%% file: Theorems/softmax_linear.tex
%With the notations in \cref{def:self-attn_construct}, in the following theorem, we consider  a linear self-attention layer with $\simu{H}$ attention heads that takes a sequence $\{\Hat{\ve}_t\in\RR^\simu{D}\}_{t \le \simu{T}}$ as input and outputs $\{ \Tilde{\ve}^{linear}_t \in \RR^\simu{D} \}_{t \le \simu{T}}$.

\paragraph{Bounded parameters and input sequence: } We define a linear self-attention layer to be $B_w$-bounded, if the $\ell_2$ norms of all the parameters are bounded by $B_w$. Going by \cref{def:self-attn_construct}, this implies $$\max \{ \norm{ \simuw{\mW}_Q } _2,  \norm{ \simuw{\mW}_K }_2, \norm{ \simuw{\mW}_V }_2 \} \le B_w , \quad \max \{\norm{ \simuw{\vb}_Q } _2,  \norm{ \simuw{\vb}_K }_2, \norm{ \simuw{\vb}_V }_2 \} \le B_w$$  $$\max \{ \norm{ \mW_Q^p } _2,  \norm{ \mW_K^p }_2, \norm{ \mW_V^p}_2 \} \le B_w, \quad \max \{ \norm{ \lambda^Q } _2,  \norm{ \lambda^K }_2, \norm{ \lambda^V }_2  \}  \le B_w.$$ Furthermore, we define an input sequence $\{ \Hat{\ve}_t \}_{t \le \simu{T}}$ to $B_x$-bounded, if $\norm{\Hat{\ve}_t}_2 \le B_x$ for all $t$.

Recall from the main paper (\cref{sec:exposition_linear_forward}), we used Linear \simulator{} Self-Attention layer to represent the linear operations of the auxiliary model.
In the following theorem, we show that a linear attention layer can be represented as a softmax attention layer that uses an additional attention head and an extra token $\vu$, followed by a linear layer. 
Therefore, replacing softmax attention with linear attention does not deviate too far from the canonical transformer.
We use the Linear \simulator{} Self-Attention layers in several places throughout the model.
%\snote{add something about $\tilde\ve$ here}

\begin{theorem}\label{thm:linear_softmax}
 For any $B_w > 0$, consider a $B_w$-bounded linear self-attention layer that returns $\{ \Tilde{\ve}^{linear}_t \in \RR^\simu{D} \}_{t \le \simu{T}}$ on any input $\{\Hat{\ve}_t\in\RR^\simu{D}\}_{t \le \simu{T}}$.
 Consider a softmax self-attention layer with $2\simu{H}$ attention heads and an additional token $\vu \in\RR^{2\simu{D}}$ such that for any $B_x$-bounded input $\{ \Hat{\ve}_t \}_{t \le \simu{T}}$, it takes a modified input sequence $\{ \Bar{\ve}_1, \cdots, \Bar{\ve}_{\simu{T}}, \vu\}$, and returns $\{ \Tilde{\ve}^{softmax}_t \in \RR^{ 2\simu{D} } \}_{t \le \simu{T}}$.
 Each modified input token $\Bar{\ve}_t \in \RR^{2\simu{D}}$ is obtained by concatenating additional $0$s to $\Hat{\ve}_t$.
 Then, for any $B_x > 0$, and $\epsilon \le \mathcal{O}(\simu{T}^{-2} B_w^{-5} B_x^{-5})$, there exists $\mW_O \in \RR^{\simu{D} \times 2\simu{D}}$ and such a softmax self-attention layer such that 
\begin{align*}
    \norm{ \mW_O \Tilde{\ve}^{softmax}_t - \Tilde{\ve}^{linear}_t }_2 \le \mathcal{O}( \sqrt{\epsilon} ),
\end{align*}
for all $t \le \simu{T}$. 
%For any $\epsilon > 0$ and any linear self-attention layer there exists a softmax self-attention layer that can represent the linear self-attention layer with error $\mathcal{O}(\epsilon)$, using an additional attention head and an additional token in the input sequence. 
\end{theorem}

\begin{proof}

    Consider an input sequence $\{\vx_t\}_{t \le \simu{T}}$.
    Let the attention scores of any linear head $h \le \simu{H}$ in the linear attention layer be given by $\{a^{h}_{t, j}\}_{j \le \simu{T}},$ at any given position $t$. Additionally, let the value vectors for the linear attention be given by $\vv_t$. To repeat our self-attention definition, the output of the attention layer at any position $t$ is given by $\attmerge (\{ \Tilde{\ve}^{linear, h}_t \}_{h \le \simu{H}} )$, where
    \begin{align*}
        \Tilde{\ve}^{linear, h}_t = \sum_{ j \le \simu{T}} a^h_{t, j} \vv^h_j.
    \end{align*}

    Under our assumption, $B_w$ denotes the maximum $\ell_2$ norm of all the parameters in the linear self-attention layer and $B_x$ the maximum $\ell_2$ norm in the input sequence, i.e. $\max_{t \le \simu{T}} \norm{\vx_t}_2 \le B_x$. With a simple application of Cauchy-Schwartz inequality, we can show that $\max_{j \le \simu{T}} |a^{h}_{t, j}| \le \mathcal{O}(B_w^2 B_x^2), $ and $\max_{t \le \simu{T}} \norm{\vv^h_t}_2 \le \mathcal{O}(B_w B_x).$

    For  $\epsilon \le \mathcal{O} (\simu{T}^{-10/9} B_w^{-40/9} B_x^{-40/9} ) $, we can then use \cref{lem:softmax_appr} to represent for each $t, j \le \simu{T}$,
    \begin{align*}
         a^{h}_{t, j} &=  \frac{ \epsilon^{-3} e^{ 
\epsilon a_{t, j} } }{ \sum_{ t' \le \simu{T} } e^{ \epsilon a^{h}_{t, t'}  } + e^{ -2 \log \epsilon } } - \epsilon^{-1} + \mathcal{O}\left( \epsilon(  \simu{T} + a^{h}_{t, j} ) \right) \\&
:= \epsilon^{-3} \mathrm{softmax}\left( \{  \epsilon a^{h}_{t, 1}, \epsilon a^{h}_{t, 2}, \cdots, \epsilon a^{h}_{t, \simu{T}}, -2 \log \epsilon \} \right)_j - \epsilon^{-1} +  \mathcal{O}\left( \epsilon^{0.9} \right).
    \end{align*}

    %Thus, we can define $\vu$ to be a one-hot vector, containing $1$ in the additiona and extend $\Bar{e}_t$ from $\Hat{e}_t$ to contain $0$s in the additional dimensions.

    \paragraph{Softmax attention construction:} We define $\vu$, and the query and key parameters of the softmax attention layer such that for the first $\simu{H}$ attention heads,  the query-key dot products for all the attention heads between any pairs $\{(\Bar{\ve}_t, \Bar{\ve}_j)\}_{t, j \le \simu{T}}$ is given by $\{ \epsilon a^{h}_{t, j} \}_{h \le \simu{H}}$, while being $-2\log \epsilon$ between $\vu$ and any token $\Bar{\ve}_t$, with $t \le \simu{T}$. For the rest  of $\simu{H}$ attention heads, the attention scores are uniformly distributed across all pairs of tokens (attention score between any pair of tokens is given by $\frac{1}{\simu{T} + 1}$).

    We set the value parameters of the softmax attention layer such that at any position $t \le \simu{T}$, the value vector is given by $\attmerge( \{ \epsilon^{-3} \vv_t,  \vv_t\}).$ The value vector returned for $\vu$ contains all $0$s.
    %$\Bar{\ve}_t$, while it returns a vector of $0$s for the token $\vu$. %\ap{finish! Should be over by 11.59 pm today}

    \paragraph{Softmax attention computation:}  Consider an attention head $h \le \simu{H}$ in the softmax attention layer now. The output of the attention head at any position $t \le \simu{T}$ is given by
    \begin{align*}
        \Tilde{\ve}^{softmax, h}_t &= \sum_{j \le \simu{T}} \mathrm{softmax}\left( \{  \epsilon a^{h}_{t, 1}, \epsilon a^{h}_{t, 2}, \cdots, \epsilon a^{h}_{t, \simu{T}}, -2 \log \epsilon \} \right)_j \epsilon^{-3} \vv_j^h \\& =  \sum_{j \le \simu{T}} \left( a_{t, j}^h + \epsilon^{-1} + \mathcal{O}(\epsilon^{0.9}) \right) \vv^h_j.
    \end{align*}
    This has an additional $ \sum_{j \le \simu{T}} \left( \epsilon^{-1} + \mathcal{O}(\epsilon^{0.9}) \right) \vv^h_j$, compared to $\Tilde{\ve}^{linear, h}_t$. However, consider the output of the attention head $\simu{H} + h$ at the same position:
    \begin{align*}
        \Tilde{\ve}^{softmax, \simu{H} +  h}_t = \frac{ 1 }{ \simu{T} + 1 } \sum_{  j \le \simu{T} }   \vv^h_j.
    \end{align*}

    Hence, we can use the output matrix $\mW_O$ to get $\Tilde{\ve}^{softmax, h}_t - \frac{\simu{T}+1}{\epsilon} \Tilde{\ve}^{softmax, \simu{H} +  h}_t = \sum_{j \le \simu{T}} \left( a_{t, j}^h + \mathcal{O}(\epsilon^{0.9}) \right) \vv^h_j.$ 
    The additional term $\mathcal{O}(\epsilon^{0.9})\sum_{  j \le \simu{T} }   \vv^h_j$ can be further shown to be $\mathcal{O}(\epsilon^{0.5})$ small with the assumed bound of $\epsilon$, since each $ \vv^h_j $ is atmost $\mathcal{O}(B_w B_x)$ in $\ell_2$ norm with a Cauchy Schwartz inequality. 
    %This requires $\epsilon \le \mathcal{O}( B^{-9/5}_w B^{-9/5}_x \simu{T}^{-9/5}  ).$ 
    %which is $\mathcal{O}(\epsilon)$ close to $\Tilde{\ve}^{linear, h}_t $.
\end{proof}

 \begin{lemma}\label{lem:softmax_appr}
     For $\epsilon > 0$, $B > 0$, and a sequence $\{a_1, a_2, \cdots, a_T\}$ with each $a_i \in \RR$ and $\abs{a_i} \le B$, the following holds true for all $i \le T$,
     \begin{align*}
           \frac{ \epsilon^{-3} e^{ 
\epsilon a_{i} } }{ \sum_{ t' \le T } e^{ \epsilon a_{t'}  } + e^{ -2 \log \epsilon } } = a_{i}  + \frac{1}{\epsilon} + \mathcal{O}\left( \epsilon^{0.9} \right) ,
     \end{align*}
     provided $\epsilon \le \mathcal{O}( T^{-10/9} B^{-20/9} ).$
 \end{lemma}

 \begin{proof}
      We will use the following first-order Taylor expansions:
    \begin{align}
        &e^x = 1 + x + \mathcal{O}(x^2) \label{eq:taylor_e}. \\
        & \frac{1}{1+x} = 1 - \mathcal{O}(x). \label{eq:taylor_inv}
    \end{align}
    Hence, for any $x \ll 1$, $x \approx e^x - 1.$ 

    Simplifying the L.H.S. of the desired bound, we have
    \begin{align}
         \frac{ \epsilon^{-3} e^{ 
\epsilon a_{i} } }{ \sum_{ t' \le T } e^{ \epsilon a_{t'}  } + e^{ -2 \log \epsilon } }  &= \frac{ \epsilon^{-3}  (1 + \epsilon a_i + \mathcal{O}(\epsilon^2 a_i^2) ) }{ \sum_{ t' \le T } (1 + \epsilon a_{t'} + \mathcal{O} (\epsilon^2 a_{t'}^2)  ) + e^{ -2 \log \epsilon } } \label{eq:softmax_step1} \\&
= \frac{ \epsilon^{-1} +  a_i + \mathcal{O}(\epsilon a_i^2)  }{ \sum_{ t' \le T } (\epsilon^{2} + \epsilon^3 a_{t'} + \mathcal{O} (\epsilon^4 a_{t'}^2)  ) + 1 } \label{eq:softmax_step2} \\&
= \left(  \epsilon^{-1} +  a_i + \mathcal{O}(\epsilon a_i^2) \right) \left( 1 + \mathcal{O}( \epsilon^2 T)  \right) \label{eq:softmax_step3} \\&
= \epsilon^{-1} + a_i + \mathcal{O}( \epsilon T + a_i^2 T \epsilon^2 + a_i^2 T \epsilon^3 + \epsilon a_i^2 ) = \epsilon^{-1} + a_i +  \mathcal{O}( \epsilon^{0.9} ). \nonumber
    \end{align}
    We used taylor expansion of exponential function( \cref{eq:taylor_e} ) in \cref{eq:softmax_step1} to get \cref{eq:softmax_step2}, and taylor expansion of inverse function(\cref{eq:taylor_inv}) to get \cref{eq:softmax_step3} from \cref{eq:softmax_step2}. Furthermore, with the lower bound assumption on $\epsilon$, $\sum_{ t' \le T } (\epsilon^{2} + \epsilon^3 a_{t'} + \mathcal{O} (\epsilon^4 a_{t'}^2)  )$ can be shown to be atmost $3\epsilon^2T$, which amounts to $\mathcal{O}(\epsilon^2 T)$ error in \cref{eq:softmax_step3}. The final error bound has again been simplified using the lower bound assumption on $\epsilon$.
 \end{proof}

\subsection{Simulating Multiplication from \cite{akyurek2022learning}}

We refer to the multiplication strategy of \cite{akyurek2022learning} at various places.

\begin{lemma}\label{thm:gelu_multiplication}[Lemma 4 in \cite{akyurek2022learning}]
The $GeLU$ \cite{zzhendrycks2016gaussian} nonlinearity can be used to perform multiplication: specifically,
\begin{align*}
    \sqrt{\pi/2} ( GeLU(x+y) - GeLU(y) ) = xy + \mathcal{O}(x^3y^3).
\end{align*}
\end{lemma}

Thus, to represent an element-wise product or a dot product between two sub-vectors in a token embedding, we can use a MLP with a $GeLU$ activation.

%\paragraph{Non-causal attention:} Like prefix language models \citep{liu2018generating,raffel2020exploring}, the self-attention layers in TinT need to be non-causal on a prefix sub-portion of the input, specifically in modules representing descent operations, backpropagation in the language model head, and the self-attention layers of the simuiliary model. For descent operations, the gradients from the token embeddings are used to update the simuiliary model parameters stored in the prefix tokens. The backpropagation module for the language model head relies on computing the gradient at position $t$ with respect to the loss by utilizing the token embedding at position $t+1$. Finally, for the self-attention layer, the gradients propagating through position $t$ are influenced by the gradients at positions greater than or equal to $t$ (more details in \cref{sec:self-attnt-backprop}). Consequently, the attention mechanism must account for future positions during the backpropagation process.

%To fulfill these requirements, explicit masks are employed, which serve to enable bidirectional attention solely within the portion of the input sequence used for updating the simuiliary model parameters. These masks effectively control the attention patterns, ensuring the desired non-causal behavior for the designated modules within $\textsc{TinT}$.

%% file: Linear_simulation_appendix.tex
\section{Linear layer}\label{sec:Linearappendix}
In the main paper, we defined the linear layer without the bias term for simplicity (\cref{def:linear_mainpaper}). In this section, we will redefine the linear layer with the bias term and present a comprehensive construction of the Linear Forward module.

\begin{definition}[Linear layer]\label{def:linear_appendix}
    For a weight $\mW\in\RR^{\aux{D} \times \aux{D}}$ and bias $\vb \in \RR^{ \aux{D} }$, a linear layer takes $\vx\in\RR^{\aux{D}}$ as input and outputs $\vy = \mW\vx + \vb$.
\end{definition}

In the discussions below, we consider a linear layer in the auxiliary model with parameters $\{\mW, \vb\}$ that takes in input sequence $\vx_1, \cdots, \vx_{\aux{T}}$ and outputs $\vy_1, \cdots, \vy_{\aux{T}}$, with $\vy_t = \mW \vx_t + \vb$ for each $t \le \aux{T}$. Since this involves a token-wise operation, we will present our constructed modules with a general token position $t$ and the prefix tokens $\{\vv_j\}.$

\paragraph{\simulator{} Linear Forward module} Continuing our discussion from \cref{sec:exposition_linear_forward}, we represent $S$ stacked rows of $\mW$ as a prefix embedding. In addition, we store the bias $\vb$ in the first prefix embedding ($\vv_1$).  

Using a set of $S'$ unique attention heads in a \simulator{} attention module (\cref{def:self-attn_construct}), we copy the bias $\vb$ to respective token embeddings and use a \simulator{} linear layer to add the biases to the final output.

\paragraph{Auxiliary's backpropagation through linear layer} For a linear layer as defined in \cref{def:linear_appendix}, the linear backpropagation layer takes in the loss gradient w.r.t. output ($\losspartial{\vy}$) and computes the loss gradient  w.r.t. input ($\losspartial{\vx}$).

\begin{definition}[Linear backpropagation ]\label{def:linear_backprop}
    For a weight $\mW\in\RR^{ \aux{D}  \times \aux{D} }$ , the linear backpropagation layer takes $\losspartial{ \vy } \in\RR^{ \aux{D} }$ as input and outputs $\losspartial{\vx} = \mW^{\top} \losspartial{\vy}$.
\end{definition}

\paragraph{\simulator{} Linear backpropagation module} 
%Correspondingly,
This module will aim to simulate the auxiliary's linear backpropagation.
The input embedding $\ve_t$ to this module will contain the gradient of the loss w.r.t. $\vy_t$, i.e. $\losspartial{  \vy_t }$. As given in \cref{def:linear_backprop}, this module will output the gradient of the loss w.r.t. $\vx_t$, given by $\losspartial{\vx_t} = \mW^{\top} \losspartial{\vy_t}.$ 

We first use the residual connection to copy the prefix embeddings $\{ \vv_j \}$ (i.e., the rows of $\mW$) from the forward propagation module.
%We first copy using the residual connection, the prefix embeddings $\{ \vv_j \}$ from the linear forward module, which contains the stacked rows of the matrix $\mW$. 
A straightforward construction would be to use the Linear Forward module but with the columns of $\mW$ stored in the prefix tokens, thereby simulating multiplication with $\mW^\top$.
However, such a construction requires applying attention to the prefix tokens, which increases the size of the construction substantially.
%To get the output, a simple strategy can be to use the linear forward module but with rows of $\mW^{\top}$ in the prefix tokens $\{\vv_j\},$ and bias set to $0$. However, this will involve additional attention modules (and hence additional training parameters) on the content of $\{\vv_j\}$ itself. 

We instead perform the operation more efficiently by splitting it across attention heads. In particular, once we view the operation as $\losspartial { \vx_t } = \sum_{i} \left(\losspartial { \vy_t } \right)_i \vw_i$, we can see that the attention score between the current token and the prefix token containing $\vw_i$ must be $\left(\losspartial {\vy_t } \right)_i$. Using value vectors as rows of $\mW$ returns the desired output.
Similar to the Linear Forward module, we shard the weights into $S'$ parts to parallelize across more attention heads. Please see \cref{fig:linear_back}.

%The more efficient strategy will be to view the required operation as $\losspartial { \vx_t } = \sum_{i} \left(\losspartial { \vy_t } \right)_i \vw_i$. 
%This can then be represented with  the attention heads, where for each attention head, the attention score between $\vv_j$, for any $j \le K$, and $\mathbf{e}_t$ is given by $\left(\losspartial {\vy_t } \right)_i$, with the value vector being $\vw_i$ for some $\vw_i$ stored in $\vv_j$. The linear projection layer will sum up the results of the attention heads. 

%each attention head represents the attention score between $\ve_t$ and $\vv_j$ that contains $\vw_i$ as $\left(\frac{\partial \loss}{\partial \vy_t}\right)_i$, with the value matrix using $\vw_i$.

\begin{figure}[t]
    \centering
    \includegraphics[width=0.7\textwidth]{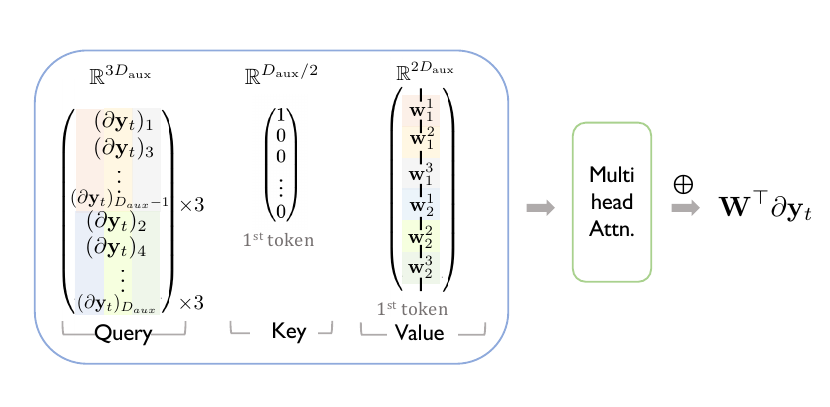}
    \caption{\simulator{} simulates the backward pass of a linear layer as a $H$-head attention layer ($H=6$ pictured), with the gradient of the loss w.r.t. linear layer output ($\losspartial{\vy_t}$) as the query, the positional one-hot vector of prefix embeddings as the key, and the parameters of the auxiliary model stored in the prefix embeddings as the value. Similar to the Linear Forward module (\cref{fig:linear_forward}), we distribute the dot product computations across all attention heads by sharding the vectors into $S'$ ($S'=3$ here) parts. We omitted the identical transformation for query, and value matrices, and permutation-based transformation for key matrix for illustration purposes.}
    \label{fig:linear_back}
\end{figure}

\begin{figure}[t]
    \centering
    \includegraphics[width=\textwidth]{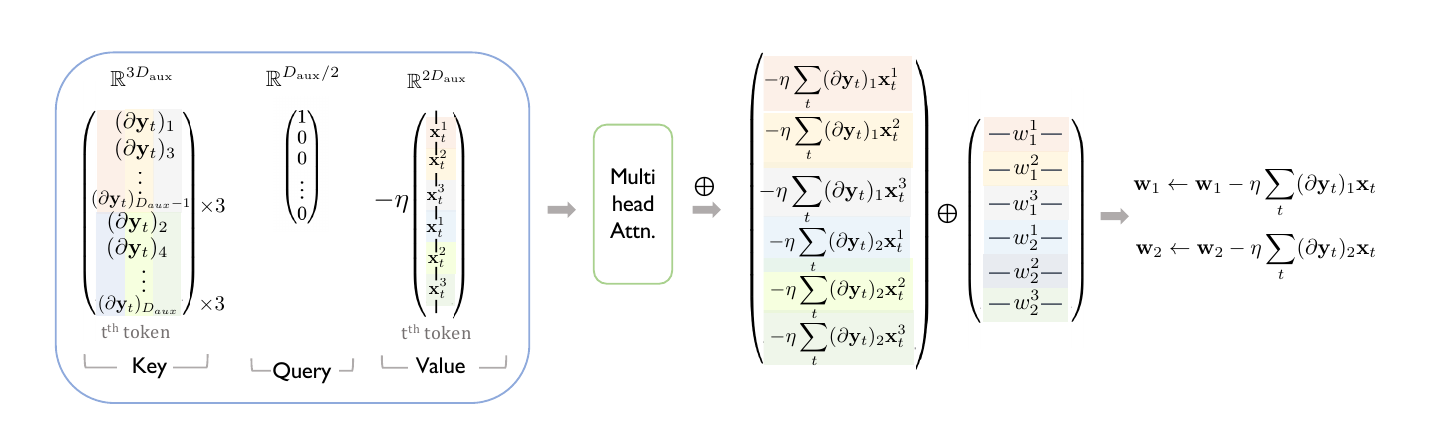}
    \caption{\simulator{} computes the parameter gradients for a linear layer as a $H$-head attention layer ($H=6$ pictured),  with  the gradient of the loss w.r.t. linear layer output ($\losspartial{\vy_t}$) as the query, the positional one-hot vector of prefix embeddings as the key, and  the input to the linear layer ($\vx_t$) as the value. The auxiliary model parameters in the prefix embeddings are then updated using a residual connection.
    Similar to the Linear Forward module (\cref{fig:linear_forward}), we distribute the dot product computations across all attention heads, by sharding the vectors into $S'$ ($S'=3$ here) parts. We omitted the identical transformation for query, and value matrices, and permutation-based transformation for key matrix for simplicity. }
    \label{fig:linear_descent}
\end{figure}

\paragraph{Auxiliary's linear descent update} Finally, the linear descent layer updates the weight and the bias parameters using a batch of inputs $\{\vx_t\}_{t \le {\aux{T}}}$ and the loss gradient w.r.t. the corresponding outputs $\{ \losspartial{\vy_t} \}_{t \le {\aux{T}}}$.

\begin{definition}[Linear descent]\label{def:linear_descent}
    For a weight $\mW\in\RR^{\aux{D}\times \aux{D}}$ and a bias $\vb \in \RR^{ \aux{D}}$, the linear descent layer takes in a batch of inputs $\{ \vx_t \in \RR^ \aux{D} \}_{t \le {\aux{T}}}$ and gradients $\{ \losspartial{\vy_t} \in \RR^ \aux{D} \}_{t \le {\aux{T}}}$ and updates the parameters as follows:
    \begin{align*}
        \mW \gets \mW - \eta \sum_{t \le {\aux{T}}} \losspartial{\vy_t} \vx_t^{\top}; \quad \quad
        \vb \gets \vb - \eta \sum_{t \le {\aux{T}}} \losspartial{\vy_t}.
    \end{align*}
\end{definition}

\input{figures/Linear}
\paragraph{\simulator{} Linear descent module} The input embedding $\ve_t$ to this module will contain the gradient of the loss w.r.t. $\vy_t$, i.e. $\losspartial{  \vy_t }$.

As in the Linear backpropagation module, the prefix tokens $\{ \vv_j \}$ will contain the rows of  $\mW$ and $\vb$, which have been copied from the Linear forward module using residual connections. Since, in addition to the gradients, we also require the input to the linear layer, we will use residual connections to copy the input $\{ \vx_t \}$ to their respective embeddings $\{ \ve_t \},$ from the Linear Forward module. As given in \cref{def:linear_descent}, this module will update $\mW$ and $\vb$ using the gradient descent rule. 

%Since the update also requires $\vx_t$, we will use the residual connection to get $\vx_t$ from the Linear Forward module.

Focusing on $\vw_i$, the descent update is given by $\vw_i \gets \vw_i - \eta \sum_{t} \left( \losspartial { \vy_t } \right)_i \vx_t$. For the prefix token $\vv_j$ that contains $\vw_i$, the update term $- \eta \sum_{t} \left( \losspartial { \vy_t } \right)_i \vx_t$ can be expressed with an attention head that represents the attention between the prefix token $\vv_j$ and any token $\ve_t$ with score $\left( \losspartial { \vy_t } \right)_i$ and value $-\eta \vx_t$. The residual connection can then be used to update the weights $\vw_i$ in $\vv_j$. 

For the bias $\vb$, the descent update is give by $\vb \gets \vb - \eta \sum_{t} \losspartial { \vy_t } $. With $\vb$ present in $\vv_1$, we use one attention head to represent the attention score between prefix token $\vv_1$ and any token $\ve_t$ as $1$, with the value being $-\eta \losspartial{\vy_t}.$ The residual connection can then be used to update the weights $\vb$ in $\vv_1$.

The above process can be further parallelized across multiple attention heads, by sharding each weight computation into $S'$ parts. Please see \cref{fig:linear_descent}.

\subsection{$\simu{H}$-split operation} \label{sec:H-split_operation_app}

We leverage local structure within the linear operations of \simulator{} to make the construction smaller. 
We build two $\simu{H}$-split operations to replace all the linear operations. We use $\simu{d}$ to denote $\simu{D}/\simu{H}$ in the following definitions. 
\begin{definition}[Split-wise $\simu{H}$-split Linear operation]
    For weight and bias parameters $\simuw{\mW} \in \RR^{\simu{H} \times \simu{d}  \times \simu{d}}, \simuw{\mB} \in \RR^{\simu{H} \times \simu{d}}$ , this layer takes in input $\ve \in \RR^{\simu{D}}$ and returns $\Tilde{\ve} = \attmerge( \Tilde{\mS}  + \simuw{\mB} )$, with $\Tilde{\mS} \in \RR^{\simu{H} \times \simu{d}} $ defined with rows $ \{ \simuw{\mW}_h \attsplit_{\simu{H}}(\ve)_h \}_ {h \le \simu{H} }$. 
\end{definition}

\begin{definition}[Dimension-wise $\simu{H}$-split Linear operation]
    For weight and bias parameters $\simuw{\mW} \in \RR^{\simu{d} \times \simu{H} \times \simu{H}}, \simuw{\mB} \in \RR^{\simu{d} \times \simu{H}}$ , this layer takes in input $\ve \in \RR^{\simu{D}}$, defines $\mS \in \RR^{\simu{d} \times \simu{H} }$ with columns $\{ \attsplit_{\simu{H}}(\ve)_h \}_{ h \le \simu{H} }$, and returns $\Tilde{\ve} = \attmerge( (\Tilde{\mS} + \simuw{\mB})^{\top} )$, where $\Tilde{\mS} \in \RR^{\simu{d} \times \simu{H}} $ is defined with rows $ \{ \simuw{\mW}_d \simuw{\vs}_d \}_ {d \le \simu{d} }$.
\end{definition}

We find that we can replace all the linear operations with a splitwise $\simu{H}$-split Linear operation followed by a dimensionwise $\simu{H}$-split Linear operation, and an additional splitwise $\simu{H}$-split Linear operation, if necessary. A linear operation on $\simu{D}$-dimensional space involves $\simu{D}^2$ parameters, while its replacement requires $\simu{D}^2/\simu{H} + 2\simu{D}\simu{H}$ parameters, effectively reducing the total number of necessary parameters by $\simu{H}$.

We motivate the $\simu{H}$-split linear operations with an example. We consider the Linear Forward module in \cref{fig:linear_forward} for simulating a linear operation with parameters $\mW \in \RR^{\aux{D} \times \aux{D}}$ and no biases. For simplicity of presentation, we assume $\aux{D}$ is divisible by $4$. We stack $2$ rows of weights per prefix embedding. We distribute the dot-product computation across the $\simu{H}=6$ attention heads, by sharding each weight into $3$ parts. Since we require to have enough space to store all the sharded computation from the linear attention heads, we require $\simu{D} = 3 \aux{D}$ (we get $3$ values for each of the $\aux{D}$ weights in $\mW$). 
For presentation, for a given vector $\vv \in \RR^{\aux{D}}$, we represent $\attsplit_{3} (\vv)_i$ by $\vv^i$ for all $1 \le i \le 3$.

Now, consider the final linear operation responsible for combining the output of the attention heads. The output, after the linear operation, should contain $\mW \vx_t$ in the first $\aux{D}$ coordinates.
At any position $t$, if we stack the output of the linear attention heads as rows of a matrix $\mS_t \in \RR^{ \simu{H} \times \simu{D} / \simu{H} }$ we get

\begin{align*}
    \mS_t  = \begin{bmatrix}
     \langle \vw_1^1, \vx_t^1 \rangle & \langle \vw_3^1, \vx_t^1 \rangle & \langle \vw_5^1, \vx_t^1 \rangle  & \cdots & \langle \vw_{\aux{D}-1}^1, \vx_t^1 \rangle  \\
     \langle \vw_1^2, \vx_t^2 \rangle & \langle \vw_3^2, \vx_t^2 \rangle & \langle \vw_5^2, \vx_t^2 \rangle  & \cdots & \langle \vw_{\aux{D}-1}^2, \vx_t^2 \rangle \\
     \langle \vw_1^3, \vx_t^3 \rangle & \langle \vw_3^3, \vx_t^3 \rangle & \langle \vw_5^3, \vx_t^3 \rangle  & \cdots & \langle \vw_{\aux{D}-1}^3, \vx_t^3 \rangle \\
     \langle \vw_2^1, \vx_t^1 \rangle & \langle \vw_4^1, \vx_t^1 \rangle & \langle \vw_6^1, \vx_t^1 \rangle  & \cdots & \langle \vw_{\aux{D}}^1, \vx_t^1 \rangle  \\
     \langle \vw_2^2, \vx_t^2 \rangle & \langle \vw_4^2, \vx_t^2 \rangle & \langle \vw_6^2, \vx_t^2 \rangle  & \cdots & \langle \vw_{\aux{D}}^2, \vx_t^2 \rangle\\
     \langle \vw_2^3, \vx_t^3 \rangle & \langle \vw_4^3, \vx_t^3 \rangle & \langle \vw_6^3, \vx_t^3 \rangle  & \cdots & \langle \vw_{\aux{D}}^3, \vx_t^3 \rangle
\end{bmatrix}
\end{align*}

Note that for each $j \le \aux{D}$, we have $ \langle \vw_j, \vx_t \rangle = \sum_{i=1}^{3}  \langle \vw_j^i, \vx_t^i \rangle$.  Thus, with a column-wise linear operation  on $\mS_t$, we can sum the relevant elements in each column to get
\begin{align*}
    &\mS^{col}_t  = \\& \begin{bmatrix}
     \langle \vw_1, \vx_t \rangle &  \langle \vw_3, \vx_t \rangle &   \cdots & \langle \vw_{ \aux{D}/2 - 1}, \vx_t \rangle  & 0 & 0  & \cdots & 0  \\
     \langle \vw_2, \vx_t \rangle &  \langle \vw_4, \vx_t \rangle &  \cdots & \langle \vw_{ \aux{D}/2 }, \vx_t \rangle  & 0 & 0  & \cdots & 0  \\
     0 &  0 & \cdots & 0  &  \langle \vw_{ \aux{D} / 2 + 1 }, \vx_t \rangle & \langle \vw_{ \aux{D} / 2 + 3 }, \vx_t \rangle  & \cdots & \langle \vw_{ \aux{D} - 1 }, \vx_t \rangle  \\
    0 &  0 &  \cdots & 0  &  \langle \vw_{ \aux{D} / 2 + 2 }, \vx_t \rangle & \langle \vw_{ \aux{D} / 2 + 4 }, \vx_t \rangle & \cdots & \langle \vw_{ \aux{D} }, \vx_t \rangle  \\
     0 &  0 &  \cdots & 0  &  0 &  0 & \cdots & 0   \\
     0 &  0 &  \cdots & 0  &  0 &  0 & \cdots & 0  
\end{bmatrix}
\end{align*}

A row-wise linear operation on $\mS^{col}_t$ can space out the non-zero elements in the matrix and give us
\begin{align*}
    &\mS^{row}_t  = \\& \begin{bmatrix}
     \langle \vw_1, \vx_t \rangle & 0 & \langle \vw_3, \vx_t \rangle & 0 &  \cdots & \langle \vw_{ \aux{D}/2 - 1}, \vx_t \rangle  & 0   \\
     0 & \langle \vw_2, \vx_t \rangle & 0 & \langle \vw_4, \vx_t \rangle &  \cdots & 0 & \langle \vw_{ \aux{D}/2 }, \vx_t \rangle   \\
     \langle \vw_{ \aux{D} / 2 +  1}, \vx_t \rangle & 0 & \langle \vw_{ \aux{D} / 2 +  3}, \vx_t \rangle & 0 &  \cdots & \langle \vw_{ \aux{D} - 1}, \vx_t \rangle  & 0   \\
    0 & \langle \vw_{ \aux{D} / 2 +  2}, \vx_t \rangle & 0 & \langle \vw_{ \aux{D} / 2 +  4}, \vx_t \rangle &  \cdots & 0 & \langle \vw_{ \aux{D} }, \vx_t \rangle   \\
     0 &  0 &  \cdots & 0  &    \cdots & 0 &  0   \\
     0 &  0 &  \cdots & 0  &    \cdots & 0 &  0  
\end{bmatrix}
\end{align*}

Finally, a column-wise linear operation on $\mS^{row}_t$ helps to get the non-zero elements in the correct order.
\begin{align*}
    &\Bar{\mS}^{col}_t  = \\&\begin{bmatrix}
     \langle \vw_1, \vx_t \rangle & \langle \vw_2, \vx_t \rangle & \langle \vw_3, \vx_t \rangle & \langle \vw_4, \vx_t \rangle  &  \cdots & \langle \vw_{\aux{D}/2-1}, \vx_t \rangle & \langle \vw_{\aux{D}/2}, \vx_t \rangle \\
     \langle \vw_{\aux{D}/2 + 1}, \vx_t \rangle & \langle \vw_{\aux{D}/2 + 2}, \vx_t \rangle & \langle \vw_{\aux{D}/2 + 3}, \vx_t \rangle & \langle \vw_{\aux{D}/2 + 4}, \vx_t \rangle  &  \cdots & \langle \vw_{\aux{D}-1}, \vx_t \rangle & \langle \vw_{\aux{D}}, \vx_t \rangle \\
     0 &  0 & 0 & 0   & \cdots & 0 & 0 \\
     \vdots & \vdots & \vdots & \vdots & \cdots & \vdots & \vdots \\
     0 &  0 & 0 & 0  & \cdots & 0 & 0
\end{bmatrix}
\end{align*}

The desired output is then given by $\attmerge( \{ \Bar{\vs}^{col}_{t, j} \} \}_{j=1}^{\aux{D}} )$, which contains $\mW \vx_t$ in the first $\aux{D}$ coordinates. The operations that convert $\mS_t$ to $\mS_t^{col}$ and $\mS_t^{row}$ to $\Bar{\mS}_t^{row}$ represents a split-wise $6$-split linear operation, while the operation that converts $\mS^{col}_t$ to $\mS_t^{row}$ represents a dimension-wise $6$-split linear operation. A naive linear operation on the output of the attention heads would require $\simu{D}^2$ parameters, while its replacement requires $\simu{D}^2/6$ parameters to represent a dimension-wise $6$-split linear operation, and an additional $12 \simu{D}$ parameters to represent the split-wise $6$-split linear operations.

%We follow a similar approach as the backpropagation module, since the gradient descent update for $\vw_i$ can be represented as $\vw_i \gets \vw_i - \eta \sum_{t} \left( \losspartial { \vy_t } \right)_i \vx_t$. 

%This can then be represented with a single attention head, which computes the attention score between $\vv_j$ containing $\vw_i$, and $\mathbf{e}_t$ by $\left(\losspartial { \vy_t }\right)_i$, with the value vector being $-\eta \vx_t$. The residual connection will be used to update the weights $\vw_i$ in $\vv_j$. 

%% file: figures/Linear.tex
\iffalse
\begin{figure*}[t!]
    \centering
    \begin{subfigure}[t]{0.5\textwidth}
        \includegraphics[width=\linewidth]{figures/ICL_linearforward.pdf}
        \caption{Forward}
    \end{subfigure}%
    ~ 
    \begin{subfigure}[t]{0.5\textwidth}
        \includegraphics[width=\linewidth]{figures/ICL_linearback.pdf}
        \caption{Backward}
    \end{subfigure}
    \caption{Simulating forward and gradient backpropagation of a linear layer, shape of input is $batchsize \times (n+m) \times 4d$}
    \label{fig:general_structure_simulator}
\end{figure*}
\fi

%% file: Attention_simulation_appendix.tex
\section{Self-attention layer}\label{sec:self-attnt-backprop_appendix}

\begin{figure}
    \centering
    \includegraphics[width=\textwidth]{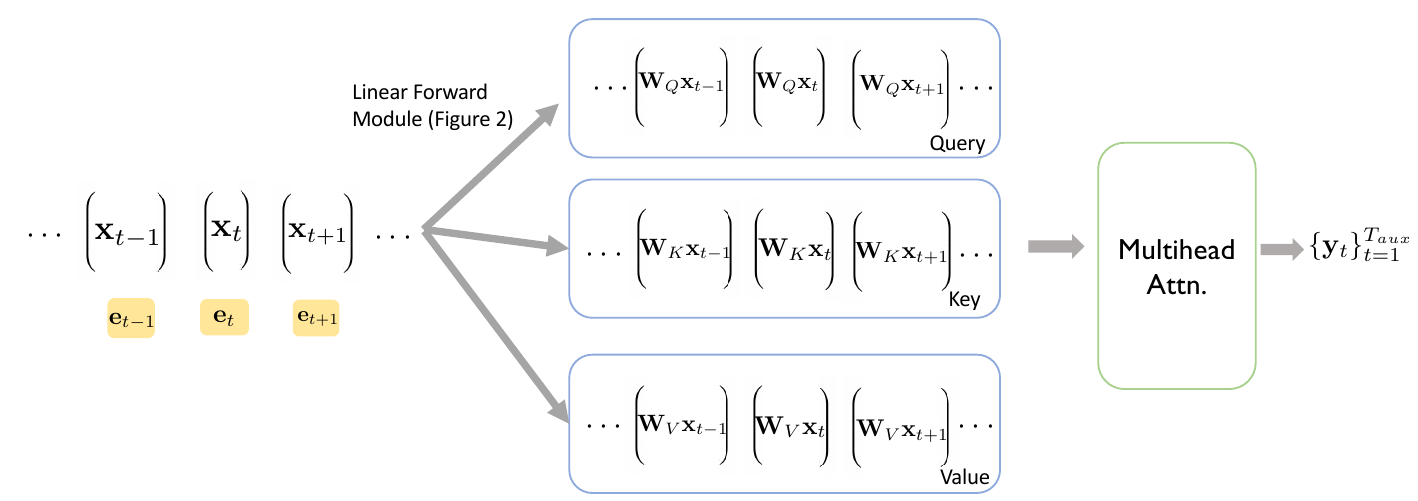}
    \caption{\simulator{} simulates the forward pass of a self-attention layer of the auxiliary model with a Linear Forward module (\cref{fig:linear_forward}) and a \simulator{} softmax attention layer (\cref{def:self-attn_construct}). The Linear Forward module computes the query, key, and value vectors using a Linear Forward module on the current embeddings, changing the prefix embeddings to correspond to $\mW_Q$, $\mW_K$, and $\mW_K$ respectively.}
    \label{fig:attn_forward}
\end{figure}

\begin{figure}
    \centering
    \includegraphics[width=\textwidth]{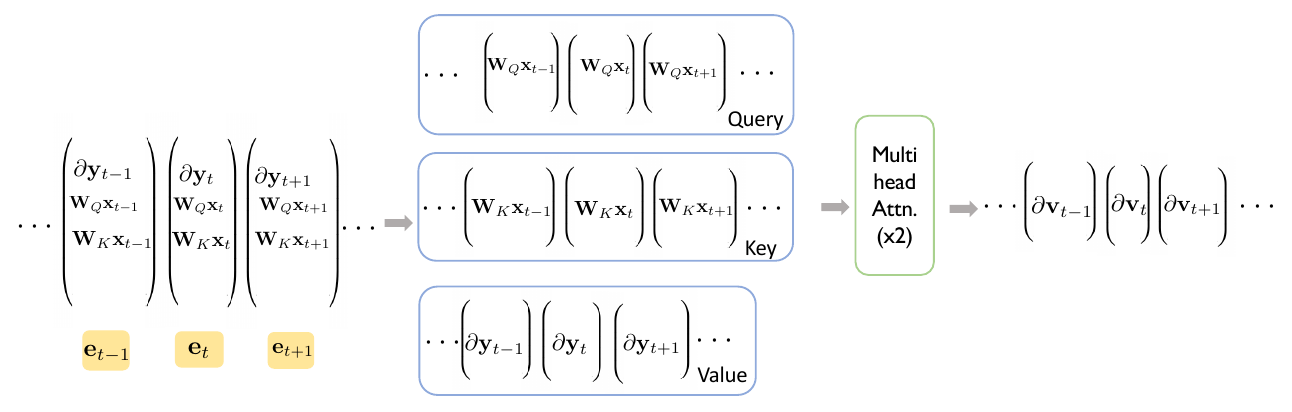}
\caption{The gradient w.r.t. the value vectors $\{\losspartial{\vv_t}\}$  (\cref{def:self-attn-backprop}) forms the integral component for both \simulator{} self-attention backward and descent update modules. \simulator{} computes $\{\losspartial{\vv_t}\}$ using a softmax attention and a linear attention layer. We first use residual connections to copy the query and key vectors to the current embeddings from the \simulator{} Self-attention Forward module (\cref{fig:attn_forward}). The  softmax attention layer re-computes the attention scores $\{a^h_{t, j}\}$ between all token pairs $\{(t, j)\}$ and stores them in the token embeddings. The linear attention layer uses the one-hot position embeddings of the input tokens as the query to use the transposed attention scores $\{a^h_{j, t}\}$ for all token pairs $\{(t, j)\}$ and use the gradients  $\{\losspartial{\vy_t}\}$ as the value vectors to compute $\{\losspartial{\vv_t}\}$.}
    \label{fig:value_gradient}
\end{figure}

We first introduce multi-head attention,  generalizing  single-head attention (\cref{def:self-attn_auxiliary_single}). 
\begin{definition}[Auxiliary self-attention with $\aux{H}$ heads]\label{def:self-attn}
    For query, key, and value weights $\mW_Q, \mW_K, \mW_V \in\RR^{\aux{D}\times \aux{D}}$ and bias $\vb_{Q}, \vb_{K}, \vb_{V} \in\RR^{\aux{D}}$, a self-attention layer with $\aux{H}$ attention heads and a function $\attnfn: \RR^{\aux{T}} \to \RR^{\aux{T}}$ takes a sequence $\{\vx_t\in\RR^{\aux{D}}\}_{t \le \aux{T}}$ as input and outputs $\{ \vy_t \}_{t \le \aux{T}}$, with 
    \begin{align}
        \vy_t = \attmerge ( \{ \sum_{j \le \aux{T}} a^h_{t, j}  \vv^{h}_j  \}_{h \le \aux{H}} ).  \label{eq:attnt_forward}
    \end{align}
    $a^{h}_{t, j}$ is defined as the attention score of head $h$ between tokens at positions $t$ and $j$, and is given by
    \begin{align} \label{eq:attn}
        a^h_{t, j} =  \mathrm{softmax} ( \mK^{h} \vq^{h}_t )_j.
    \end{align}
     Here, $\vq_t$, $\vk_t$, $\vv_t$ denote the query, key, and value vectors at each position $t$, computed as $\mW_Q \vx_t + \vb_Q$, $\mW_K \vx_t + \vb_K$, and $\mW_V \vx_t + \vb_V$ respectively. In addition, $\vq^{h}_t, \vk^{h}_t, \vv^{h}_t$ denote $\attsplit_{\aux{H}} (\vq_t)_h$, $\attsplit_{\aux{H}} (\vk_t)_h$, and $\attsplit_{\aux{H}} (\vv_t)_h$ respectively for all $t \le \aux{T}$, and $h \le \aux{H}$. 
     $\mK^h \in \RR^{\aux{T} \times \aux{D} }$ is defined with its rows as $\{\vk^h_t\}_{t \le \aux{T}}$ for all $h \le \aux{H}$.
\end{definition}

In the discussions below, we consider a self-attention layer in the auxiliary model with parameters $\{\mW_Q, \vb_{Q}, \mW_K, \vb_K, \mW_V, \vb_{V}\}$ that takes in input sequence $\vx_1, \cdots, \vx_{\aux{T}}$ and outputs $\vy_1, \cdots, \vy_{\aux{T}}$, with $\{ \vy_t \}_{t=1}^{\aux{T}}$ given by \eqref{eq:attnt_forward}. As in the definition,  $\vq_t, \vk_t, \vv_t$ denote the query, key, and value vectors for position $t$. We will use \simulator{} self-attention modules in order to simulate the operations on the auxiliary's self-attention layer. To do so, we will need $\simu{H} \ge \aux{H}$ in the corresponding \simulator{} self-attention modules.

\paragraph{\simulator{} Self-attention forward module}
The input embedding to this  module $\ve_t$ at each position $t$ will contain $\vx_t$ in its first $\aux{D}$ coordinates. The self-attention module can be divided into four sub-operations: Computation of (a)  query vectors $\{ \vq_t \}_{t \le T}$, (b)  key vectors $\{ \vk_t \}_{t \le T}$, (c)  value vectors $\{ \vv_t \}_{t \le T}$, and (d)  $\{ \vy_t \}_{t \le T}$ using \eqref{eq:attnt_forward}. Please see \cref{fig:attn_forward}.

\begin{itemize}
    \item Sub-operations (a): The computation of query vector $\vq_t := \mW_Q \vx_t + \vb_Q$  at each position $t$ is a linear operation involving parameters $\mW_Q, \vb_Q$. Thus, we can first feed in the stacked rows of $\mW_Q$ and $\vb_Q$ onto the prefix embeddings $\{ \vv_j \}$. We use a Linear Forward module (\cref{sec:Linearappendix}) on the current embeddings and the prefix embeddings to get embedding $\ve^q_{t}$ at each position $t$ that contains $\vq_t$ in the first $\aux{D}$ coordinates.

    \item Sub-operations (b, c): Similar to (a), we feed in the stacked rows of the necessary parameters onto the prefix embeddings $\{ \vv_j \}$, and call two Linear Forward Modules (\cref{sec:Linearappendix}) independently to get embeddings $\ve^k_t$, and $\ve^v_t$ containing $\vk_t$ and $\vv_t$ respectively.

    We now combine the embeddings $\ve^q_t$, $\ve^k_t$, and $\ve^v_t$ to get an embedding $\ve_t$ that contain $\vq_t, \vk_t, \vv_t$ in the first $3\aux{D}$ coordinates.
    
    \item Sub-operation (d): Finally, we call a \simulator{} self-attention module (\cref{def:self-attn_construct}) on our current embeddings $\{ \ve_t \}_{t \le T}$ to compute $\{ \vy_t \}_{t \le T}$. The query, key, and value parameters in the self-attention module  contain sub-Identity blocks that pick out the relevant information from $\vq_t, \vk_t, \vv_t$ stored in $\ve_t$.
\end{itemize}

\textit{Remark:} Sub-operations (a), (b), and (c) can be represented as a single linear operation with a weight $\mW \in \RR^{3 \aux{D} \times \aux{D} }$ by concatenating the rows of $\{\mW_Q, \mW_K, \mW_V\}$ and a bias $\vb \in \RR^{3 \aux{D}}$ that concatenates $\{\vb_Q, \vb_K, \vb_V\}$. Thus, they can be simulated with a single Linear Forward Module, with $\mW, \vb$ fed into the prefix embeddings.
However, we decide to separate them in order to limit the number of prefix embeddings and the embedding size. E.g. for GPT-2, $\aux{D} = 768$. This demands either a $3\times$ increase in the embedding size in \simulator{} or a $3\times$ increase in the number of prefix embeddings. Hence, in order to minimize the parameter cost, we call Linear Forward Module separately to compute $\vq_t$, $\vk_t$, and $\vv_t$ at each position $t$.

\paragraph{Auxiliary's backpropagation through self-attention} For an auxiliary self-attention layer as defined in \cref{def:self-attn}, the backpropagation layer takes in the loss gradient w.r.t. output ($\{ \losspartial{\vy_t} \}_{t \le \aux{T}}$) and computes the loss gradient  w.r.t. input token ($\{  \losspartial{\vx_t} \}_{t \le \aux{T}}$). 

%%%%%%%%%%%%%%%%%%%%%%%%%%%%%%%%%
%Here was the definition of gradient attention propagation before
\begin{restatable}{definition}{defselfattnbackprop}[Auxiliary self-attention backpropagation]\label{def:self-attn-backprop}
    %\begin{definition}
    For query, key, and value weights $\mW_Q, \mW_K, \mW_V \in\RR^{\aux{D}\times \aux{D}}$ and bias $\vb_{Q}, \vb_{K}, \vb_{V} \in\RR^{\aux{D}}$, the backpropagation layer corresponding to a self-attention layer with $\aux{H}$ attention heads takes a sequence $\{ \losspartial{ \vy_t } \in\RR^{\aux{D}} \}_{t \le \aux{T}}$ and $ \{  \vx_t  \in\RR^{\aux{D}} \}_{t \le \aux{T}} $ as input and outputs $\{ \losspartial{ \vx_t }  \}_{t \le \aux{T}}$, with 
    \begin{align*}
    \losspartial{\vx_t} &= \mW_Q^{\top} \losspartial{ \vq_t }  + \mW_K^{\top} \losspartial{ \vk_t }    + \mW_V^\top \losspartial{ \vv_t } , \quad \text{ with } 
    \\ \losspartial{ \vq_t } &=  \attmerge( \{  \sum_j a^h_{t, j} ( (\losspartial{\vy_t^h})^{\top} \vv^{h}_j ) [ \vk^h_j - \sum_{j'} a^h_{t, j'} \vk^{h}_{j'} ]  \}_{h \le \aux{H}} ); \\
    \losspartial{ \vk_t } &=  \attmerge( \{  \sum_j a^h_{j, t} \vq^{h}_{j}  [ (\losspartial{\vy_j^h})^{\top} (\vv^{h}_t - \sum_{j'} a^h_{j, j'} \vv^{h}_{j'})  ]  \}_{h \le \aux{H}} ); \\
    \losspartial{ \vv_t } &= \attmerge(\{ \sum_{j} a^h_{j, t}  \losspartial{\vy^h_j} \}_{h \le \aux{H}} ) 
    \end{align*} 
    Here, $\vq_t$, $\vk_t$, and $\vv_t$ refer to query, key, and value vectors at each position $t$, with the attention scores $\{a_{t, j}^h\}_{t, j \le \aux{T}, h \le \aux{H}}$. 

\end{restatable}
%%%%%%%%%%%%%%%%%%%%%%%%%%%%%%%%%%

\paragraph{Complexity of true backpropagation} The much-involved computation in the above operation is due to the computation of $\losspartial{\vq_t}$ and $\losspartial{\vk_t}$ at each position $t$. For the following discussion, we assume that our current embeddings $\ve_t$ contain $\vq_t, \vk_t, \vv_t$, in addition to the gradient $\losspartial{\vy_t}$. The computation of $\losspartial{\vq_t}$ (and similarly $\losspartial{\vk_t}$) at any position $t$ involves the following sequential computations and the necessary \simulator{} modules.
\begin{itemize}
    \item  $\{ \{ \losspartial{\vy_t^h})^{\top} \vv^{h}_j \}_{ j \le \aux{T} } \}_{ h \le \aux{H} }$ with a \simulator{} linear self-attention module (\cref{def:self-attn_construct}), with atleast $\aux{H}$ attention heads that represent the attention score between $\ve_t$ and any other token $\ve_j$, by $\{ (\losspartial{\vy_t^h})^{\top} \vv^{h}_j  \}_{ h \le \aux{H} }$.

    \item Attention scores $\{ a^h_{t, j} \}_{h \le \aux{H}}$, which requires a \simulator{} softmax self-attention module (\cref{def:self-attn_construct}), with at least $\aux{H}$ heads, that uses the already present $\{\vq_t, \vk_t, \vv_t\}$ in the current embeddings $\ve_t$ to re-compute the attention scores. 

    \item $\{ a_{t, j}^h (\losspartial{\vy_t^h})^{\top} \vv^{h}_j  \}_{h \le \aux{H}}$ for all $j \le \aux{T}$ by multiplying the attention scores $\{ a_{t, j}^h \}_{h \le  \aux{H}}$ with $\{ (\losspartial{\vy_t^h})^{\top} \vv^{h}_j \}_{h \le \aux{H}}$ using an MLP layer (\cref{thm:gelu_multiplication}). Furthermore, $\{ \sum_{j} a^h_{t, j} \vk^{h}_{j} \}_{ h \le \aux{H} }$ needs to be computed in parallel as well, with additional attention heads.   
    
    \item  $\losspartial{ \vy_t }$ with a \simulator{} linear self-attention module (\cref{def:self-attn_construct}), with atleast $\aux{H}$ attention heads that represent the attention score between any token $\ve_j$ and $\ve_t$ by $\{ a_{t, j}^h (\losspartial{\vy_t^h})^{\top} \vv^{h}_j  \}_{h \le \aux{H}}$, with value vectors given by $\{ \vk^h_j - \sum_{j'} a^h_{t, j'} \vk^{h}_{j'} \}_{ h \le \aux{H} }$. 
\end{itemize}

The sequential computation requires the simulator to store $\{ \{ \losspartial{\vy_t^h})^{\top} \vv^{h}_j \}_{ j \le \aux{T} } \}_{ h \le \aux{H} }$ and $\{ a^h_{t, j} \}_{h \le \aux{H}}$ in the token embedding $\ve_t$, which requires an additional $2\aux{T}\aux{H}$ embedding dimension size.
%\paragraph{Gradient Propagation}
To avoid the much-involved computation for the true gradient propagation, we instead only use the gradients w.r.t. $\vv_t$.

\paragraph{Approximate auxiliary self-attention backpropagation} We formally extend the definition of approximate gradients $\{ \losspartial{\vx_t} \}_{t=1}^{ \aux{T} }$ from \cref{def:self-attn-appr-backprop} to multi-head attention in \cref{def:self-attn-appr-backprop}.

%\ifthenelse{\boolean{arxiv}}{
\begin{definition}\label{def:self-attn-appr-backprop}
    For 
query, key, and value weights $\mW_Q, \mW_K, \mW_V \in\RR^{\aux{D}\times \aux{D}}$ and bias $\vb_{Q}, \vb_{K}, \vb_{V} \in\RR^{\aux{D}}$, the approximate backpropagation layer corresponding to a self-attention layer with $\aux{H}$ attention heads takes a sequence $\{ \losspartial{ \vy_t } \in\RR^{\aux{D}} \}_{t \le {\aux{T}}}$ and $ \{  \vx_t  \in\RR^{\aux{D}} \}_{t \le {\aux{T}}} $ as input and outputs $\{ \losspartial{ \vx_t } := \attmerge( \{ \losspartial{\vx^h_t} \} _{h \le \aux{H}} ) \}_{t \le {\aux{T}}}$, with 
\begin{align*}
\Hat{ \losspartial{\vx_t} } = \mW_V^\top \losspartial{ \vv_t } , \quad \text{ where }
\losspartial{ \vv_t } = \attmerge(\{ \sum_{j} a^h_{j, t}  \losspartial{\vy^h_j} \}_{h \le \aux{H}} ) 
\end{align*} 
Here, $\vq_t$, $\vk_t$, and $\vv_t$ refer to query, key, and value vectors at each position $t$, as defined in \cref{def:self-attn}, with the attention scores $\{a_{t, j}^h\}_{t, j \le {\aux{T}}, h \le \aux{H}}$ defined in \cref{eq:attn}. 
\end{definition}

\begin{figure}[t]
    \centering
    \includegraphics[width=0.75\textwidth]{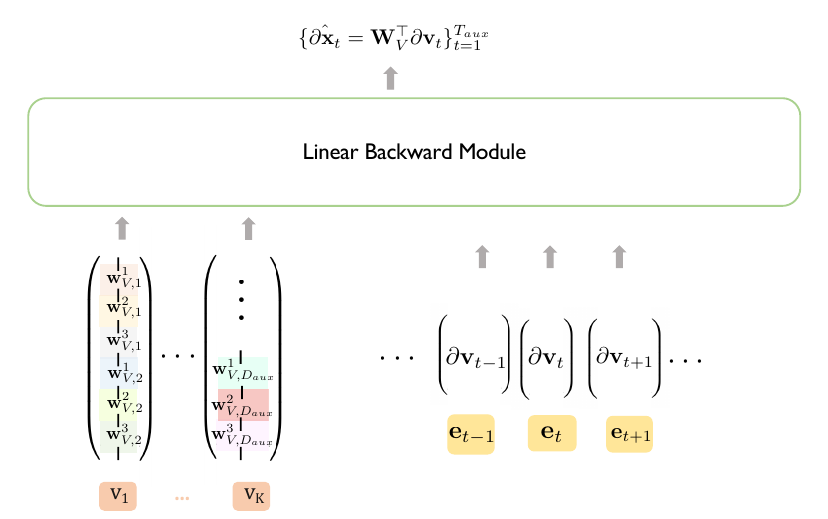}
    \caption{\simulator{} simulates the backward pass of a self-attention layer of the auxiliary model using a Linear Backward module (\cref{fig:linear_back}). The input embeddings contain the gradient of the loss w.r.t. the value vectors ($\losspartial{\vv_t}$) computed in \cref{fig:value_gradient}. The value matrix $\mW_V$ is encoded in the prefix embeddings. We call the Linear Backward module on this sequence.}
    \label{fig:attn_back}
\end{figure}

\begin{figure}[t] 
 \centering
    \includegraphics[width=0.75\textwidth]{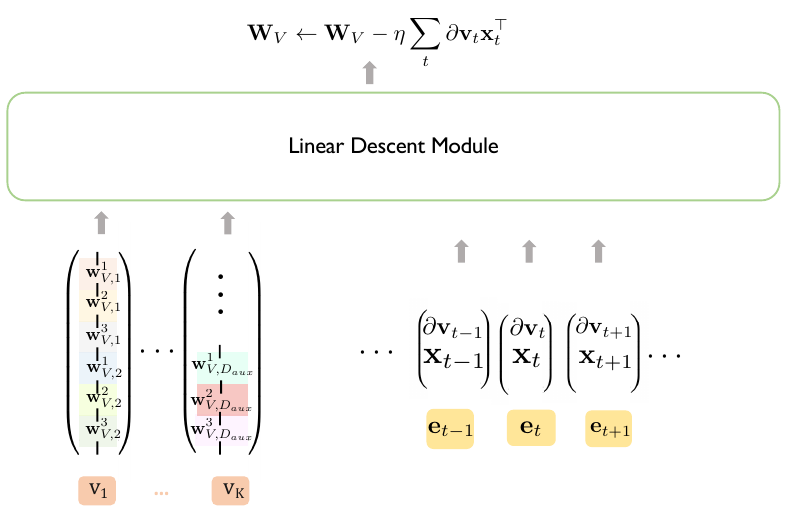}
    \caption{ \simulator{} simulates the backward pass of the self-attention layer in the auxiliary model by employing the Linear Descent module (\cref{fig:linear_descent}). The input embeddings consist of the gradient of the loss with respect to the value vectors ($\losspartial{\vv_t}$) computed in \cref{fig:value_gradient}. Additionally, we incorporate a residual connection to copy the input from the Self-attention Forward module (\cref{fig:attn_forward}) into $\vx_t$. Before invoking the Linear Descent module, we represent the value parameters ($\mW_V$) into the prefix embeddings. \simulator{} simulates the backward pass of a self-attention layer of the auxiliary model using a Linear Descent module (\cref{fig:linear_descent}). }
    \label{fig:attn_descent}
\end{figure}

\input{Theorems/Attention_backprop}

\paragraph{\simulator{} Self-attention backpropagation module} 
The input embeddings $\ve_t$ contain $\losspartial{\vy_t}$ in the first $\aux{D}$ coordinates.  Since we require to re-compute the attention scores $\{a^h_{t, j}\}_{j \le \aux{T}, h \le \aux{H} }$, we need to copy the query, key, and value vectors $\vq_t$, $\vk_t$, and $\vv_t$ from the \simulator{} self-attention Forward module at each position $t$. Furthermore, we use the residual connection to copy the prefix embeddings $\{ \vv_j \}$, which contain the rows of $\mW_V$, from the \simulator{} self-attention Forward module.

The operation can be divided into three sub-operations: Computing (a) attention scores $\{ a_{t, j}^h  \}_{h \le \aux{H} }$ for all $j \le \aux{T}$, at each position $t$, (b) $\losspartial{\vv_t}$ from $\{ a_{t, j}^h  \}_{h \le \aux{H}}$ and $\losspartial{\vy_t}$, and (c) $\Hat{\losspartial{\vx_t}}$ from $\losspartial{\vv_t}$.

\begin{itemize}
    \item Sub-operation (a): Since, the current embeddings $\ve_t$ contain $\vq_t, \vk_t$, we can simply call a self-attention attention module to compute the attention scores   $\{ a_{t, j}^h  \}_{h \le \aux{H}}$ for all $j\le T$ and store them in the current embeddings. We further retain $\losspartial{\vy_t}$ and $\vv_t$ for further operations using residual connections.

    \item Sub-operation (b): With the current embeddings $\ve_t$ containing the attention scores   $\{ a_{t, j}^h  \}_{h \le \aux{H}}$ for all $j\le T$, and the gradient $\losspartial{\vy_t}$, we can compute $\losspartial{\vv_t}$ using a \simulator{} linear
 self-attention module with atleast $\aux{H}$ attention heads, that  represent the attention scores between tokens $\ve_t$ and $\ve_j$ for any $j$ as $\{ a_{j, t}^h \}_{h \le \aux{H} }$ and use $\attsplit_{ \aux{H} } ( \losspartial{\vy_t} )$ as their value vectors. 

    \item Sub-operation (c): And finally, the computation of $\Hat{\losspartial{\vx_t}}$ is identical to the  backpropagation through a linear layer, with parameters $\mW_V$ and $\vb_V$. Hence, we call a Linear backpropagation module on the current embeddings, that contain $\losspartial {\vy_t}$ and the prefix embeddings that contain $\mW_V$ and $\vb_V$.
\end{itemize}

\paragraph{Separating sub-operations (a) and (b)} The operation for computing $\losspartial{\vv_t}$ in \cref{def:self-attn-appr-backprop} looks very similar to the computation of $\vy_t$ in \cref{eq:attnt_forward}. However, the major difference is that instead of the attention scores being $\{ a_{t, j}^h \}_{h \le \aux{H} }$ between token $t$ and any token $j$, we need the attention scores to be $\{ a_{j, t}^h \}_{h \le \aux{H} }$. Thus, unless our model allows a transpose operation on the attention scores, we need to first store them in our embeddings and then use an additional self-attention module that can pick the right attention scores between tokens using position embeddings. Please see \cref{fig:attn_back}.

%\paragraph{Value parameter update}
\paragraph{Auxiliary's value descent update} Similar to the complexity of true backpropagation, the descent updates for $\mW_Q, \vb_Q, \mW_K, \vb_K$ are quite expensive to express with the transformer layers. Hence, we focus simply on updating on $\mW_V, \vb_V$, while keeping the others fixed.

\begin{definition}[Auxiliary self-attention value descent]\label{def:self-attn-value-update}
For query, key, and value weights $\mW_Q, \mW_K, \mW_V \in\RR^{\aux{D}\times \aux{D}}$ and bias $\vb_{Q}, \vb_{K}, \vb_{V} \in\RR^{\aux{D}}$, the value descent layer corresponding to a self-attention layer with $\aux{H}$ attention heads and any function $f_{\mathrm{attn}}: \RR^{\aux{T}} \to \RR^{\aux{T}}$ takes in a batch of gradients $\{ \losspartial{ \vy_t } \in\RR^{\aux{D}} \}_{t \le \aux{T}}$ and inputs $ \{  \vx_t  \in\RR^{\aux{D}} \}_{t \le \aux{T} } $ and updates $\mW_V, \vb_V$ as follows:
\begin{align*}
&\mW_V \gets \mW_V - \eta \sum_{t \le \aux{T} } \losspartial{ \vv_t } \vx_t^{\top}, \quad \vb_V \gets \vb_V - \eta \sum_{t \le \aux{T} } \losspartial{ \vv_t }, \\& \text{ where } \losspartial{ \vv_t } = \attmerge(\{ \sum_{j} a^h_{j, t}  \losspartial{\vy^h_j} \}_{h \le \aux{H} } ) 
\end{align*} 
Here, $\vv_t$ refers to value vectors at each position $t$, as defined in \cref{def:self-attn}. 
\end{definition}

\paragraph{\simulator{} Self-attention descent module} 
The input embeddings contain $\losspartial{\vv_t}$ in the first $\aux{D}$ coordinates, from the \simulator{} self-attention backpropagation module. Furthermore, the prefix embeddings $\{ \vv_j \}$ contain the stacked rows of $\mW_V$ and $\vb_V$, continuing from the \simulator{} self-attention backpropagation module.

Since we further need the input $\vx_t$ to the auxiliary self-attention layer under consideration, we use residual connections to copy $\vx_t$ from the \simulator{} self-attention Forward module at each position $t$.

The updates of $\mW_V$ and $\vb_V$ are equivalent to the parameter update in a linear layer, involving gradients $\{ \losspartial{\vv_t} \}$ and input $\{ \vx_t \}$. Thus, we call a Linear descent module on the current embeddings and the prefix embeddings to get the updated value parameters. Please see \cref{fig:attn_descent}.

\subsection{Proofs of theorems and gradient definitions}
We restate the theorems and definitions, before presenting their proofs for easy referencing.

\defselfattnbackprop*
\begin{proof}[Derivation of gradient in \cref{def:self-attn-backprop}]
    Recalling the definition of $\vy_t$ from \cref{def:self-attn},
    \begin{align*}
        &\vy_t = \attmerge ( \{ \sum_{j \le \aux{T}} a^h_{t, j}  \vv^{h}_j  \}_{h \le \aux{H}} ); \quad
        a^h_{t, j} =  \mathrm{softmax} ( \mK^{h} \vq^{h}_t )_j, \\&
        \vq_t = \mW_Q \vx_t + \vb_Q  \quad \vk_t =  \mW_K \vx_t + \vb_K, \quad \vv_t = \mW_V \vx_t + \vb_V.
    \end{align*}
    $\vq^{h}_t, \vk^{h}_t, \vv^{h}_t$ denote $\attsplit_{\aux{H}} (\vq_t)_h$, $\attsplit_{\aux{H}} (\vk_t)_h$, and $\attsplit_{\aux{H}} (\vv_t)_h$ respectively for all $t \le \aux{T}$, and $h \le \aux{H}$. 
     $\mK^h \in \RR^{\aux{T} \times \aux{D} }$ is defined with its rows as $\{\vk^h_t\}_{t \le \aux{T}}$ for all $h \le \aux{H}$.

    We explain the proof for an arbitrary token position $t$.
    With the application of the chain rule, we have
    \begin{align*}
        \losspartial{\vx_t} &= ( \frac{ \partial \vq_t }{ \partial \vx_t } )^{\top}  \losspartial{\vq_t} +   ( \frac{ \partial \vk_t }{ \partial \vx_t } )^{\top}  \losspartial{\vk_t} +
        ( \frac{ \partial \vv_t }{ \partial \vx_t } )^{\top}  \losspartial{\vv_t} \\&
        = \mW_Q^{\top} \losspartial{\vq_t} + \mW_K^{\top} \losspartial{\vk_t} + \mW_V^{\top} \losspartial{\vv_t},
    \end{align*}
    where the second step follows from the definitions of $\vq_t, \vk_t, $ and $\vv_t$ respectively.

    \paragraph{Computation of $ \losspartial {\vq_t}$: } 
    With the $\attsplit$ operation of $\vq_t$ across $\aux{H}$ heads for the computation of $\vy_t$, the computation of the backpropagated gradient $\losspartial{\vq_t}$ itself needs to be split across $\aux{H}$ heads. Furthermore, query vector $\vq_t$ only affects $\vy_t$, implying $\frac{ \partial \vy_{t'}}{\partial \vq_t} = 0$ for any $t' \ne t$.
    Thus, we have for any head $h \le \aux{H}$, if $\vy_t^h$ represents the output of attention head $h$, given by $\sum_{j \le \aux{T}} a^h_{t, j} \vv_j^h$, 
    \begingroup
    \allowdisplaybreaks
    \begin{align}
        \losspartial{\vq_t^h} &= ( \frac{ \partial \vy_t^h }{ \partial \vq_t^h } )^{\top} \losspartial{\vy_t^h} \nonumber \\
        &= \sum_{j \le \aux{T}}  \langle \vv_j^h, \losspartial{\vy_t^h} \rangle \frac{ \partial a^h_{t, j} }{ \partial \vq_t^h } \nonumber  \\ &= 
        \sum_{j \le \aux{T}}  \langle \vv_j^h, \losspartial{\vy_t^h} \rangle \frac{ \partial }{ \partial \vq_t^h } \left( \frac{ e^{ \langle \vk_j^h , \vq_t^h \rangle }  }{ \sum_{t' \le \aux{T}} e^{  \langle \vk_{t'}^h , \vq_t^h \rangle } } \right) \label{eq:softmax_expand}  \\
        &= \sum_{j \le \aux{T}}  \langle \vv_j^h, \losspartial{\vy_t^h} \rangle  \left[   \frac{1}{ \sum_{t' \le \aux{T}} e^{  \langle \vk_{t'}^h , \vq_t^h \rangle }  } \frac{ \partial e^{ \langle \vk_j^h , \vq_t^h \rangle } }{ \partial \vq_t^h }  -  \left( \frac{ e^{ \langle \vk_j^h , \vq_t^h \rangle }  }{ (\sum_{t' \le \aux{T}} e^{  \langle \vk_{t'}^h , \vq_t^h \rangle })^2  }\right)  \sum_{j' \le \aux{T}} \frac{ \partial e^{ \langle \vk_{j'}^h , \vq_t^h \rangle } }{ \partial \vq_t^h }  \right]
        \label{eq:query_expand_fractionderivation} \\
        &= \sum_{j \le \aux{T}}  \langle \vv_j^h, \losspartial{\vy_t^h} \rangle  \left[  \left( \frac{ e^{ \langle \vk_j^h , \vq_t^h \rangle }  }{ \sum_{t' \le \aux{T}} e^{  \langle \vk_{t'}^h , \vq_t^h \rangle } } \right) \vk_j^h  - \left( \frac{ e^{ \langle \vk_j^h , \vq_t^h \rangle }  }{ \sum_{t' \le \aux{T}} e^{  \langle \vk_{t'}^h , \vq_t^h \rangle } } \right) \sum_{j' \le \aux{T}} \left( \frac{ e^{ \langle \vk_{j'}^h , \vq_t^h \rangle }  }{ \sum_{t' \le \aux{T}} e^{  \langle \vk_{t'}^h , \vq_t^h \rangle } } \right) \vk^h_{j'}  \right] \label{eq:query_expanded_fractionderivation}  \\&
        = \sum_{j \le \aux{T}}  a^h_{t, j} \langle \vv_j^h, \losspartial{\vy_t^h} \rangle \left( \vk_j^h - \sum_{j' \le \aux{T}} a^h_{t, j'}  \vk^h_{j'} \right). \nonumber 
    \end{align}
    \endgroup
    
    In \cref{eq:softmax_expand}, we have expanded the definition of $\mathrm{softmax}$ in $a^h_{t, j} := \mathrm{softmax}(\mK^h \vq_t^h)_j$ in order to better motivate the derivative of $a^h_{t, j}$ w.r.t. $\vq_t^h$. Finally, $\losspartial{\vq_t}$ is given by $\attmerge(\{ \losspartial{ \vq_t^h } \}_{ h \le \aux{H} } ).$
    
    \paragraph{Computation of $ \losspartial {\vk_t}$: } Continuing as the computation of $\losspartial {\vq_t}$, we split the computation of $ \losspartial {\vk_t} $ across the $\aux{H}$ attention heads. However, unlike $\vq_t$, $\vk_t$ affects $\vy_{j}$ for all $j \le \aux{T}$. For any head $h \le \aux{H}$, we follow the chain-rule step by step to get
    \begingroup
    \allowdisplaybreaks
    \begin{align}
        \losspartial{\vk_t^h} &= \sum_{ j \le \aux{T} } ( \frac{ \partial \vy_{j}^h }{ \partial \vk_t^h } )^{\top} \losspartial{\vy_j^h} = \sum_{ j \le \aux{T} } \left( \frac{ \partial \sum_{j' \le \aux{T}} a_{j, j'} \vv_{j'}^h  }{ \partial \vk_t^h } \right)^{\top} \losspartial{\vy_j^h} \nonumber \\
        &= \sum_{ j \le \aux{T} } \langle \vv_t^h, \losspartial{\vy_{j}^h} \rangle \frac{ \partial a^h_{j, t} }{ \partial \vk_t^h }  + \sum_{j \le \aux{T}} \sum_{ j' \le \aux{T}; j' \ne t } \langle \vv_{j'}^h, \losspartial{\vy_{j}^h} \rangle \frac{ \partial a^h_{j, j'} }{ \partial \vk_t^h } \label{eq:sep_t_nott} \\
        &=   \sum_{ j \le \aux{T} } \langle \vv_t^h, \losspartial{\vy_{j}^h} \rangle \frac{ \partial }{ \partial \vk_t^h } \left( \frac{ e^{ \langle \vk_t^h, \vq_j^h \rangle } }{ \sum_{t' \le \aux{T}} e^{ \langle \vk_{t'}^h, \vq_j^h \rangle } }  \right) \label{eq:sumt_step1}\\&
        + \sum_{ j \le \aux{T} } \sum_{ j' \le \aux{T}; j' \ne t } 
 \langle \vv_{j'}^h, \losspartial{\vy_{j}^h} \rangle \frac{ \partial }{ \partial \vk_t^h } \left( \frac{ e^{ \langle \vk_{j'}^h, \vq_j^h \rangle } }{ \sum_{t' \le \aux{T}} e^{ \langle \vk_{t'}^h, \vq_j^h \rangle } }  \right)  \label{eq:sumnott_step1}\\&
= \sum_{ j \le \aux{T} } \langle \vv_t^h, \losspartial{\vy_{j}^h} \rangle \left[ \left( \frac{ e^{ \langle \vk_t^h, \vq_j^h \rangle } }{ \sum_{t' \le \aux{T}} e^{ \langle \vk_{t'}^h, \vq_j^h \rangle } }  \right) \vq_j^h - \left( \frac{ e^{ \langle \vk_t^h, \vq_j^h \rangle } }{ \sum_{t' \le \aux{T}} e^{ \langle \vk_{t'}^h, \vq_j^h \rangle } }  \right)^2  \vq_{j}^h \right]  \label{eq:sumt_step2}\\&
 - \sum_{ j \le \aux{T} } \sum_{ j' \le \aux{T}; j' \ne t } \langle \vv_{j'}^h, \losspartial{\vy_{j}^h}  \rangle \left( \frac{ e^{ \langle \vk_{j'}^h, \vq_j^h \rangle } }{ \sum_{t' \le \aux{T}} e^{ \langle \vk_{t'}^h, \vq_j^h \rangle } }  \right) \left( \frac{ e^{ \langle \vk_t^h, \vq_j^h \rangle } }{ \sum_{t' \le \aux{T}} e^{ \langle \vk_{t'}^h, \vq_j^h \rangle } }  \right) \vq_j^h  \label{eq:sumnott_step2}\\&
 = \sum_{ j \le \aux{T} } \langle \vv_t^h, \losspartial{\vy_{j}^h} \rangle ( a^h_{j, t} - (a^h_{j, t})^2  ) \vq_j^h - \sum_{ j \le \aux{T} } \sum_{ j' \le \aux{T}; j' \ne t } \langle \vv_{j'}^h, \losspartial{\vy_{j}^h}  \rangle a^h_{j, j'} a^h_{j, t} \vq_j^h  \nonumber\\&
 = \sum_{ j \le \aux{T} }  a^h_{j, t}   \langle \losspartial{\vy_j^h},  \vv^{h}_t - \sum_{j'} a^h_{j, j'} \vv^{h}_{j'} \rangle   \vq^{h}_{j}  \nonumber
    \end{align}
    \endgroup

In \cref{eq:sep_t_nott}, we separate the inside sum into two components, since the derivative w.r.t. $\vk_t^h$ differ for the two components, as outlined in the derivation of \cref{eq:sumt_step2} from \cref{eq:sumt_step1}, and \cref{eq:sumnott_step2} from \cref{eq:sumnott_step1}. We have skipped a step going from  \cref{eq:sumnott_step1,eq:sumt_step1} to  \cref{eq:sumt_step2,eq:sumnott_step2} due to typographical simplicity. The skipped step is extremely similar to \cref{eq:query_expand_fractionderivation} in the derivation of $\losspartial{\vq_t^h}.$ Finally, $\losspartial{\vk_t}$ is given by $\attmerge(\{ \losspartial{ \vk_t^h } \}_{ h \le \aux{H} } ).$

\paragraph{Computation of $ \losspartial {\vv_t}$:} Similar to the gradient computation of $\vq_t$, the computation of $ \losspartial {\vv_t} $ needs to be split across the $\aux{H}$ attention heads. However, like $\vk_t$, $\vv_t$ affects $\vy_j$ for all $j \le \aux{T}$. For any head $h \le \aux{H}$, we follow the chain-rule step by step to get
\begin{align*}
    \losspartial{\vv_t^h} &= \sum_{ j \le \aux{T} } ( \frac{ \partial \vy_{j}^h }{ \partial \vv_t^h } )^{\top} \losspartial{\vy_j^h}  = \sum_{ j \le \aux{T} } \left( \frac{ \partial \sum_{j' \le \aux{T}} a_{j, j'} \vv_{j'}^h  }{ \partial \vv_t^h } \right)^{\top} \losspartial{\vy_j^h} 
    = \sum_{ j \le \aux{T} } a_{j, t}^h \losspartial{\vy_j^h}
\end{align*}

\end{proof}

\boundsofthardattnt*
\begin{proof}[Proof of \cref{thm:attn_backprop}]
    For typographical simplicity, we discuss the proof at an arbitrary position $t$. Recall the definition of an $\varepsilon$-hard attention head from \cref{def:hard_attn}. An attention head is defined to be $\varepsilon$-hard on an input sequence $\{ \vx_t \}_{t=1}^{\aux{T}}$, if for each position $t$, there exists a position $t_0$ such that the attention score $a_{t, t_0} \ge 1 - \varepsilon$.

    For the proof, we simply focus on $\losspartial{\vq_t}$, and the proof for $\losspartial{\vk_t}$ follows like-wise. 
    \paragraph{Bounds on $\vq_t$:}
    Recalling the definition of $\losspartial{\vq_t}$ from \cref{def:self-attn-backprop}, we have
    \begin{align*}
        \losspartial{ \vq_t } =  \attmerge( \{  \sum_j a^h_{t, j} ( (\losspartial{\vy_t^h})^{\top} \vv^{h}_j ) [ \vk^h_j - \sum_{j'} a^h_{t, j'} \vk^{h}_{j'} ]  \}_{h \le \aux{H}} ).
    \end{align*}
    
    Focusing on a head $h \le \aux{H}$, define $\losspartial{ \vq_t^h } = \sum_j a^h_{t, j} ( (\losspartial{\vy_t^h})^{\top} \vv^{h}_j ) [ \vk^h_j - \sum_{j'} a^h_{t, j'} \vk^{h}_{j'} ]$ and $t_0 \le \aux{T}$ as the token position where the $\vq_t$ attends the most to, i.e. $a^h_{t, t_0} \ge 1 - \varepsilon$ and $\sum_{j \le \aux{T}; j \ne t_0} a^h_{t, j} \le \varepsilon$. Then,
    \begin{align*}
        \norm{ \losspartial { \vq^h_t } }_2 &= \norm{ \sum_j a^h_{t, j} ( (\losspartial{\vy_t^h})^{\top} \vv^{h}_j ) [ \vk^h_j - \sum_{j'} a^h_{t, j'} \vk^{h}_{j'} ] }_2 \\&
        = \norm{ a^h_{t, t_0} ( (\losspartial{\vy_t^h})^{\top} \vv^{h}_{t_0} ) [ \vk^h_{t_0} - \sum_{j'} a^h_{t, j'} \vk^{h}_{j'} ]  + \sum_{j \ne t_0} a^h_{t, j} ( (\losspartial{\vy_t^h})^{\top} \vv^{h}_j ) [ \vk^h_j - \sum_{j'} a^h_{t, j'} \vk^{h}_{j'} ] }_2  \\&
        \leq \underbrace{\norm{ a^h_{t, t_0} ( (\losspartial{\vy_t^h})^{\top} \vv^{h}_{t_0} ) [ \vk^h_{t_0} - \sum_{j'} a^h_{t, j'} \vk^{h}_{j'} ] }_2}_{\mathrm{Term 1}} + \underbrace{\norm{ \sum_{j \ne t_0} a^h_{t, j} ( (\losspartial{\vy_t^h})^{\top} \vv^{h}_j ) [ \vk^h_j - \sum_{j'} a^h_{t, j'} \vk^{h}_{j'} ] }_2}_{\mathrm{Term 2}}, 
    \end{align*}
    where the final step uses a Cauchy-Schwartz inequality. We focus on the two terms separately.

    \begin{enumerate}
        \item $\mathrm{Term 1}$: Focusing on $\vk^h_{t_0} - \sum_{j'} a^h_{t, j'} \vk^{h}_{j'}$, we have
        \begin{align}
            \norm{\vk^h_{t_0} - \sum_{j'} a^h_{t, j'} \vk^{h}_{j'}}_2 &= \norm{ (1 - a_{t, t_0}) \vk^h_{t_0} - \sum_{j' \ne t_0} a^h_{t, j'} \vk^{h}_{j'} }_2 \nonumber\\&
            \leq (1 - a_{t, t_0}) \norm{ \vk^h_{t_0} }_2 + \sum_{j' \ne t_0}  a^h_{t, j'} \norm{ \vk^{h}_{j'} }_2 \nonumber\\&
            \leq  ( (1 - a_{t, t_0}) + \sum_{j' \ne t_0}  a^h_{t, j'} ) \max_{j} \norm{ \vk^h_j }_2 \nonumber\\&
            \leq 2 \varepsilon \max_{j} \norm{ \vk^h_j }_2 \label{eq:bound_on_diffk}.
        \end{align}
        We use a Cauchy-Schwartz inequality in the second and third steps and the attention head behavior in the final step.

        Hence, $\mathrm{Term 1}$ can now be bounded as follows:
        \begin{align*}
            \norm{ a^h_{t, t_0} ( (\losspartial{\vy_t^h})^{\top} \vv^{h}_{t_0} ) [ \vk^h_{t_0} - \sum_{j'} a^h_{t, j'} \vk^{h}_{j'} ] }_2 &=  a^h_{t, t_0}  \abs{ (\losspartial{\vy_t^h})^{\top} \vv^{h}_{t_0} } \norm{  \vk^h_{t_0} - \sum_{j'} a^h_{t, j'} \vk^{h}_{j'} }_2, \\&
            \leq 2 \varepsilon \norm{ \losspartial{\vy_t^h} }_2 \norm{ \vv^{h}_{t_0} }_2   \max_{j} \norm{ \vk^h_j }_2.
        \end{align*}
        In the final step, in addition to the bound from \cref{eq:bound_on_diffk}, we use a  Cauchy-Schwartz inequality to bound $\abs{ (\losspartial{\vy_t^h})^{\top} \vv^{h}_{t_0} }$ and bound the attention score $a^h_{t, t_0} $ by $1$.
        
        \item $\mathrm{Term 2}$: Focusing on $\vk^h_{j} - \sum_{j'} a^h_{t, j'} \vk^{h}_{j'}$ for any $j \le \aux{T}$, we have using two Cauchy-Schwartz inequalities:
        \begin{align}
            \norm{\vk^h_{j} - \sum_{j'} a^h_{t, j'} \vk^{h}_{j'}}_2 &\leq \norm{\vk^h_j}_2 + \norm{ \sum_{j'} a^h_{t, j'} \vk^{h}_{j'} }_2 
            \leq (1 + \sum_{j'} a^h_{t, j'}) \max_{j'} \norm{ \vk^{h}_{j'} }_2  = 2 \max_{j'} \norm{ \vk^{h}_{j'} }_2 \label{eq:bound_k_j_diff}. 
        \end{align}
        Hence,
        \begin{align*}
            \norm{ \sum_{j \ne t_0} a^h_{t, j} ( (\losspartial{\vy_t^h})^{\top} \vv^{h}_j ) [ \vk^h_j - \sum_{j'} a^h_{t, j'} \vk^{h}_{j'} ] }_2 &\leq  \left( \sum_{ j \ne t_0 } a^h_{t, j} \right) \max_j \abs{ (\losspartial{\vy_t^h})^{\top} \vv^{h}_j }  \norm{ \vk^h_j - \sum_{j'} a^h_{t, j'} \vk^{h}_{j'} }_2 \\&
            \leq 2\varepsilon \norm{ \losspartial{ \vy_t^h } }_2 \left( \max_j \norm{ \vv_j^h }_2 \right) \left( \max_{j'} \norm{ \vk^{h}_{j'} }_2 \right)
            .
        \end{align*}
        In the final step, in addition to the bound from \cref{eq:bound_k_j_diff}, we use a  Cauchy-Schwartz inequality to bound $\abs{ (\losspartial{\vy_t^h})^{\top} \vv^{h}_j }$ and use the $\varepsilon$-hard behavior of the attention head to bound $\sum_{j \ne t_0} a^h_{t, j}.$

    \end{enumerate}
    Combining the bounds on both terms, we have
    \begin{align*}
        \norm{ \losspartial { \vq^h_t } }_2 &\leq 2 \varepsilon \norm{ \losspartial{\vy_t^h} }_2 \norm{ \vv^{h}_{t_0} }_2   \max_{j} \norm{ \vk^h_j }_2 + 2\varepsilon \norm{ \losspartial{ \vy_t^h } }_2 \left( \max_j \norm{ \vv_j^h }_2 \right) \left( \max_{j'} \norm{ \vk^{h}_{j'} }_2 \right) \\& \leq 4 \varepsilon \norm{ \losspartial{ \vy_t^h } }_2 \left( \max_j \norm{ \vv_j^h }_2 \right) \left( \max_{j'} \norm{ \vk^{h}_{j'} }_2 \right).
    \end{align*}
    We bound the remaining terms as follows.
    \begin{itemize}
        \item $\norm{ \losspartial{ \vy_t^h } }_2 \le B_y$, under the bounded assumption of the gradients.

        \item For any $j \le \aux{T}$, we have $\norm{ \vk_j^{h} }_2 \leq \norm{ \vk_j }_2$ since $\vk_j = \attmerge( \{ \vk_j^{h'} \}_{h' \in \aux{H} } )$. Furthermore, from the defintion of the key vector $\vk_j$, $\norm{ \vk_j }_2 = \norm{ \mW_K \vx_j + \vb_K }_2 \leq \norm{\mW_K}_2 \norm{\vx_j}_2 + \norm{\vb_K}_2$ with a Cauchy-Schwartz inequality. Under the bounded assumptions of $\mW_K, \vb_K$ and input $\vx_j$, we have $\norm{ \vk_j }_2 \leq B_w(1 + B_x).$

        \item Similar procedure can be followed for bounding $ \max_j \norm{ \vv_j^h }_2$. 
    \end{itemize}

    Thus, we have $\norm{ \losspartial { \vq^h_t } }_2 \leq 4 \varepsilon \norm{ \losspartial{ \vy_t^h } }_2 \left( \max_j \norm{ \vv_j^h }_2 \right) \left( \max_{j'} \norm{ \vk^{h}_{j'} }_2 \right) \leq 4 \varepsilon B^2_w (1 + B_x)^2 B_y.$

    \paragraph{Bounds on $\norm{\Hat{\losspartial{\vx_t}} - \losspartial{\vx_t}}_2$:}
    From the definitons of $\Hat{\losspartial{\vx_t}}$ and $\losspartial{\vx_t}$ from \cref{def:self-attn-appr-backprop,def:self-attn-appr-backprop}, we have
    
    \begin{align*}
        \norm{\Hat{\losspartial{\vx_t}} - \losspartial{\vx_t}}_2 
        &= \norm{ \mW_K^{\top} \losspartial{\vk_t} + \mW_Q^{\top} \losspartial{\vq_t} }_2 \leq \norm{ \mW_K }_2 \norm{ \losspartial{\vk_t} }_2 +  \norm{ \mW_Q }_2 \norm{ \losspartial{\vq_t} }_2
        \\& \leq 8 \varepsilon B^3_w (1 + B_x)^2 B_y = \mathcal{O}(  \varepsilon B^3_w  B_x^2 B_y ),
    \end{align*}
    where we use Cauchy-schwartz inequality in the second step. We use the assumed bounds on $\norm{ \mW_Q }_2, \norm{ \mW_K }_2$, and the computed bounds on $\norm{ \losspartial{\vq_t} }_2, \norm{ \losspartial{\vk_t} }_2$ in the pre-final step.
\end{proof}

%% file: Theorems/Attention_backprop.tex
In the upcoming theorem, we formally show that if on a given sequence $\{\vx_t\}_{t \le \aux{T}}$, for all token positions all the attention heads in a self-attention layer primarily attend to a single token, then the approximate gradient $\Hat{\losspartial{\vx_t}}$ is  close to the true gradient $\losspartial{\vx_t}$ at each position $t$. 

\begin{definition}[$\varepsilon$-hard attention head]\label{def:hard_attn}
For the Self-Attention layer of $\aux{H}$ heads in \cref{def:self-attn}, on a given input sequence $\{\vx_t\}_{t=1}^{\aux{T}}$, an attention head $h \le \aux{H}$ is defined to be $\varepsilon$-hard on the input sequence, if for all positions $t \le \aux{T}$, there exists a position $t_0 \le \aux{T}$ such that $a_{t, t_0}^h \ge 1 - \varepsilon$.
\end{definition}

%\begin{definition}
%g
%\end{definition}

%\begin{theorem}\label{thm:attn_backprop}
\begin{restatable}{theorem}{boundsofthardattnt}\label{thm:attn_backprop}
With the notations in \cref{def:self-attn,def:self-attn-backprop,def:self-attn-appr-backprop}, if on a given input sequence $\{\vx_t\}_{t=1}^{\aux{T}}$, with its query, key, and value vectors $\{\vq_t, \vk_t, \vv_t\}_{t=1}^{\aux{T}}$, all the $\aux{H}$ attention heads are $\varepsilon$-hard for some $\varepsilon > 0$, then for a given sequence of gradients $\{\losspartial{\vy_t}\}_{t=1}^{\aux{T}}$, 
\begin{align*}
    \norm{\losspartial{\vq_t}}_2, \norm{\losspartial{\vk_t}}_2 \le \mathcal{O}( \varepsilon B_x^2 B^2_w B_y ), \quad \text{ for all } t \le \aux{T},
\end{align*}
where $B_x = \max_{t \le \aux{T}} \norm{\vx_t}_2$, $B_y = \max_{t \le \aux{T}} \norm{ \losspartial{\vy_t} }_2$, and $B_w = \max\{ \norm{\mW_K}_2, \norm{\mW_Q}_2, \norm{\mW_V}_2, \norm{\vb_V}_2, \norm{\vb_K}_2, \norm{\vb_V}_2 \} $.

This implies, for each position $t$, $\norm{ \Hat{\losspartial{\vx_t}} - \losspartial{\vx_t} }_2 \le \mathcal{O}(\varepsilon B_x^2 B_w^3 B_y).$
%$\max\{ \norm{\mW_K}_2, \norm{\mW_V}_2, \norm{\vb_K}_2, \norm{\vb_V}_2 \}$,  
\end{restatable}
%\end{theorem}

%% file: Layernorm_simulation.tex
\section{Layer normalization}\label{sec:LNappendix}
%Before defining layer normalization, we first define the function $\normalize: \RR^N \to \RR^N$. $\normalize$ takes in input $\vx$ and returns $(\vx - \mu)/\sigma$, where $\mu = \frac{1}{N} \sum_{i=1}^{N} x_i$ and $\sigma = ( \frac{1}{N} \sum_{i=1}^N (x_i - \mu)^2 )^{1/2}.$ 

%\begin{definition}[Layer normalization]\label{def:layernorm}
%    For a linear vector $\vgamma \in \RR^N$ and bias $\vb \in \RR^N$, a Layer normalization layer takes $\vx\in\RR^N$ as input and outputs $\vy = \vgamma \odot \normalize(\vx) + \vb $.
%\end{definition}
%We repeat the definition of layer normalization from the main paper below.
%\layernormdefine*

\begin{restatable}{definition}{layernormdefine}[Layer Normalization]\label{def:layernorm}
Define a normalization function $f:\RR^d\to\RR^d$ that performs $f(\vx) = (\vx-\mu) / \sigma$, where $\mu$ and $\sigma$ are the mean and standard deviation of $\vx$, respectively.
	Then, layer normalization with parameters $\vgamma, \vb \in \RR^{\aux{D}}$ takes as input  $\vx \in \RR^{\aux{D}}$ and outputs $\vy\in\RR^{\aux{D}}$, which is computed as $\vz = f(\vx), \vy = \vgamma \odot \vz + \vb.$
\end{restatable}
\begin{restatable}{definition}{layernormbackprop}[Exact Gradient for Layer Normalization]\label{def:layernorm_backprop}
	Using notations in \cref{def:layernorm}, given the gradient of the loss w.r.t the output of the Layer Normalization $\losspartial{\vy}$, backpropagation computes $\losspartial{\vx}$	as
	\begin{align*}
	    \losspartial{\vx} = ( \losspartial{\vz} -  {\aux{D}}^{-1} 
 \sum_{i=1}^{ \aux{D} } \losspartial{z_i} -  \langle \losspartial{\vz}, \vz \rangle \vz ) / \sigma \qquad \losspartial{\vz} = \vgamma \odot \losspartial{\vy}.
	\end{align*} 
\end{restatable}
\looseness-1 Exact backpropagation is expensive because $\langle \losspartial{ \vz }, \vz \rangle \vz$ requires using at least two sequential MLPs. We thus approximate it with a first-order Taylor expansion, which is entry-wise close to the true gradient.
% to approximate the above operation.
\begin{restatable}{definition}{layernormapprgradient}[$\epsilon$-approximate Layer Normalization Gradient]\label{def:layernorm_appr_backprop}
     With notations defined above, this layer takes $\losspartial{ \vy }, \vx \in\RR^{\aux{D}}$ as input and outputs
   $\Hat{ \losspartial{ \vx } } = \frac{1}{\epsilon} ( f(\vx + \epsilon \gamma \odot \losspartial{ \vy } ) - f(\vx) ).$
\end{restatable}

In the discussions below, we consider a layer normalization layer in the auxiliary model with parameters $\{\vgamma, \vb\}$ that takes in input sequence $\vx_1, \cdots, \vx_{ \aux{T} }$ and outputs $\vy_1, \cdots, \vy_{ \aux{T} }$, with $\vy_t = \vgamma \odot \vz_t + \vb; \vz_t = f(\vx_t)$ for each $t \le \aux{T}$. Since this involves a token-wise operation, we will present our constructed modules with a general token position $t$ and the prefix tokens $\{\vv_j\}.$ We will use $\mW_{\vgamma}$ as a diagonal matrix in $\RR^{\aux{D} \times \aux{D}}$, containing $\vgamma$ on its main diagonal.

\paragraph{\simulator{} Layer normalization Forward module}
The input embedding to this  module $\ve_t$ will contain $\vx_t$ in its first $\aux{D}$ coordinates. 
The layer normalization computation can be divided into two sub-operations: (a) application of $\normalize$, and (b) linear computation using $\vgamma, \vb$. We will present a \simulator{} module for each sub-operation.

We can represent the function $\normalize$ using a layer normalization operation itself, with its weight and bias parameters set as $\mathbf{1}$ and $\mathbf{0}$ respectively. However, since the relevant input exists only in the first $\aux{D}$ coordinates, the operation on the first $\aux{D}$ coordinates needs to be independent of the rest of the coordinates. To do so, we instead use Group normalization (\cref{def:group_norm}) on $\ve_t$, with groups of size $\aux{D}$.

Now, the embedding $\ve_t$ contains $\normalize(\vx_t)$ in its first $\aux{D}$ coordinates. The second sub-operation can then be viewed as a Linear Layer computation, i.e. $\vy_t = \mW_{\vgamma} \vx_t + \vb$. Hence, we simply stack the rows of $\mW_{\vgamma}$ and $\vb_{\vgamma}$ onto the prefix tokens $\{ \vv_j \}$ and call the \simulator{} Linear Forward module (\cref{sec:Linearappendix}).

%To compute (a), we can use group norm on $\ve_t$, with the groups of size $\aux{D}$. To compute (b), we simply use the module for the forward computation of the Linear layer, with $W_{\vgamma}, \vb$ stored in the $\{ \vv_j \}$, where $W_{\vgamma} \in \RR^{\aux{D} \times \aux{D}}$ is a diagonal matrix formed from $\vgamma$.

\paragraph{Auxiliary's gradient backpropagation through layer normalization } With the definition of layer normalization and the normalization function $f$ in \cref{def:layernorm}, the auxiliary's backpropagation operation takes in the loss gradient w.r.t. output ($\losspartial{\vy}$) and computes the loss gradient  w.r.t. input ($\losspartial{\vx}$).

\iffalse
\begin{definition}[Layer normalization backpropagation]\label{def:layernorm_backprop}
    For a linear vector $\vgamma \in \RR^{\aux{D}}$ and bias $\vb \in \RR^{\aux{D}}$, a Layer normalization backpropagation layer takes $\losspartial{ \vy }, \vx \in\RR^{\aux{D}}$ as input and outputs
    \begin{align*}
        \losspartial{ \vx } = \frac{1}{\sigma} \left( \losspartial{ \vz } - \normalize( \losspartial{ \vz } ) - \frac{1}{\aux{D}} \langle \losspartial{ \vz }, \vz \rangle \vz \right), \quad \mu = \frac{1}{\aux{D}} \sum_{i=1}^{\aux{D}} x_i, \quad \sigma = ( \frac{1}{\aux{D}} \sum_{i=1}^{\aux{D}} (x_i - \mu)^2 )^{1/2},
    \end{align*}
    where $\vz$ represents $\normalize(  \vx )$ and $\losspartial{ \vz }$ is given by $\gamma \odot \losspartial{\vy}.$
\end{definition}
\fi

\layernormbackprop*
\paragraph{Complexity of true backpropagation} The above operation is computation heavy since it involves computing (a) $\losspartial{ \vz }$, (b) $\normalize( \losspartial{ \vz } )$, (c) $\langle \losspartial{ \vz }, \vz \rangle \vz$, and (d) multiplying by a factor of $\frac{1}{\sigma}$. $\langle \losspartial{ \vz }, \vz \rangle \vz$ in itself will require two MLP layers, following \cref{thm:gelu_multiplication}. In order to reduce the number of layers, we turn to first-order Taylor expansion for approximating the above operation.

\iffalse
\begin{definition}[First-order layer normalization backpropagation]%\label{def:layernorm_appr_backprop}
    For a linear vector $\vgamma \in \RR^{\aux{D}}$ and bias $\vb \in \RR^{\aux{D}}$ and an additional hyperparameter $\epsilon$, a "First-order" layer normalization backpropagation layer takes $\losspartial{ \vy } \in\RR^{\aux{D}}$ as input and outputs
    \begin{align*}
        \Hat{ \losspartial{ \vx } } = \frac{1}{\epsilon} ( \normalize(\vx + \epsilon  \losspartial{ \vz } ) - \normalize(\vx) ),
    \end{align*}
    where $\vz$ represents $\normalize(  \vx )$ and $\losspartial{ \vz }$ is given by $\gamma \odot \losspartial{\vy}.$
\end{definition}
\fi

\layernormapprgradient*

\input{Theorems/LN_backprop}

\paragraph{\simulator{} Layer normalization backpropagation module}
The input embeddings $\ve_t$ contain $\losspartial{\vy_t}$ at each position $t$ in the first $\aux{D}$ coordinates. Since we further need the input to the auxiliary's layer normalization layer under consideration, we copy $\vx_t$ from the \simulator{} Layer normalization Forward module at each position $t$ using residual connections. Furthermore, residual connections have been used to copy the contents of the prefix tokens $\{ \vv_j \}$ from the Layer normalization Forward module, which contain $\mW_{\vgamma}, \vb$. Recall that for ease of presentation, we use $\vz_t$ to represent $\normalize(\vx_t)$.

%\cref{thm:first_order_LN_backprop} shows that $\losspartial{ \vx }$ and $\Hat{ \losspartial{ \vx } }$ differ by atmost $\mathcal{O}(\epsilon)$, which can be reduced with decreasing values of $\epsilon$.
We set $\epsilon$ as a hyperparameter and return $\Hat{ \losspartial{ \vx } }$ as the output of this module. The computation of $\Hat{ \losspartial{ \vx } }$ can be divided into two sub-operations: (a) computation of $\losspartial{\vz_t} := \vgamma \odot \losspartial{\vy_t}$,  and (b) computation of $\frac{1}{\epsilon} ( \normalize(\vx_t + \epsilon  \losspartial{ \vz_t } ) - \normalize(\vx_t) )$. We represent each sub-operation as a \simulator{} module.

To compute $\losspartial{\vz_t} := \vgamma \odot \losspartial{\vy_t}  = W_{\vgamma} \losspartial{\vy_t}$, we can observe that the required operation is identical to backpropagating through a linear layer with parameters $\mW_{\vgamma}$ and $\vb$. Hence, we simply call the Linear Backpropagation module on the current embeddings. We use residual connections to retain $\vx_t$ at each location $t$, and the contents of the prefix tokens $\{ \vv_j \}$.

Now, the embedding $\ve_t$ contains $\losspartial{\vz_t}$ and  $\vx_t$. In order to backpropagate through $\normalize$, we first use a linear layer to compute $\vx_t + \epsilon  \losspartial{ \vz_t }$ and retain $\vx_t$. Following the same procedure as the Forward module, we use a Group normalization layer with weight and bias parameters $\mathbf{1}$ and $\mathbf{0}$ respectively, to compute $\normalize(\vx_t + \epsilon  \losspartial{ \vz_t } )$ and $\normalize(\vx_t)$. Finally, we use a linear layer to compute $ \frac{1}{\epsilon} ( \normalize(\vx_t + \epsilon  \losspartial{ \vz_t } ) - \normalize(\vx_t) ).$

%Given $\losspartial{\vy_t}$ and $\vx_t$, we can show that for any $\epsilon>0$
%\begin{align*}
%    \losspartial {\vx_t} =  \frac{1}{\epsilon} (\normalize( \vx_t + \epsilon \vz_t ) - \normalize( \vx_t ) ) + \mathcal{O}(\epsilon),
%\end{align*}
%where $\vz_t = \vgamma \odot \losspartial{\vy_t}$.

%This can be computed by first applying the backpropagation module of Linear layer to compute $\vz_t$, followed by a group norm and linear layer.

%We can encode the above computation, with a single MLP layer.

\paragraph{Auxiliary's Descent update} 

And finally, the auxiliary's descent operation updates parameters $\vgamma, \vb$ using a batch of inputs $\{\vx_t\}_{t \le T}$ and the loss gradient w.r.t. the corresponding outputs $\{ \losspartial{\vy_t} \}_{t \le T}$.

\begin{definition}[Auxiliary's layer normalization descent]\label{def:layernorm_descent}
    For parameters $\vgamma, \vb \in \RR^{\aux{D}}$, descent update takes in a batch of inputs $\{ \vx_t \in \RR^{\aux{D}} \}_{t \le {\aux{T}}}$ and gradients $\{ \losspartial{\vy_t} \in \RR^{\aux{D}} \}_{t \le {\aux{T}}}$ and updates the parameters as follows:
    \begin{align*}
        \vgamma \gets \vgamma - \eta \sum_{t \le {\aux{T}}} \losspartial{\vy_t} \odot \vz_t; \quad \quad
        \vb \gets \vb - \eta \sum_{t \le {\aux{T}}} \losspartial{\vy_t},
    \end{align*}
    where $\vz_t$ represents $\normalize(\vx_t)$.
\end{definition}

The update of $\vgamma$ involves an elementwise multiplication between $\losspartial{\vy_t}$ and $\vz_t$, which requires an MLP layer (\cref{thm:gelu_multiplication}). %If we decide to use an attention module, similar to the Linear Descent module, we require $\aux{D}$ attention heads 
With the prefix tokens containing the rows of $\mW_{\vgamma}$ and $\vb$, we instead consider the update of $\vb$ alone with the descent update. 
%We can decide to \ap{Write about Wgamma}

%However, instead of $\gamma$, we have $\mW_{\gamma}$ in the pre
\paragraph{\simulator{} Layer normalization descent module}
The input embeddings contain $\losspartial{\vy_t}$ in the first $\aux{D}$ coordinates. The prefix tokens contain $\mW_{\vgamma}, \vb$, which have been copied from the Forward module using residual connections. The update of $\vb$ is identical to the auxiliary's descent update through a linear layer. Hence, we apply a \simulator{} Linear descent module to the current embeddings, updating only the bias $\vb$ and switching off the update to $\mW_{\gamma}.$

%, and  $\vz_t$ at each position $t$. Furthermore, residual connections have been used to copy the contents of the prefix tokens $\{ \vv_j \}$ from the Layer normalization Forward module, which contains $\mW_{\vgamma}, \vb$. 

\subsection{Additional definitions}
We describe \simulator{} group normalization layer below, which we use in different modules to simulate the auxiliary's layer normalization operations.

\begin{definition}[\simulator{} $\aux{D}$-Group normalization]\label{def:group_norm}
    Define a normalization function $f:\RR^d\to\RR^d$ that performs $f(\vx) = (\vx-\mu) / \sigma$, where $\mu$ and $\sigma$ are the mean and standard deviation of $\vx$, respectively. Then, $\aux{D}$-Group RMSnorm with parameters $\simuw{\vgamma}, \simuw{\vb} \in \RR^{\aux{D}}$ takes as input  $\vx \in \RR^{\simu{D}}$ and outputs $\vy = \attmerge( \{ \vy^h \in \RR^{\aux{D}} \}_{h \le \lfloor \simu{D}/\aux{D} \rfloor} )$, with
    \begin{align*}
        \vy^h = \simuw{\vgamma} \odot f(\vx^h) + \simuw{\vb},
    \end{align*}
    where $\vx^h = \attsplit_{\lfloor \simu{D}/\aux{D} \rfloor} ( \vx )_h.$
\end{definition}

\subsection{Proof of theorems and gradient definitions}
We restate the theorems and definitions, before presenting their proofs for easy referencing.

\layernormbackprop*

\begin{proof}[Derivation of gradient in \cref{def:layernorm_backprop} ]
With the normalization function $f$ and parameters $\vx, \vb \in \RR^{ \aux{D}}$, recall from \cref{def:layernorm} that given an input $\vx \in \RR^{ \aux{D} }$, a layer normalization layer returns $\vy = \gamma \odot \vz + \vb; \vz = f(\vx).$  Let $\mu$ and $\sigma$ denote the mean and standard deviation of $\vx$. They can be computed as 
\begin{align*}
    \mu = \frac{1}{ \aux{D} } \sum_{i=1}^{\aux{D}} x_i, \quad \sigma = \sqrt{ \frac{ 1 }{ \aux{D} } \sum_{i=1}^{\aux{D}} (x_i - \mu)^2  }.
\end{align*}

With the chain rule, we can compute $\losspartial{\vx}$ from $\losspartial{\vy}$ as follows.
\begin{align}
    \losspartial{\vx} = (\frac{\partial \vz}{ \partial \vx})^{\top} \losspartial{ \vz }; \quad \text{ with } 
    \losspartial{\vz} = (\frac{\partial \vy}{ \partial \vz})^{\top}  \losspartial{ \vy }. \label{eq:chainrule_ln}
\end{align}

Since $\vy = \gamma \odot \vz + \vb$, we have $\frac{\partial \vy}{ \partial \vz} = \mW_{\vgamma}$, where $\mW_{\vgamma}$ represents a diagonal matrix with $\vgamma$ on the main diagonal. Thus, $\losspartial{\vz} = \mW_{\vgamma} \losspartial{\vy} = \vgamma \odot \losspartial{\vy}.$

With $\vz = f(\vx) = \frac{\vx - \mu}{\sigma}$, we have
\begin{align}
    \frac{\partial \vz}{ \partial \vx} &= \frac{\partial }{ \partial \vx} \left( \frac{\vx - \mu}{\sigma} \right) = \frac{1}{\sigma} \frac{\partial \vx}{ \partial \vx } -  \frac{1}{\sigma} \frac{\partial \mu}{ \partial \vx } - \frac{(\vx - \mu)}{\sigma^2} \left(\frac{\partial \sigma}{ \partial \vx } \right)^{\top} \nonumber\\&
    = \frac{1}{\sigma} \left( \mI - \frac{ 1 }{ \aux{D} } \mathbf{1} \mathbf{1}^{\top} - \vz \vz^{\top} \right). \label{eq:dz_dx_ln}
\end{align}

In the final step, we require $\frac{\partial \mu}{ \partial \vx }$ and $\frac{\partial \sigma}{ \partial \vx }$, which are computed as follows.
\begin{itemize}
    \item $\frac{\partial \mu}{ \partial \vx } \in \RR^{\aux{D}}$ with its $j$th element given by 
    \begin{align*}
        \left(\frac{\partial \mu}{ \partial \vx }\right)_j = \frac{\partial \mu}{ \partial x_j } = \frac{\partial   }{ \partial x_j } (\frac{1}{ \aux{D} } \sum_{i=1}^{\aux{D}} x_i) =  \frac{1}{ \aux{D} }.
    \end{align*}

    \item $\frac{\partial \sigma}{ \partial \vx } \in \RR^{\aux{D}}$ with its $j$th element given by 
    \begin{align*}
        \left(\frac{\partial \sigma}{ \partial \vx }\right)_j &= \frac{\partial \sigma}{ \partial x_j } = \frac{\partial   }{ \partial x_j } \left( \sqrt{ \frac{1}{ \aux{D} } \sum_{i=1}^{\aux{D}} (x_i - \mu)^2 }\right) \\&=  \frac{1}{ \sqrt{\sum_{i=1}^{\aux{D}} (x_i - \mu)^2 } } \sum_{i=1}^{\aux{D}} (x_i - \mu) \frac{\partial (x_i - \mu)}{\partial x_j} \\& = \frac{1}{ \sqrt{\sum_{i=1}^{\aux{D}} (x_i - \mu)^2 } } \left(  (x_j - \mu) - \frac{1}{ \aux{D} } \sum_{i=1}^{ \aux{D} } (x_i - \mu) \right) 
        = \frac{ x_j - \mu }{ \sigma } := z_j,
    \end{align*}
    where we have re-utilized the $\frac{\partial \mu}{ \partial \vx }$ in the pre-final step.
\end{itemize}

Hence, from \cref{eq:chainrule_ln}, 
\begin{align*}
\losspartial{\vx} = (\frac{\partial \vz}{ \partial \vx})^{\top} \losspartial{ \vz } 
    =   \frac{1}{\sigma} \left( \mI - \frac{ 1 }{ \aux{D} } \mathbf{1} \mathbf{1}^{\top} - \vz \vz^{\top} \right) \losspartial{ \vz } = \frac{1}{\sigma} \left( \losspartial{\vz} - \frac{ 1 }{ \aux{D} }  \langle \mathbf{1}, \losspartial{\vz}  \rangle  \mathbf{1} -  \langle \vz, \losspartial{\vz}  \rangle \vz \right).
\end{align*}    

\end{proof}

We repeat \cref{thm:first_order_LN_backprop} for easier reference. 
\firstorderLNbackprop*

\begin{proof}[ Proof of \cref{thm:first_order_LN_backprop} ]
With the normalization function $f$ and parameters $\vx, \vb \in \RR^{ \aux{D}}$, recall from \cref{def:layernorm} that given an input $\vx \in \RR^{ \aux{D} }$, a layer normalization layer returns $\vy = \gamma \odot \vz + \vb; \vz = f(\vx).$  Let $\mu$ and $\sigma$ denote the mean and standard deviation of $\vx$. They can be computed as 
\begin{align*}
    \mu = \frac{1}{ \aux{D} } \sum_{i=1}^{\aux{D}} x_i, \quad \sigma = \sqrt{ \frac{ 1 }{ \aux{D} } \sum_{i=1}^{\aux{D}} (x_i - \mu)^2  }.
\end{align*}
We will refer to $\frac{\partial \vz}{\partial \vx}$ from \cref{eq:dz_dx_ln} and the formulation of $\losspartial{\vx}$ from \cref{eq:chainrule_ln} for our current proof. To recall, they are
\begin{align*}
    \frac{\partial \vz}{ \partial \vx} =
     \frac{1}{\sigma} \left( \mI - \frac{ 1 }{ \aux{D} } \mathbf{1} \mathbf{1}^{\top} - \vz \vz^{\top} \right), \qquad \losspartial{\vx} = (\frac{\partial \vz}{ \partial \vx})^{\top} \losspartial{ \vz }.
\end{align*}

Using a second-order Taylor expansion of the normalization function $f$ around $\vx$, we have
\begin{align*}
    f( \vx + \epsilon \losspartial{ \vz } ) &= f(\vx) + \epsilon  \frac{ \partial  f(\vx) }{ \partial \vx }\losspartial{ \vz }  +  \int_{0}^{\epsilon}   \losspartial{ \vz }^{\top} \frac{\partial }{ \partial \vx_{\theta} } \left( \frac{\partial  f(\vx_{\theta})}{ \partial \vx_{\theta} } \right) \losspartial{ \vz } \theta d\theta \\&
    = f(\vx) + \epsilon  \frac{ \partial  f(\vx) }{ \partial \vx }\losspartial{ \vz }  -    \int_{0}^{\epsilon}   \frac{1}{\sigma^2_{\theta}} \left( \norm{ \losspartial{ \vz } }_2^2 - \frac{1}{\aux{D}} \sum_{i=1}^{ \aux{D} } (\langle \mathbf{1}, \losspartial{\vz}  \rangle)^2 - ( \langle \vz_{\theta}, \losspartial{\vz}  \rangle )^2 \vz_{\theta}    \right) \theta d\theta,
\end{align*}
where $\vx_{\theta}$ represents $\vx + \theta \losspartial{\vz}, \vz_{\theta} = f(\vx_{\theta})$. The second step follows similar steps for computing $  \frac{\partial \vz}{ \partial \vx} $ in \cref{eq:dz_dx_ln}. We avoid this computation since we only need to make sure that the second-order term is  bounded. Furthermore, if $\epsilon \le \mathcal{O}\left( \frac{\sigma}{ \sqrt{\aux{D}} \norm{ \losspartial{\vz} }_2 } \right), $ we can show the $\ell_2$-norm of the second-order term can be bounded by $\mathcal{O} ( \epsilon^2 \aux{D}^{3/2} \sigma^{-2} \norm{ \losspartial{\vz} }^2_2  ).$ We avoid this computation as well. \ap{Fill in the details for the final version}

Thus, from the above formulation, we have
\begin{align*}
    \lim_{\epsilon \to 0} \frac{ f( \vx + \epsilon \losspartial{ \vz } ) - f(\vx) }{ \epsilon } = \frac{ \partial  f(\vx) }{ \partial \vx }\losspartial{ \vz } = \left( \frac{ \partial  f(\vx) }{ \partial \vx } \right)^{\top} \losspartial{ \vz } = \losspartial{\vx}.
\end{align*}

The pre-final step follows from \cref{eq:dz_dx_ln}, where $\frac{ \partial  f(\vx) }{ \partial \vx } = \frac{\partial \vz}{ \partial \vx} = \frac{1}{\sigma} \left( \mI - \frac{ 1 }{ \aux{D} } \mathbf{1} \mathbf{1}^{\top} - \vz \vz^{\top} \right)$ can be shown to be symmetric. The final step follows from the gradient formulation in \cref{eq:chainrule_ln}. Including the error term, we have the final bound as
\begin{align*}
    \norm{ \frac{ f( \vx + \epsilon \losspartial{ \vz } ) - f(\vx) }{ \epsilon } - \losspartial{\vx} }_2 \le \mathcal{O} ( \epsilon \aux{D}^{3/2} \sigma^{-2} \norm{ \losspartial{\vz} }^2_2 ).
\end{align*}
Using $\losspartial{\vz} = \gamma \odot \losspartial{\vy}$ and a Cauchy-Schwartz inequality gives the final bound.
%If $\sigma$ represents the standard deviation of $\vx$, 
\end{proof}

%We simply use the parameter update module from the linear layer to update $W_{\vgamma}, \vb$ stored in the $\{ \vv_j \}$.

%% file: Theorems/LN_backprop.tex
%\begin{theorem}\label{thm:first_order_LN_backprop}
The following theorem shows that the first-order gradient is a good approximation of the true gradient, and in the limit of $\epsilon$ tending to $0$, the approximation error tends to $0$ as well.

\begin{restatable}{theorem}{firstorderLNbackprop}\label{thm:first_order_LN_backprop}
%\ap{make it more formal} Under the assumption that $\vx$, $\losspartial{\vy}$, $\vgamma, \vb$ are bounded, $\norm{ \losspartial{ \vx } - \Hat{ \losspartial{ \vx } } }_{\infty} \le \mathcal{O}(\epsilon)$.
For any $\epsilon > 0$, and a layer normalization layer with parameters $\vgamma, \vb\in \RR^{ \aux{D} }$, for an input $\vx \in \RR^{ \aux{D} }$ and gradient $\losspartial{\vy} \in \RR^{ \aux{D} }$, 
\begin{align*}
    \norm{ \Hat{\losspartial{\vx}} - \losspartial{\vx} }_2 \le \mathcal{O} ( \epsilon \aux{D}^{3/2} \sigma^{-2} \norm{\vgamma}_2^2 \norm{  \losspartial{\vy} }^2_2 ),
\end{align*}
where $\sigma$ denotes the standard deviation of $\vx$. $\losspartial{\vx}, \Hat{\losspartial{ \vx }}$ have been computed from $\vx$, $\losspartial{\vy}$ and $\epsilon$ using  \cref{def:layernorm_backprop,def:layernorm_appr_backprop}.
\end{restatable}
%\end{theorem}

%% file: Activation_simulation.tex
\section{Activation layer}\label{sec:act_appendix}

\begin{definition}[Auxiliary activation]\label{def:act}
    For a continuous function $\act: \RR \to \RR$,  an activation layer takes $\vx\in\RR^{\aux{D}}$ as input and outputs $\vy = \act(\vx)$ with $ y_i = \act (x_i) $ for all $i \le \aux{D}$.
\end{definition}

  In the discussions below, we consider an activation layer in the auxiliary model with activation function $\act$ that takes in input sequence $\vx_1, \cdots, \vx_{\aux{T}}$ and outputs $\vy_1, \cdots, \vy_{\aux{T}}$, with $\vy_t =\act(\vx_t)$ for each $t \le {\aux{T}}$. Since this involves a token-wise operation, we will present our constructed modules with a general token position $t$. Since no parameters of the auxiliary model are involved in this operation, the prefix tokens $\{ \vv_j \}$ contain $0$ in the following modules.

\paragraph{\simulator{} Activation Forward module} The embedding $\ve_t$ contains $\vx_t$ in its first $\aux{D}$ indices. We simply pass the embeddings into activation $\act$, which returns $\act(\vx_t)$ in its first $\aux{D}$ indices.

\paragraph{Auxiliary's backpropagation through activation} With the definition in \cref{def:act}, the auxiliary's backpropagation takes in the loss gradient w.r.t. output ($\losspartial{\vy}$) and computes the loss gradient  w.r.t. input ($\losspartial{\vx}$). We further assume that the derivative of $\act$ is well-defined everywhere. This assumption includes non-differentiable activation functions with well-defined derivatives like $ReLU$.

\begin{definition}[Auxiliary activation backpropagation]\label{def:act_backprop}
   For a continuous function $\act: \RR \to \RR$, with a well-defined derivative $\act'(x) = \partial \act(x) / \partial x$ for each $x \in \RR$,  the backpropagation takes $\losspartial{ \vy }, \vx \in\RR^{\aux{D}}$ as input and outputs
    \begin{align*}
        \losspartial{ \vx } = \act'(\vx) \odot \losspartial{\vy},
    \end{align*}
    where $\act'(\vx) \in \RR^{\aux{D}}$ with $\act'(\vx)_i = \act'(x_i)$ at each $i \le \aux{D}.$  
\end{definition}

\paragraph{Complexity of true backpropagation} The above operation is computation heavy since 
 it involves $\act'(\vx) \odot \losspartial{\vy}$. As mentioned for the layer normalization module, the element-wise multiplication between $\act'(\vx)$ and $\losspartial{\vy}$ will require an MLP module following \cref{thm:gelu_multiplication}. Furthermore, it involves changing the activation function in \simulator{} in specific modules to $\act'$. To circumvent this, we instead turn to a first-order Taylor approximation.

\begin{definition}[Approximate Activation backpropagation]\label{def:approx_act_backprop}
   For a continuous function $\act: \RR \to \RR$ and a hyperparameter $\epsilon$, the layer takes $\losspartial{ \vy }, \vx \in\RR^{\aux{D}}$ as input and outputs
    \begin{align*}
        \Hat{ \losspartial{ \vx } } = \frac{1}{\epsilon} \left( \act(\vx + \epsilon \losspartial{\vy}) - \act(\vx) \right).
    \end{align*}
\end{definition}

\input{Theorems/Activation_backprop}

\paragraph{\simulator{} Activation backpropagation module}
The input embeddings contain $\losspartial{\vy_t}$ in the first $\aux{D}$ embeddings. 
With the requirement of the activation layer input for gradient, we copy $\vx_t$ from the Forward module at each position $t$. We set $\epsilon$ as a hyper-parameter and return $\Hat{ \losspartial{ \vx_t } }$ as the output of this module.

$\Hat{ \losspartial{ \vx_t } }$ will be computed using a single-layer MLP with activation $\act$ as follows. The first linear layer of the MLP will be used to compute $\vx_t + \epsilon \losspartial{\vy_t}$ and $\vx_t$. After the activation $\act$, the embedding $\ve_t$ contains $\act(\vx_t + \epsilon \losspartial{\vy_t})$ and $\act(\vx_t)$. The final linear layer of the MLP will be used to compute $\frac{1}{\epsilon} \left( \act(\vx_t + \epsilon \losspartial{\vy_t}) - \act(\vx_t) \right)$.

\input{Proofs/Proof_activation_backprop}

%% file: Theorems/Activation_backprop.tex
The following theorems show that under mild assumptions on the activation function and the input, gradient pair, the first-order gradient is a good approximation to the true gradient.

\begin{restatable}{theorem}{firstorderactbckprop}\label{thm:first_order_act_backprop}    
%\end{restatable}
%\begin
For any $\epsilon > 0$, $B_y, B_{act} > 0$, consider a second-order differentiable activation function $\act: \RR \to \RR$, with $\partial^2 \act (x) / \partial (x^2) $ bounded by $B_{act}$ for each $x \in \RR$. Then, for any  input $\vx \in \RR^{\aux{D}}$ and gradient $\losspartial{\vy} \in \RR^{\aux{D}}$ with $\norm{\losspartial{\vy}}_{2} \le B_y$, the following holds true:
\begin{align*}
    \norm{ \losspartial{ \vx } - \Hat{ \losspartial{ \vx } } }_{2} \le \mathcal{O}( B_{act} B_y^2 \epsilon  ),
\end{align*}
where $\losspartial{ \vx }, \Hat{ \losspartial{ \vx } }$ have been defined using $\vx, \losspartial{\vy}$, and $\epsilon$ in \cref{def:act_backprop,def:approx_act_backprop}.
%\end{theorem}
\end{restatable}

For $\mathrm{ReLU}$ activation, which is not second-order differentiable at $0$, we instead bound the difference between $\losspartial{ \vx }, \Hat{ \losspartial{ \vx } }$ by defining some form of alignment between input and  gradient pair $\vx, \losspartial{\vy}$.

\begin{definition}[$(\epsilon, \rho)$-alignment]\label{def:eps-rho_aligned}
Input and gradient $\vx, \losspartial{\vy} \in \RR^{\aux{D}}$ are said to be $(\epsilon, \rho)$-aligned, if there exist a set $C \subseteq [\aux{D}]$, with  $\abs{C} \ge (1 - \rho) \aux{D}$, such that for each $i$ in $C$, $\abs{x_i} > \epsilon \abs{ (\losspartial{\vy} )_i }.$
\end{definition}

$\epsilon$ controls the fraction of coordinates where $\abs{x_i} \leq \epsilon \abs{ (\losspartial{\vy} )_i }$. As $\epsilon \to 0$, $\rho \to 0$  as well for bounded gradients. 

%\begin{example}
\begin{restatable}{example}{examplealignment}
    For any $B_{min}, B_{max} > 0$, all inputs $\vx$ that satisfy $\min_i \abs{x_i} > B_{min}$ , and gradients $\losspartial{\vy}$ that satisfy $\max_j \abs{(\losspartial{\vy})_j} \le B_{max}$, are $( B_{min} / B_{max}, 0)$-aligned.
\end{restatable}
    
%\end{example}

%\begin{theorem}\label{thm:first_order_act_backprop_relu}
\begin{restatable}{theorem}{firstorderactbckproprelu}\label{thm:first_order_act_backprop_relu}
For any $\epsilon, \rho > 0$ and $B_y > 0$, for any input $\vx \in \RR^{\aux{D}}$ and gradient $\losspartial{\vy} \in \RR^{\aux{D}}$, with $\norm{\losspartial{\vy}}_{\infty} \le B_y$,  that are $(\epsilon, \rho)$-aligned by \cref{def:eps-rho_aligned}, 
\begin{align*}
    \norm{ \losspartial{ \vx } - \Hat{ \losspartial{ \vx } } }_{2} \le \mathcal{O}( B_y \sqrt{\rho \aux{D}} ).
\end{align*}
where $\losspartial{ \vx }, \Hat{ \losspartial{ \vx } }$ have been defined using $\vx, \losspartial{\vy}$, $\epsilon$ and $\act = \mathrm{ReLU}$ in \cref{def:act_backprop,def:approx_act_backprop}.
%\end{theorem}
\end{restatable}

%For example, for all inputs $\vx$ that additionally satisfy $\min_i \abs{x_i} \ge B_{x}$ for some $B_x > 0$, we can set $\epsilon \leq \frac{B_y}{B_x}$

%% file: Proofs/Proof_activation_backprop.tex
\subsection{Proofs of theorems}
We restate the theorems, before presenting their proofs for easy referencing.

\firstorderactbckprop*
\begin{proof}
    The proof follows along the lines of \cref{thm:first_order_LN_backprop}. Recall that given an input $\vx$, the activation layer outputs $\vy = \act(\vx)$, where the function $\act$ is applied coordinate-wise on $\vx$. Given input $\vx$ and the output gradient $\losspartial{\vy}$, the gradient w.r.t. the input is given by $\losspartial{\vx} = \act'(\vx) \odot \losspartial{\vy}$, where the $\act'$ function is also applied coordinate wise to $\vx$. We defined $\Hat{\losspartial{\vx}}$ as an $\epsilon$-approximate gradient, given by $\frac{1}{\epsilon} ( \act(\vx + \epsilon \losspartial{\vy}) - \act(\vx ) ).$ Since both $\act$ and $\act'$ are applied coordinate-wise, we can look at the coordinate-wise difference between $\losspartial{\vx}$ and $\Hat{\losspartial{\vx}}$.

    Consider an arbitrary coordinate $i \le \aux{D}$. Under the assumption that $\act$ is second-order differentiable, we have
    \begin{align*}
       (\Hat{\losspartial{\vx}})_i &= \frac{1}{\epsilon} \left( \act( x_i + \epsilon (\losspartial{\vy})_i ) - \act( x_i) \right) \\&
        = \act'(x_i)  (\losspartial{\vy})_i + \frac{1}{\epsilon} \int_{\theta=0}^{\epsilon} \frac{\partial^2 \act( x_{\theta} ) }{ \partial x^2_{\theta} } (\losspartial{\vy})^2_i  \theta d\theta \\&
        = \act'(x_i)  (\losspartial{\vy})_i + \mathcal{O}( \epsilon B_{act} (\losspartial{\vy})_i^2 ),
    \end{align*}
    where $x_{\theta}$ represents $ x_i + \theta (\losspartial{\vy})_i$ in the second step. In the final step, we utilize the upper bound assumption on $\frac{\partial^2 \act(x)}{ \partial x^2 }.$

    Thus, $(\losspartial{\vx})_i - (\Hat{\losspartial{\vx}})_i = \mathcal{O}( \epsilon B_{act} (\losspartial{\vy})_i^2 )$, and so
    \begin{align*}
        \norm{ \losspartial{\vx} - \Hat{\losspartial{\vx}} }_2 = \mathcal{O} ( \epsilon B_{act} \sum_{i=1}^{\aux{D}}(\losspartial{\vy})_i^2 ) = \mathcal{O} ( \epsilon B_{act} \norm{ \losspartial{\vy} }_2^2 ) \le \mathcal{O} ( \epsilon B_{act}  B_y^2 ).
    \end{align*}
\end{proof}

\examplealignment*
\begin{proof}
    Recall the definition of $(\epsilon, \rho)$-alignment from \cref{def:eps-rho_aligned}. Input and gradient $\vx, \losspartial{\vy} \in \RR^{\aux{D}}$ are said to be $(\epsilon, \rho)$-aligned, if there exist a set $C \subseteq [\aux{D}]$, with  $\abs{C} \ge (1 - \rho) \aux{D}$, such that for each $i$ in $C$, $\abs{x_i} > \epsilon \abs{ (\losspartial{\vy} )_i }.$ 

    Consider an arbitrary coordinate $i \le \aux{D}$. We have $\abs{x_i} > \epsilon  \abs{(\losspartial{\vy})_i}$ for any $\epsilon < \abs{x_i} / \abs{(\losspartial{\vy})_i}$. Under the assumption that $\abs{x_i} > B_{min}$, and $\abs{(\losspartial{\vy})_i} \le B_{max}$, a bound of $B_{min} / B_{max}$ suffices. 
\end{proof}

\firstorderactbckproprelu*
\begin{proof}
    Recall that given an input $\vx$, the activation layer outputs $\vy = \act(\vx)$, where the function $\act$ is applied coordinate-wise on $\vx$. Given input $\vx$ and the output gradient $\losspartial{\vy}$, the gradient w.r.t. the input is given by $\losspartial{\vx} = \act'(\vx) \odot \losspartial{\vy}$, where the $\act'$ function is also applied coordinate wise to $\vx$. We defined $\Hat{\losspartial{\vx}}$ as an $\epsilon$-approximate gradient, given by $\frac{1}{\epsilon} ( \act(\vx + \epsilon \losspartial{\vy}) - \act(\vx ) ).$ Since both $\act$ and $\act'$ are applied coordinate-wise, we can look at the coordinate-wise difference between $\losspartial{\vx}$ and $\Hat{\losspartial{\vx}}$. For $\mathrm{ReLU}$ activation, $\act'(x) = \mathrm{sign}(x)$ for all $x \in \RR \setminus \{0\}$, with $\act'(0) = 1$ to avoid ambiguity. 
    
    Going by the definition of $(\epsilon, \rho)$-alignment of the input and gradient from \cref{def:eps-rho_aligned}, we have a set $C$ with $\abs{C} \ge (1-\rho) \aux{D}$ such that for each $i \in \aux{D}$, $\abs{x_i} > \epsilon  \abs{(\losspartial{\vy})_i}$. For all coordinates $i \in C$, we can then observe that $\mathrm{sign} (x_i + \epsilon (\losspartial{\vy})_i) = \mathrm{sign} (x_i)$, implying 
    $$
    \act( x_i + \epsilon (\losspartial{\vy})_i ) - \act( x_i ) = \epsilon (\losspartial{\vy})_i \act'(x_i) = \epsilon (\losspartial{\vx})_i
    $$

    For coordinates $i \notin C$, we have three possible cases:
    \begin{itemize}
        \item $\mathrm{sign} (x_i) = \mathrm{sign} (x_i + \epsilon (\losspartial{\vy})_i ) $: In this case, we can again show $\act( x_i + \epsilon (\losspartial{\vy})_i ) - \act( x_i ) = \epsilon (\losspartial{\vy})_i \act'(x_i) = \epsilon (\losspartial{\vx})_i.$

        \item $\mathrm{sign} (x_i) = 0$, $\mathrm{sign} (x_i + \epsilon (\losspartial{\vy})_i ) = 1$: In this case, we have $\act'(x_i) = 0$, and so $(\losspartial{\vx})_i = 0$. Additionally, $\mathrm{sign} ( (\losspartial{\vy})_i ) = 1$, and so
        \begin{align*}
        \abs{ \act( x_i + \epsilon (\losspartial{\vy})_i ) - \act( x_i ) - \epsilon (\losspartial{\vx})_i } &= \abs{ x_i + \epsilon (\losspartial{\vy})_i} 
        \leq \epsilon \abs{  (\losspartial{\vy})_i },
        \end{align*}
        where in the final step, we use the fact that $x_i < 0$ and $\abs{x_i} < \epsilon \abs{ (\losspartial{\vy})_i }.$ 
        
        \item  $\mathrm{sign} (x_i) = 1$, $\mathrm{sign} (x_i + \epsilon (\losspartial{\vy})_i ) = 0$: In this case, we have $\act'(x_i) = 1$, and so $(\losspartial{\vx})_i = (\losspartial{\vy})_i$.  Additionally, $\mathrm{sign} ( (\losspartial{\vy})_i ) = 0$, and so
        \begin{align*}
        \abs{ \act( x_i + \epsilon (\losspartial{\vy})_i ) - \act( x_i ) - \epsilon (\losspartial{\vx})_i } &= \abs{ - x_i  - \epsilon (\losspartial{\vy})_i } 
        \leq \abs{ \epsilon (\losspartial{\vy})_i },
        \end{align*}
        where in the final step, we use the fact that $x_i \ge 0$ and $\abs{x_i} < \epsilon \abs{ (\losspartial{\vy})_i }.$ 
        
    \end{itemize}

    Thus, from the above discussion, we have
    \begin{align*}
        \norm{ \losspartial{\vx} - \Hat{\losspartial{\vx}} }_2 &= \frac{1}{\epsilon} \left( \sum_{i=1}^{\aux{D}} (\act( x_i + \epsilon (\losspartial{\vy})_i ) - \act( x_i ) - \epsilon (\losspartial{\vx})_i)^2 \right)^{1/2} \\&
        = \frac{1}{\epsilon} \left(  \sum_{i \notin C} (\act( x_i + \epsilon (\losspartial{\vy})_i ) - \act( x_i ) - \epsilon (\losspartial{\vx})_i)^2 \right)^{1/2} \\&
        \leq \left(  \sum_{i \notin C}  (\losspartial{\vy})_i^2 \right)^{1/2} \leq \sqrt{\rho \aux{D}} \sqrt{\max_{i \notin C}  (\losspartial{\vy})^2_i } \leq \sqrt{\rho \aux{D}} B_y.
    \end{align*}
    The final step includes a simple Cauchy Schwartz inequality and the desired bound comes from the assumed bound on $\norm{ \losspartial{\vy} }_2 $.
\end{proof}

%% file: language_model_head.tex
\section{Language model head} \label{sec:lm_head}
Additionally, we provide a description of the gradient computation for the loss function that involves the language model head. This computation entails performing a $\mathrm{softmax}$ operation over the entire vocabulary. If $\cV$ denotes the vocabulary set of the auxiliary model, and $\mE \in \RR^{\abs{\cV} \times \aux{D}}$ denotes the embedding matrix of the auxiliary model, we directly utilize the  embedding matrix for the auto-regressive loss in the \simulator. Additionally, we do not update the embedding matrix of the auxiliary model; instead, we solely backpropagate the gradients through the language model head. Recent work in~\citep{kumar2022finetune} has shown that keeping the embedding matrix fixed while updating the model can stabilize SGD. 
We demonstrate that the backpropagated gradients can be expressed as the combination of the language model head and a self-attention layer.

\begin{definition}[KL-loss gradient through auxiliary's language model head] \label{def:auxiliary_lm_back}
Given an embedding matrix $\mE \in \RR^{\abs{V} \times \aux{D}}$, the language model head takes in input $\vx \in \RR^{\aux{D}}$ and a target distribution $\vq \in \RR^{\abs{V}}$ and returns gradient $\losspartial{\vx} \in \RR^{\aux{D}}$, with $\losspartial{\vx} = \mE^{\top} \left( \mathrm{softmax} ( \mE \vx ) - \vq \right).$ 
\end{definition}

In the autoregressive loss on a sequence of tokens, the target output distribution at any position is the next occurring token. If  $\{\vx^{un}_t\}_{t=1}^{\aux{T}}$ denote the uncontextualized embeddings of a sequence of tokens after encoding them via the embedding matrix, and 
$\{\vx_t\}_{t=1}^{\aux{T}}$ denote their contextualized embeddings after passing through the auxiliary model, then the gradient $\losspartial{\vx_t}$ at any position $t$ can be simplified as $\mE^{\top}  \mathrm{softmax} ( \mE \vx_t ) - \vx^{un}_{t+1}.$ We illustrate the involved \simulator{} module w.r.t. an arbitrary position $t$.

\paragraph{\simulator{} autoregressive loss gradient module} 
The current embedding $\ve_t$ contains the contextualized embedding $\vx_t$ in its first $\aux{D}$ coordinates. Furthermore, $\ve_t$ includes the uncontextualized embedding $\vx^{un}_t$, copied from the input layer using residual connections. The prefix tokens ${ \vv_j }$ are assigned a value of $0$ and do not participate in the subsequent computations.

The loss computation can be decomposed into two sub-operations: (a) computing $\vy_t := \mE^{\top} \mathrm{softmax} ( \mE \vx_t )$, and (b) calculating $\losspartial{\vx_t} = \vy_t - \vx^{un}_{t+1}$.

For the first sub-operation, we use a feed-forward layer with $\mathrm{softmax}$ activation, with hidden and output weights $\mE$ and $\mE^{\top}$ respectively, that takes in the first $\aux{D}$ of $\ve_t$ and returns $\vy_t$ in the first $\aux{D}$ coordinates. We retain $\vx_t^{un}$ using a residual connection.

The final sub-operation can be interpreted as a \simulator{} self-attention layer. With $\ve_t$ containing both $\vy_t$ and $\vx_t^{un}$, we use a linear self-attention layer (\cref{def:self-attn_construct}) with two attention heads. The first attention head assigns an attention score of $1$ to pairs $\{(t, t+1)\}_{t \le \aux{T} - 1}$, while assigning an attention score of $0$ to the remaining pairs. At any position $t$, $-\vx^{un}_{t}$ is considered the value vector. The second attention head assigns an attention score of $1$ to pairs $\{(t, t)\}_{t \le \aux{T}}$, while assigning an attention score of $0$ to the remaining pairs. At any position $t$, $\vy_{t}$ is considered the value vector. The outputs of both attention heads are subsequently combined using a linear layer.

\begin{remark}
We conducted experiments using mean-squared loss and Quad loss \cite{saunshi2020mathematical}, which do not necessitate softmax computations for gradient computation. As an example, in the case of mean-squared loss, if our objective is to minimize $\frac{1}{2} \sum_{t=1}^{T} \norm{ \vx_t - \vx^{un}_{t+1} }^2$, the gradient can be computed as $\losspartial{\vx_t} = \vx_t - \vx^{un}_{t+1}$. Similarly, in the case of Quad loss, the gradient is $\losspartial{\vx_t} = \frac{1}{\abs{V}} \sum_{i} \ve_i - \vx^{un}_{t+1}$. However, in all of our language model experiments (\cref{sec:experiments}), both gradients resulted in minimal improvement in perplexity compared to the auxiliary model. Therefore, we continue utilizing the standard KL loss for optimization.
\end{remark}

\begin{remark}
For ease of implementation in the codebase, we utilize a dedicated loss module that takes in $\vy_t, \vx^{un}_{t+1}$ as input and directly computes $\losspartial{\vx_t} = \vy_t - \vx^{un}_{t+1}$.
\end{remark}

%% file: Parameter_sharing_app.tex
\section{Parameter sharing}\label{sec:parameter_sharing_app}

\paragraph{Feed-forward layer of auxiliary model:} In a standard auxiliary transformer, like GPT-2, the feed-forward layer is a token-wise operation that takes in an input $\vx \in \RR^{\aux{D}}$ and returns $\vy = \mA \sigma( \mW \vx )$, with $\mA \in \RR^{\aux{D} \times 4\aux{D}}$ and $\mW \in \RR^{4\aux{D} \times \aux{D}}$. 
A naive construction of the \simulator to simulate its forward operation will have 2  Linear Forward modules (\cref{sec:exposition_linear_forward}), separated by an activation. However, this requires $4\times$ more prefix embeddings to represent the parameters, compared to other linear operations in the auxiliary transformer that use $ \RR^{\aux{D} \times \aux{D}}$ weight parameters.

To avoid this, we can instead break down the computation into 4 sub-feed-forward layers, each with its own parameters $\{\{\mW^i, \mA^i\}\}_{1 \leq i \leq 4}$. Here $\{\mW^i\}_{1 \leq i \leq 4}$ represent $4$-shards of the rows of $\mW$, and $\{\mA^i\}_{1 \leq i \leq 4}$ represent $4$-shards of the columns of $\mA$. The forward, backward, and descent operations on these 4 sub-feed-forward layers can be effectively parallelized. For example, the forward operation of each layer can be simulated by a single \simulator module, consisting of two Linear Forward modules and activation, changing only the prefix embeddings to correspond to $\{\{\mW^i, \mA^i\}\}_{1 \leq i \leq 4}$.

%% file: Additional_modules.tex
\section{Additional modules}\label{sec:additional_modules}
We describe the forward, backward, and decent update operations of additional modules, used in different model families, like LLaMA \cite{touvron2023llama} and BLOOM \cite{scao2022bloom}. We discuss the simulation of these modules, using similar \simulator{} modules.

\subsection{Root mean square normalization (RMSnorm)} 
The operation of RMSnorm \cite{zhang2019root} is very similar to layer normalization.
\begin{definition}[RMSnorm]\label{def:rmsnorm}
    For an arbitrary dimension $d$, define a normalization function $f:\RR^d\to\RR^d$ that performs $f(\vx) = \vx / RMS(\vx)$, where $RMS(\vx) = (\sum_{i=1}^{d} x_i^2)^{1/2}.$ Then, RMSnorm with parameters $\vgamma, \vb \in \RR^{\aux{D}}$ takes as input  $\vx \in \RR^{\aux{D}}$ and outputs $\vy\in\RR^{\aux{D}}$, which is computed as $\vz = f(\vx), \vy = \vgamma \odot \vz + \vb.$
\end{definition}

The extreme similarity between RMSnorm and layer normalization (\cref{def:layernorm}) helps us create similar \simulator{} modules as described in \cref{sec:LNappendix}, where instead of Group normalization layers, we use Group RMSnorm layers described below.

\begin{definition}[\simulator{} $\aux{D}$-Group RMSnorm]\label{def:group_rmsnorm}
    For an arbitrary dimension $d$, define a normalization function $f:\RR^d\to\RR^d$ that performs $f(\vx) = \vx / RMS(\vx)$, where $RMS(\vx) = (\sum_{i=1}^{d} x_i^2)^{1/2}.$ Then, $\aux{D}$-Group RMSnorm with parameters $\simuw{\vgamma}, \simuw{\vb} \in \RR^{\aux{D}}$ takes as input  $\vx \in \RR^{\simu{D}}$ and outputs $\vy = \attmerge( \{ \vy^h \in \RR^{\aux{D}} \}_{h \le \lfloor \simu{D}/\aux{D} \rfloor} )$, with
    \begin{align*}
        \vy^h = \simuw{\vgamma} \odot f(\vx^h) + \simuw{\vb},
    \end{align*}
    where $\vx^h = \attsplit_{\lfloor \simu{D}/\aux{D} \rfloor} ( \vx )_h.$
\end{definition}

\subsection{Attention variants}
In order to incorporate additional attention variants, e.g. Attention with Linear Biases (ALiBi) \cite{press2021train}, and rotary position embeddings \cite{su2021roformer}, we can change the definition of softmax attention layer in \cref{def:self-attn_construct} likewise. 

We showcase the changes for ALiBi.
\begin{definition}[Auxiliary ALiBi self-attention with $\aux{H}$ heads]\label{def:self-attn_alibi}
    For query, key, and value weights $\mW_Q, \mW_K, \mW_V \in\RR^{\aux{D}\times \aux{D}}$, bias $\vb_{Q}, \vb_{K}, \vb_{V} \in\RR^{\aux{D}}$ and $\vm \in \RR^{\aux{H}}$, ALiBi self-attention layer with $\aux{H}$ attention heads and a function $\attnfn: \RR^{\aux{T}} \to \RR^{\aux{T}}$ takes a sequence $\{\vx_t\in\RR^{\aux{D}}\}_{t \le \aux{T}}$ as input and outputs $\{ \vy_t \}_{t \le \aux{T}}$, with 
    \begin{align}
        \vy_t = \attmerge ( \{ \sum_{j \le \aux{T}} a^h_{t, j}  \vv^{h}_j  \}_{h \le \aux{H}} ).  \label{eq:attnt_forward_alibi}
    \end{align}
    $a^{h}_{t, j}$ is defined as the attention score of head $h$ between tokens at positions $t$ and $j$, and is given by
    \begin{align} \label{eq:attn_alibi}
        a^h_{t, j} =  \mathrm{softmax} ( \mK^{h} \vq^{h}_t + m_h \vr_t )_j.
    \end{align}
    Here $\vr_t \in \RR^{\aux{T}}$ denotes a relative position vector at each position $t$ that contains $(j - t)$ at each coordinate $j \le \aux{T}$.
     Here, $\vq_t$, $\vk_t$, $\vv_t$ denote the query, key, and value vectors at each position $t$, computed as $\mW_Q \vx_t + \vb_Q$, $\mW_K \vx_t + \vb_K$, and $\mW_V \vx_t + \vb_V$ respectively. In addition, $\vq^{h}_t, \vk^{h}_t, \vv^{h}_t$ denote $\attsplit_{\aux{H}} (\vq_t)_h$, $\attsplit_{\aux{H}} (\vk_t)_h$, and $\attsplit_{\aux{H}} (\vv_t)_h$ respectively for all $t \le \aux{T}$, and $h \le \aux{H}$. 
     $\mK^h \in \RR^{\aux{T} \times \aux{D} }$ is defined with its rows as $\{\vk^h_t\}_{t \le \aux{T}}$ for all $h \le \aux{H}$.
\end{definition}

To include operations involving ALiBi, we modify the self-attention module of \simulator{} to change the definition of the attention scores like \cref{eq:attn_alibi}.
\begin{definition}[Modified \simulator{} self-attention for ALiBi with $\simu{H}$ heads]\label{def:self-attn_alibi_tint}
    For parameters $\{ \simuw{\mW}_Q, \simuw{\mW}_K, \simuw{\mW}_V \in \RR^{ \simu{D} \times \simu{D} } \}$, $\{ \simuw{\vb}_Q, \simuw{\vb}_K, \simuw{\vb}_V \in \RR^{\simu{D}} \}$, $\{ \mW^p_Q, \mW^p_K, \mW^p_V \in \RR^{ \simu{T} \times \simu{D}/\simu{H} } \}$, $\{ \lambda^{Q},  \lambda^{K}, \lambda^{V} \in \RR^{ \simu{H} } \}$ and $\simuw{\vm} \in \RR^{\simu{T}}$, \simulator{} self-attention with $\simu{H}$ attention heads and a function $\attnfn: \RR^{ \simu{T} } \to \RR^{ \simu{T} }$ takes a sequence $\{ \Hat{\ve}_t\in\RR^ { \simu{D} }  \}_{t \le \simu{T}}$ as input and outputs $\{ \Tilde{\ve}_t \in \RR^{ \simu{D} } \}_{t \le \simu{T} }$, with 
    \begin{align*}
        &\Tilde{\ve}_t = \attmerge ( \{ \sum_{j \le \simu{T} } a^h_{t, j}  \Tilde{\vv}^{h}_j )_h \}_{h \le \simu{H} } ), \text{ with   } a^h_{t, j} =
        \attnfn ( \Tilde{\mK}^{h} \Tilde{\vq}_t^h + \simuw{m}_h \vr_t )_j  \\
        &\Tilde{\vq}^{h}_t = \attsplit_H (\vq_t)_h + \lambda^Q_h \mW^p_Q \simuw{\vp}_t; \quad \Tilde{\vk}^{h}_t = \attsplit_H (\vk_t)_h + \lambda^K_h \mW^p_K \simuw{\vp}_t + ; \\& \Tilde{\vv}^{h}_t = \attsplit_H (\vv_t)_h + \lambda^V_h \mW^p_v \simuw{\vp}_t .
    \end{align*}    
     Here $\vr_t \in \RR^{\simu{T}}$ denotes a relative position vector at each position $t$ that contains $(j - t)$ at each coordinate $j \le \simu{T}$. Here, $\vq_t$, $\vk_t$, $\vv_t$ denote the query, key, and value vectors at each position $t$, computed as $\simuw{\mW}_Q \Hat{\ve}_t + \simuw{\vb}_Q$, $\simuw{\mW}_K \Hat{\ve}_t + \simuw{\vb}_K$, and $\simuw{\mW}_V \Hat{\ve}_t + \simuw{\vb}_V$ respectively.
     $\Tilde{\mK}^h \in \RR^{ \simu{T} \times \simu{D}/\simu{H} }$ is defined with its rows as $\{\Tilde{\vk}^h_t\}_{t \le \simu{T}}$ for all $h \le \simu{H}$.
\end{definition}

After referring to \cref{sec:self-attnt-backprop_appendix}, we make the following modifications to the Forward, Backward, and Descent modules. In the Forward module, we incorporate the modified self-attention module to compute the attention scores using ALiBi attention. In the Backward module, since we do not propagate gradients through the attention scores of the auxiliary model, the backpropagation formulation remains unchanged from \cref{def:self-attn-appr-backprop} when we have access to the attention scores. Similarly, in the Descent module, we update the value matrix while keeping the query and key parameters fixed. The formulation of the gradient update remains unchanged from \cref{def:self-attn-value-update} when we have access to the attention scores. Consequently, we simply modify all the self-attention modules in the simulator to include ALiBi attention, as defined by \cref{def:self-attn_alibi_tint}.

\subsection{Gated linear units (GLUs)}
We describe the operations of GLUs \cite{shazeer2020glu} using similar GLU units available to the \simulator.

\begin{definition}\label{def:GLU_forward}
    For parameters $\mW, \mV, \mW^o \in \RR^{ \aux{D} \times \aux{D} }$, and biases $\vb_W, \vb_V, \vb_{W^o} \in \RR^{ \aux{D} }$, a GLU layer with activation $\act: \RR \to \RR$, takes input $\vx \in \RR^{\aux{D}}$ and outputs $\Hat{\vy} \in \RR^{\aux{D}}$, with
    \begin{align*}
        \vy = (\mW \vx + \vb_W) \odot \act ( \mV \vx + \vb_V ); \quad
        \Hat{\vy} = \mW^o \vy + \vb_{W^o}.
    \end{align*}
\end{definition}

Typical GLUs have $8/3 \times \aux{D}$ as a hidden dimension (i.e. the dimension of $\vy$). We can use similar parameter-sharing techniques discussed for feed-forward layers (\cref{sec:parameter_sharing_app}) with the \simulator{} modules presented here. Furthermore, since $\Hat{y}$ can be expressed as a combination of the gated operation and a linear operation, we focus on the computation of $\vy$ here.

For the discussion below, we consider a GLU (without the output linear layer) in the auxiliary model,   with parameters $\mW, \mV,  \vb_W, \vb_V$, that takes in input sequence $\vx_1, \cdots, \vx_T$ and outputs $\vy_1, \cdots, \vy_T$, with $\vy_t = (\mW \vx_t + \vb_W) \odot \act ( \mV \vx_t + \vb_V )$ for each $t \le \simu{T}$. Since this involves a token-wise operation, we will present our constructed modules with a general token position $t$ and the prefix tokens $\{\vv_j\}.$

\paragraph{\simulator{} GLU Forward module} 
The embedding $\ve_t$ contains $\vx_t$ in its first $\aux{D}$ coordinates. The output $\vy_t$ can be computed using three sub-operations: (a) linear operation for $\mW \vx_t + \vb_W$, (b) linear operation for $\mV \vx_t + \vb_V$, and (c)  
gate operation to get $(\mW \vx_t + \vb_W) \odot \act( \mV \vx_t + \vb_V )$.

We use three \simulator{} modules, representing each sub-operation.
\begin{enumerate}[label=(\alph*)]
    \item $\mW \vx_t + \vb_W$ is a linear operation, hence we can use a \simulator{} Linear Forward module (\cref{sec:Linearappendix}) with the current embedding $\ve_t$ and  $\{\vv_j\}$ containing $\mW, \vb_W$ to get embedding $\Tilde{\ve}_t$ containing $\mW \vx_t + \vb_W$ in its first $\aux{D}$ coordinates.

    \item  $\mV \vx_t + \vb_V$ is a linear operation, hence we can similarly use a \simulator{} Linear Forward module (\cref{sec:Linearappendix}) with the embedding $\ve_t$ and  $\{\vv_j\}$ containing $\mW_V, \vb_V$ to get embedding $\Hat{\ve}_t$ containing $\mV \vx_t + \vb_V$ in its first $\aux{D}$ coordinates. 
    
    $\Hat{\ve}_t$ and $\Tilde{\ve}_t$ are now combined to get an embedding $\ve_t$ that contains $\mW \vx_t + \vb_W, \mV \vx_t + \vb_V$
    in its first $2 \aux{D}$ coordinates.

    \item Finally, we can use a \simulator{} GLU layer that can carry out the elementwise multiplication of $\mW \vx_t + \vb_W, \act(\mV \vx_t + \vb_V)$ to get $\vy_t$ in the first $\aux{D}$ coordinates.
\end{enumerate}

\textit{Parameter Sharing:} Since (a) and (b) involve a Linear Forward
module, we can additionally leverage parameter sharing to apply a single Linear Forward module for each of the two computations, changing only the prefix embeddings to correspond to $\mW, \vb_W$, or $\mW_V, \vb_V$.

\paragraph{Auxiliary GLU backpropagation}
For the GLU layer defined in \cref{def:GLU_forward}, the backpropagation layer takes in the loss gradient w.r.t. output ($\losspartial{\vy}$) and computes the loss gradient  w.r.t. input ($\losspartial{\vx}$).

\begin{definition}[Auxiliary GLU backpropagation]\label{def:glu_backprop}
    For the weights $\mW, \mV \in\RR^{ \aux{D}  \times \aux{D} }$ , the backpropagation layer takes $\losspartial{ \vy } \in\RR^{ \aux{D} }$ as input and outputs $\losspartial{\vx} \in\RR^{ \aux{D} }$, with $\losspartial{\vx} = \mW^{\top} \Hat{\losspartial{\vx}} + \mV^{\top} \Tilde{\losspartial{\vx}}$, where
    \begin{align*}
        \Hat{\losspartial{\vx}} =  \losspartial{\vy} \odot \act( \mV \vx + \vb_V ); \qquad
        \Tilde{\losspartial{\vx}} =   \act'( \mV \vx + \vb_V  ) \odot \losspartial{\vy} \odot (\mW \vx + \vb_W)  .
    \end{align*}
\end{definition}

A direct computation of $\Tilde{\losspartial{\vx}}$ involves changing the activation function to $\act'$.
Following a similar strategy for backpropagation through an activation layer (\cref{sec:act_appendix}), we instead use a first-order Taylor expansion to approximate $\Tilde{\losspartial{\vx}}$.

\begin{definition}[Auxiliary GLU approximate backpropagation]\label{def:glu_backprop_appr}
    For a hyper-parameter $\epsilon > 0$, for the weights $\mW, \mV \in\RR^{ \aux{D}  \times \aux{D} }$ , the approximate backpropagation layer takes $\losspartial{ \vy } \in\RR^{ \aux{D} }$ as input and outputs $\overline{\losspartial{\vx}} \in\RR^{ \aux{D} }$,  with $\overline{\losspartial{\vx}} = \mW^{\top} \Hat{\losspartial{\vx}} + \mV^{\top} \Hat{\Tilde{\losspartial{\vx}}}$, where
    \begin{align*}
        \Hat{\losspartial{\vx}} &=  \losspartial{\vy} \odot \act( \mV \vx + \vb_V )  \\ 
        \Hat{\Tilde{\losspartial{\vx}}} &=   \act( \mV \vx + \vb_V + \epsilon  \losspartial{\vy}  ) \odot \frac{1}{\epsilon} (\mW \vx + \vb_W)  - \act( \mV \vx + \vb_V  ) \odot \frac{1}{\epsilon} (\mW \vx + \vb_W) .
    \end{align*}
\end{definition}

\paragraph{\simulator{} GLU backpropagation module} 
The current embedding contains $\losspartial{\vy_t}$ in its first $\aux{D}$ coordinates. Furthermore, since we need $\mW \vx_t + \vb_W$ and $\mV \vx_t + \vb_V$ in the gradient computations, we copy them from the Forward module using residual connections. We discuss the computation of $\mW^{\top} \Hat{\losspartial{\vx_t}}$ and $\mV^{\top} \Hat{\Tilde{\losspartial{\vx_t} }}$ as separate sub-modules acting on the same embedding $\ve_t$ in parallel.

\begin{enumerate}
    \item The computation of $\mW^{\top} \Hat{\vx_t}$ involves two sub-operations: (a)  gate operation to get $\Hat{\vx_t} := \losspartial{\vy_t} \odot \act( \mV \vx_t + \vb_V )$, and (b) linear backward operation to get $\mW^{\top} \Hat{\vx_t}$. Since for this operation, we require $\mW$, we copy the contents of the prefix embeddings containing $\mW, \vb_W$ from the Forward module.
    
    \begin{enumerate}
        \item Since the current embedding $\ve_t$ contains both $\losspartial{\vy_t}$ and $\mW \vx_t + \vb_W$, we can use a \simulator{} GLU layer to get an embedding $\Hat{\ve}^{(1)}_t$ that contains $\Hat{\losspartial{\vx_t}}$.
        \item The final linear backward operation can be performed by using a \simulator{} Linear backpropagation module (\cref{sec:Linearappendix}) with the embeddings $\Hat{\ve}^{(1)}_t$ and the prefix embeddings. The final embedding $\Hat{\ve}_t$ contains $\mW^{\top} \Hat{\vx_t}$ in the first $\aux{D}$ coordinates.
    \end{enumerate}
    
    \item The computation of $\mV^{\top} \Hat{\Tilde{\vx_t}}$ involves four sub-operations: (a)  gate operation to get $\frac{1}{\epsilon} (\mW \vx_t + \vb_W) \odot \act( \mV \vx_t + \vb_V + \epsilon \losspartial{\vy_t} )$, (b)  gate operation to get $\frac{1}{\epsilon}(\mW \vx_t + \vb_W) \odot \act( \mV \vx_t + \vb_V )$, (c) a linear layer to compute $\Hat{\Tilde{\vx_t}}$, 
    (c) linear backward operation to get $\mV^{\top} \Hat{\Tilde{\vx_t}}$. Since for this operation, we require $\mV$, we copy the contents of the prefix embeddings containing $\mV, \vb_V$ from the Forward module.
    
    \begin{enumerate}
        \item Since the current embedding $\ve_t$ contains  $\losspartial{\vy_t}$, $\mV \vx_t + \vb_W$ and $\mW \vx_t + \vb_W$, we can use two \simulator{} GLU layers to get an embedding $\Tilde{\ve}^{(1)}_t$ that contains both $\frac{1}{\epsilon} (\mW \vx_t + \vb_W) \odot \act( \mV \vx_t + \vb_V + \epsilon \losspartial{\vy_t} )$ and $\frac{1}{\epsilon} (\mW \vx_t + \vb_W) \odot \act( \mV \vx_t + \vb_V)$.
        \item A linear later on  $\Tilde{\ve}^{(1)}_t$ can then return an embedding $\Tilde{\ve}^{(2)}_t$ containing $\Hat{\Tilde{\vx_t}}$ in the first $\aux{D}$ coordinates.
        \item The final operation can be performed by using a \simulator{} Linear backpropagation module (\cref{sec:Linearappendix}) with the embeddings $\Hat{\ve}^2_t$ and the prefix embeddings containing $\mV, \vb_V$. The final embedding $\Tilde{\ve}_t$ contains $\mV^{\top} \Hat{\Tilde{\vx_t}}$ in the first $\aux{D}$ coordinates.
    \end{enumerate}

\end{enumerate}

After the two parallel computations, we can sum up $\Hat{\ve}_t$ and $\Tilde{\ve}_t$ to get an embedding $\ve_t$ containing $\overline{\losspartial{\vx_t}}$ (\cref{def:glu_backprop_appr}) in the first $\aux{D}$ coordinates.

\paragraph{Auxiliary GLU descent} Finally, the auxiliary's descent  updates the weight and the bias parameters using a batch of inputs $\{\vx_t\}_{t \le T}$ and the loss gradient w.r.t. the corresponding outputs $\{ \losspartial{\vy_t} \}_{t \le T}$.

\begin{definition}[Auxiliary GLU descent ]\label{def:glu_descent}
    For weights $\mW, \mV \in\RR^{\aux{D} \times \aux{D}}$ and bias $\vb_W, \vb_V \in \RR^{ \aux{D} }$, the linear descent layer takes in a batch of inputs $\{ \vx_t \in \RR^{\aux{D}} \}_{t \le \aux{T}}$ and gradients $\{ \losspartial{ \vy_t } \in \RR^{\aux{D}} \}_{t \le \aux{T}}$ and updates the parameters as follows:
    \begin{align*}
        &\mW \gets \mW - \eta \sum_{t \le \aux{T}} \Hat{\losspartial{\vx_t}} \vx_t^{\top}; \quad \quad
        \vb_W \gets \vb_W - \eta \sum_{t \le \aux{T}} \Hat{\losspartial{\vx_t}} ,\\ &
        \mV \gets \mV - \eta \sum_{t \le \aux{T}}  \Tilde{ \losspartial{\vx_t} } \vx_t^{\top}; \quad \quad
        \vb_V \gets \vb_V - \eta \sum_{t \le \aux{T}} \Tilde{ \losspartial{\vx_t} },
    \end{align*}
    where $\Hat{\losspartial{\vx_t}}$ and $\Tilde{ \losspartial{\vx_t} }$ have been computed as \cref{def:glu_backprop}.
\end{definition}
Due to similar concerns as gradient backpropagation, we instead use $\Hat{\Tilde{ \losspartial{\vx_t} }}$ (\cref{def:glu_backprop_appr}) in place of $\Tilde{ \losspartial{\vx_t} }$ for each $t \le \aux{T}$ to update $\mV, \vb_V.$

\paragraph{\simulator{} GLU descent module} 
We discuss the two descent operations separately.

\begin{enumerate}
    \item Update of $\mW, \vb_W$: We start with the embeddings $\Hat{\ve}^{(1)}_t$ from the  backpropagation module, that contain $\Hat{\losspartial{\vx_t}}$ in the first $\aux{D}$ coordinates. 

    For the update, we additionally require the input to the auxiliary GLU layer under consideration, and hence we copy $\vx_t$ from the Forward module using residual connections. Furthermore, we copy the contents of the prefix embeddings that contain $\mW, \vb_W$ from the Forward module.

    With both $\Hat{\losspartial{\vx_t}}$ and $\vx_t$ in the embeddings, the necessary operation turns out to be the descent update of a linear layer with parameters $\mW, \vb_W$. That implies, we can call a \simulator{} Linear descent module (\cref{sec:Linearappendix}) on the current embeddings and prefix embeddings to get the desired update.

    \item We start with the embeddings $\Tilde{\ve}^{(2)}_t$ from the  backpropagation module, that contain $\Tilde{\Hat{\losspartial{\vx_t}}}$ in the first $\aux{D}$ coordinates. 

    For the update, we additionally require the input to the auxiliary GLU layer under consideration, and hence we copy $\vx_t$ from the forward module using residual connections. Furthermore, we copy the contents of the prefix embeddings  that contain $\mV, \vb_V$ from the Forward module.

    With both $\Tilde{\Hat{\losspartial{\vx_t}}}$ and $\vx_t$ in the embeddings, the necessary operation turns out to be the descent update of a linear layer with parameters $\mV, \vb_V$. That implies we can call a \simulator{} Linear descent module on the current embeddings and prefix embeddings to get the desired update.
\end{enumerate}
\textit{Parameter sharing}: Since both the descent updates involve a Linear descent
module, we can additionally leverage parameter sharing to apply a single \simulator{} Linear descent module for each of the two computations, changing the input to correspond to $\{\Hat{\ve}^{(1)}_t\}$ and prefix to correspond to $\mW, \vb_W$, or the input to correspond to $\{\Tilde{\ve}^{(2)}_t\}$ and prefix to correspond to $\mV, \vb_V$ respectively.

%% file: experiment_app.tex
\section{Construction of other variants of pre-trained models}\label{app:other_construction}

Though we only conduct experiments on an \textsc{OPT-125m} model, our construction is generally applicable to diverse variants of pre-trained language models.
\Cref{tab:construction} highlights many types of modules and the required size and computation for each.
%We showcase examples in \Cref{tab:construction} of the types of modules and their required size and computation. 
The size of a constructed model is influenced by various factors, including the number of layers, and embedding dimension in the auxiliary. 

\section{Experiments}\label{app:experiment}

\input{Tables/Downstream_table_appendix}

\textbf{Computing environment}: All the experiments are conducted on a single A100 80G GPU. 

\textbf{Hyperparameters:}
In the few-shot setting, we employ three different random seeds to select distinct sets of training examples. Grid search is performed for each seed to determine the optimal learning rate for both constructed models and dynamic evaluation. The learning rates considered for the learning rate hyperparameter in the descent update operations in $\textsc{TinT}$ are $1e-3, 1e-4, 1e-5$.~\footnote{When utilizing the full-context loss, the learning rates considered are ${1e-5, 1e-6}$, and $1e-7$ due to gradient summations in \textsc{TinT}.} Additionally, we explore various layer-step combinations to allocate a fixed budget for one full forward pass. Specifically, we update the top 3 layers for 4 steps, the top 6 layers for 3 steps, or 12 layers for 1 step.

\paragraph{Calibration:} Recall from \cref{sec:experiments} that given a downstream task input (e.g., a movie review), the model's predicted label is computed as follows. First, we design a simple task-specific prompt (e.g., ``Sentiment:'') and select label words $c_1,...,c_n$ to serve as surrogates for each class (e.g., ``positive'' and ``negative''). Then, we provide the input along with the prompt to the model, and the label word assigned the highest probability is treated as the model's prediction. We compare \simulator{} to its baselines in two settings: no calibration (reported in \cref{tab:main_result} in the main paper), and with calibration. If using calibration, then the probabilities are normalized using just the prompt as input.\footnote{Calibration is not applied to the language modeling evaluation.}  
$$
\text{No Calibration:} \argmax_{c_i} \Pr[c_i\mid\text{input, prompt}] \qquad
\text{Calibration:} \arg\max_{c_i} \frac{\Pr[c_i\mid\text{input, prompt}]}{\Pr[c_i\mid \text{prompt}]}
$$
This is a widely used calibration technique~\citep{holtzman2021surface} for prompting language models.

\paragraph{Additional observations from \cref{tab:main_result_app}, compared to \cref{tab:main_result}:} In \cref{tab:main_result_app}, we have reported the comparisons with calibration in addition to the non calibration results reported in \cref{tab:main_result}. We observe that calibration may not always be beneficial in every setting.\footnote{Such inconsistencies in the calibration method have been observed in previous works~\citep{brown2020language}.} However, even with calibration, \simulator{} remains competitive to fine-tuning of OPT models. The performance of OPT-1.3B improves with calibration. In this case, \simulator{} lags behind OPT-1.3B in the few-shot setting.

\paragraph{Results of different settings.}
Table \ref{tab:setting_result} displays the results of few-shot learning with calibration across various settings, encompassing different loss types, input formats, and layer-step configurations. Our analysis reveals that employing a label-only loss, utilizing a single-example input format, and updating all layers of the internal model for a single step yield the most favorable average result. The performance of the multi-example format is disadvantaged when dealing with tasks of long sequences such as Amazon Polarity. In general, we observe that calibrated results tend to be more consistent and stable.
\begin{table}[hbt]
\centering
\caption{Few-shot ($k=32$) results with different loss types, input formats, and layer-step configurations with a fixed compute budget, with calibration.}
\label{tab:setting_result}
\resizebox{\textwidth}{!}{
\begin{tabular}{lccc|cccccccc}
\toprule
\textbf{Loss Type}                & \textbf{Format} & \textbf{Layer} & \textbf{Step} & \textbf{Subj} & \textbf{AGNews} & \textbf{SST2} & \textbf{CR}   & \textbf{MR}   & \textbf{MPQA} & \textbf{Amazon} & \textbf{Avg.} \\ \midrule
\textbf{Label} & Single & 12 & 1 & $66.0_{(1.9)}$ & $64.7_{(0.2)}$ & $68.7_{(1.3)}$ & $69.0_{(0.7)}$ & $63.7_{(0.2)}$ & $82.8_{(0.5)}$ & $73.7_{(0.6)}$ & $69.8_{(0.1)}$ \\
 & Single & 6 & 2 & $62.7_{(0.2)}$ & $66.3_{(0.2)}$ & $68.3_{(6.1)}$ & $67.2_{(0.2)}$ & $61.8_{(1.6)}$ & $81.0_{(3.6)}$ & $74.3_{(0.5)}$ & $68.8_{(1.4)}$ \\
 & Single & 3 & 4 & $63.5_{(0.0)}$ & $67.2_{(0.8)}$ & $62.5_{(0.4)}$ & $68.7_{(1.4)}$ & $61.7_{(0.6)}$ & $76.8_{(3.3)}$ & $75.2_{(0.8)}$ & $67.9_{(0.8)}$ \\ \cmidrule{2-12}
& Multi.  & 12 & 1 & $83.2_{(2.5)}$ & $43.7_{(6.6)}$ & $60.7_{(5.7)}$ & $70.3_{(6.1)}$ & $62.8_{(8.9)}$ & $84.2_{(1.6)}$ & $66.3_{(12.3)}$ & $67.3_{(0.9)}$ \\
 & Multi. & 6 & 2 & $83.5_{(2.9)}$ & $43.2_{(8.4)}$ & $52.0_{(1.5)}$ & $70.5_{(6.0)}$ & $58.5_{(11.3)}$ & $82.0_{(0.4)}$ & $55.8_{(7.6)}$ & $63.6_{(2.7)}$ \\
 & Multi. & 3 & 4 & $84.0_{(2.3)}$ & $42.3_{(8.4)}$ & $51.5_{(1.8)}$ & $68.2_{(4.6)}$ & $58.5_{(12.0)}$ & $80.2_{(2.1)}$ & $58.5_{(7.9)}$ & $63.3_{(3.0)}$\\ \midrule
\textbf{Full-context}  & Single & 12 & 1 & $64.5_{(0.4)}$ & $65.8_{(0.2)}$ & $63.2_{(0.9)}$ & $67.3_{(0.5)}$ & $60.8_{(1.4)}$ & $73.5_{(0.8)}$ & $75.0_{(0.4)}$ & $67.2_{(0.1)}$ \\
 & Single & 6 & 2 & $66.7_{(2.0)}$ & $66.0_{(0.4)}$ & $62.7_{(0.6)}$ & $70.5_{(2.1)}$ & $59.7_{(0.9)}$ & $77.7_{(2.2)}$ & $76.0_{(0.0)}$ & $68.5_{(0.4)}$ \\
 & Single & 3 & 4 & $64.0_{(0.0)}$ & $65.8_{(0.6)}$ & $65.0_{(1.9)}$ & $67.3_{(0.2)}$ & $59.5_{(0.4)}$ & $74.2_{(1.3)}$ & $77.0_{(1.9)}$ & $67.5_{(0.8)}$\\ \cmidrule{2-12}
 & Multi. & 12 & 1 & $83.8_{(2.9)}$ & $41.0_{(10.6)}$ & $51.2_{(0.8)}$ & $68.0_{(4.5)}$ & $58.3_{(11.1)}$ & $79.0_{(3.6)}$ & $56.0_{(8.1)}$ & $62.5_{(2.8)}$ \\
 & Multi. & 6 & 2 & $85.3_{(1.9)}$ & $41.2_{(10.7)}$ & $51.2_{(1.3)}$ & $67.7_{(4.5)}$ & $57.7_{(10.8)}$ & $79.2_{(3.7)}$ & $55.8_{(7.9)}$ & $62.6_{(2.6)}$ \\
 & Multi. & 3 & 4 & $83.3_{(2.5)}$ & $41.7_{(11.3)}$ & $51.0_{(1.1)}$ & $68.2_{(4.7)}$ & $57.7_{(10.8)}$ & $79.0_{(3.2)}$ & $56.0_{(8.1)}$ & $62.4_{(2.8)}$ \\ \bottomrule
\end{tabular}}
\end{table}

\begin{table}[hbt]
\centering
\caption{Few-shot ($k=32$) results with different loss types, input formats, and layer-step configurations with a fixed compute budget, without calibration.}
\label{tab:setting_result_plain}
\resizebox{\textwidth}{!}{
\begin{tabular}{lccc|cccccccc}
\toprule
\textbf{Loss Type}                & \textbf{Format} & \textbf{Layer} & \textbf{Step} & \textbf{Subj} & \textbf{AGNews} & \textbf{SST2} & \textbf{CR}   & \textbf{MR}   & \textbf{MPQA} & \textbf{Amazon} & \textbf{Avg.} \\ \midrule
\textbf{Label} & Single & 12 & 1 & $63.3_{(0.2)}$ & $65.7_{(0.2)}$ & $71.3_{(0.6)}$ & $65.0_{(1.4)}$ & $70.7_{(0.9)}$ & $65.0_{(0.0)}$ & $76.7_{(0.2)}$ & $68.2_{(0.1)}$ \\
& Single & 6 & 2 & $63.5_{(0.0)}$ & $65.2_{(0.5)}$ & $73.3_{(1.3)}$ & $68.5_{(3.7)}$ & $71.3_{(0.2)}$ & $66.0_{(0.0)}$ & $77.5_{(0.4)}$ & $69.3_{(0.3)}$ \\
& Single & 3 & 4 & $64.2_{(0.2)}$ & $66.5_{(1.1)}$ & $73.2_{(0.6)}$ & $75.7_{(0.5)}$ & $72.0_{(0.0)}$ & $83.2_{(1.0)}$ & $78.0_{(0.4)}$ & $73.2_{(0.1)}$ \\\cmidrule{2-12}
& Multi. & 12 & 1 & $64.5_{(7.8)}$ & $35.5_{(7.4)}$ & $56.8_{(9.7)}$ & $63.0_{(6.7)}$ & $58.7_{(8.9)}$ & $75.2_{(10.8)}$ & $62.2_{(8.3)}$ & $59.4_{(0.6)}$ \\
& Multi. & 6 & 2 & $77.7_{(7.0)}$ & $35.5_{(7.4)}$ & $57.0_{(9.9)}$ & $60.0_{(6.3)}$ & $52.3_{(2.1)}$ & $58.5_{(6.1)}$ & $55.8_{(7.9)}$ & $56.7_{(2.6)}$ \\
& Multi. & 3 & 4 & $67.5_{(11.5)}$ & $38.5_{(8.2)}$ & $55.3_{(5.2)}$ & $67.0_{(3.5)}$ & $61.0_{(8.0)}$ & $65.2_{(11.2)}$ & $62.5_{(8.9)}$ & $59.6_{(1.3)}$ \\ \midrule
\textbf{Full-context}  & Single & 12 & 1 & $65.5_{(1.1)}$ & $66.5_{(0.0)}$ & $70.7_{(0.2)}$ & $64.8_{(0.5)}$ & $72.0_{(1.4)}$ & $67.0_{(0.0)}$ & $76.5_{(0.0)}$ & $69.0_{(0.3)}$ \\
& Single & 6 & 2 & $64.7_{(0.6)}$ & $66.2_{(0.2)}$ & $71.2_{(0.2)}$ & $65.3_{(0.6)}$ & $71.5_{(0.4)}$ & $67.0_{(0.0)}$ & $76.7_{(0.2)}$ & $68.9_{(0.0)}$ \\
& Single & 3 & 4 & $64.2_{(0.2)}$ & $66.2_{(0.2)}$ & $71.3_{(0.2)}$ & $64.7_{(0.2)}$ & $71.0_{(0.0)}$ & $67.0_{(0.0)}$ & $76.5_{(0.0)}$ & $68.7_{(0.0)}$ \\\cmidrule{2-12}
& Multi. & 12 & 1 & $62.2_{(7.5)}$ & $33.8_{(8.3)}$ & $52.2_{(3.1)}$ & $52.8_{(4.0)}$ & $50.8_{(1.2)}$ & $55.8_{(4.3)}$ & $55.3_{(7.2)}$ & $51.9_{(2.2)}$ \\
& Multi. & 6 & 2 & $60.0_{(5.5)}$ & $33.7_{(8.4)}$ & $50.8_{(1.2)}$ & $52.2_{(2.4)}$ & $50.2_{(0.2)}$ & $54.3_{(2.5)}$ & $55.0_{(6.7)}$ & $50.9_{(1.8)}$ \\
& Multi. & 3 & 4 & $58.7_{(4.9)}$ & $33.7_{(8.4)}$ & $50.8_{(1.2)}$ & $51.3_{(1.9)}$ & $50.0_{(0.0)}$ & $54.3_{(2.5)}$ & $55.3_{(7.2)}$ & $50.6_{(2.0)}$ \\\bottomrule
\end{tabular}}
\end{table}

%% file: Tables/Downstream_table_appendix.tex
\begin{table*}[t]
  \centering
  \caption{Zero-shot and few-shot in-context learning results across $7$ downstream tasks. All the few-shot results are averaged over three training seeds. \textsc{TinT} consistently surpasses its auxiliary model and achieves comparable performance to Fine-tuninguation. \textsc{TinT} outperforms auxiliary models by $3-4\%$ and $12-16\%$ absolute points on average in $0$-shot and  $32$-shot experiments respectively. \textsc{TinT} performs competitively with a similar-sized pre-trained model (\textsc{opt-1.3b}) in both $0$-shot and $32$-shot settings. We show the standard deviation for few-shot settings in parentheses.
  }
  \label{tab:main_result_app}
  \resizebox{\textwidth}{!}{
  \begin{tabular}{lc|cccccccc}
    \toprule
    \textbf{Model} & \textbf{Shots} & \textbf{Subj} & \textbf{AGNews} & \textbf{SST2} & \textbf{CR} & \textbf{MR} & \textbf{MPQA} & \textbf{Amazon} & \textbf{Avg.} \\
    \midrule
\multicolumn{2}{c}{} & \multicolumn{8}{c}{\textbf{\textit{Without Calibration}}} \\ \midrule
\textsc{OPT-125m} & $0$ & $64.0$ & $66.0$ & $70.5$ & $64.5$ & $71.0$ & $68.0$ & $76.5$ & $68.6$ \\
\textsc{OPT-1.3b} & $0$ & $59.0$ & $55.5$ & $54.0$ & $50.5$ & $52.5$ & $74.0$ & $57.0$ & $57.5$ \\
\textsc{OPT-125m} Fine-tuning & $0$ & $71.0$ & $67.0$ & $79.5$ & $71.5$ & $70.0$ & $68.0$ & $85.5$ & $73.2$ \\
\rowcolor{gray!10}\textsc{OPT-125m TinT} & $0$ & $67.5$ & $66.0$ & $76.5$ & $69.0$ & $76.0$ & $70.5$ & $78.5$ & $72.0$ \\
 \midrule
\textsc{OPT-125m} & $32$ & $58.7_{(4.9)}$ & $33.7_{(8.4)}$ & $50.8_{(1.2)}$ & $51.3_{(1.9)}$ & $50.0_{(0.0)}$ & $54.3_{(2.5)}$ & $55.0_{(6.7)}$ & $50.5_{(1.9)}$ \\
\textsc{OPT-1.3b} & $32$ & $74.2_{(6.1)}$ & $71.3_{(5.3)}$ & $89.8_{(3.6)}$ & $71.5_{(4.5)}$ & $68.3_{(6.1)}$ & $81.7_{(3.3)}$ & $70.3_{(9.9)}$ & $75.3_{(0.4)}$  \\ 
\textsc{OPT-125m} Fine-tuning & $32$ & $78.0_{(1.4)}$ & $66.7_{(1.6)}$ & $71.5_{(1.4)}$ & $73.7_{(3.3)}$ & $72.0_{(0.0)}$ & $80.7_{(0.6)}$ & $79.8_{(0.2)}$ & $74.6_{(2.7)}$ \\
\rowcolor{gray!10}\textsc{OPT-125m TinT} & $32$ & 
$82.3_{(2.7)}$ & $69.3_{(0.9)}$ & $73.7_{(0.8)}$ & $75.7_{(1.9)}$ & $72.3_{(1.2)}$ & $83.2_{(1.0)}$ & $78.2_{(0.2)}$ & $76.4_{(0.7)}$ \\ 
\midrule
\multicolumn{2}{c}{} & \multicolumn{8}{c}{\textbf{\textit{With Calibration}}} \\ \midrule
\textsc{OPT-125m} & $0$ & $64.0$ & $66.0$ & $53.0$ & $54.5$ & $52.5$ & $55.5$ & $58.0$ & $57.6$ \\
\textsc{OPT-1.3b} & $0$ & $73.5$ & $61.5$ & $57.5$ & $53.0$ & $54.5$ & $79.5$ & $61.0$ & $62.9$ \\
\textsc{OPT-125m} Fine-tuning & $0$ & $62.5$ & $66.0$ & $60.5$ & $53.5$ & $54.0$ & $56.5$ & $74.5$ & $61.1$ \\
\rowcolor{gray!10}\textsc{OPT-125m TinT} & $0$ & $64.0$ & $66.0$ & $56.5$ & $59.0$ & $53.5$ & $62.0$ & $66.5$ & $61.1$ \\ \midrule
\textsc{OPT-125m} & $32$ & $83.5_{(2.4)}$ & $40.7_{(10.4)}$ & $50.8_{(0.8)}$ & $67.7_{(4.1)}$ & $57.7_{(10.8)}$ & $79.2_{(8.4)}$ & $56.0_{(8.1)}$ & $62.2_{(2.7)}$ \\
\textsc{OPT-1.3b} & $32$ & $51.8_{(1.9)}$ & $66.2_{(3.1)}$ & $93.7_{(1.0)}$ & $82.8_{(2.8)}$ & $91.3_{(1.9)}$ & $83.5_{(2.5)}$ & $92.0_{(2.9)}$ & $80.2_{(0.7)}$ \\ 
\textsc{OPT-125m} Fine-tuning & $32$ & $87.2_{(0.2)}$ & $67.2_{(0.6)}$ & $72.8_{(5.9)}$ & $73.3_{(2.6)}$ & $66.7_{(7.4)}$ & $81.5_{(3.7)}$ & $70.3_{(2.1)}$ & $74.1_{(2.9)}$ \\
\rowcolor{gray!10}\textsc{OPT-125m TinT} & $32$ & $85.3_{(1.9)}$ & $67.3_{(0.6)}$ & $71.8_{(3.8)}$ & $70.7_{(1.9)}$ & $63.7_{(0.2)}$ & $83.5_{(1.6)}$ & $77.5_{(1.2)}$ & $74.3_{(1.4)}$ \\ 
    \bottomrule
  \end{tabular}}
\end{table*}